\documentclass[journal]{IEEEtran}

\usepackage{amsmath, amssymb , graphicx}

\usepackage{subcaption}
\usepackage{multirow}
\usepackage{balance}
\usepackage{hyperref}
\usepackage{dsfont}
\usepackage{breqn}
\usepackage{makecell}
\usepackage{xcolor}
\usepackage[ruled,vlined,linesnumbered]{algorithm2e}

\usepackage{amsthm}
\newtheorem{claim}{Claim}

\title{Robotic Tool Tracking under Partially Visible Kinematic Chain: A Unified Approach}
\author{Florian Richter$^1$ \IEEEmembership{Student Member, IEEE}, Jingpei Lu$^{1}$, Ryan K. Orosco$^2$ \IEEEmembership{Member, IEEE},\\ and Michael C. Yip$^1$ \IEEEmembership{Senior Member, IEEE} 
\thanks{$^1$Florian Richter, Jingpei Lu, and Michael C. Yip are with the Department of Electrical and Computer Engineering, University of California San Diego, La Jolla, CA 92093 USA. {\tt\small \{frichter, jil360, yip\}@ucsd.edu}}%
\thanks{$^2$Ryan K. Orosco is with the Department of Surgery - Division of Head and Neck Surgery, University of California San Diego, La Jolla, CA 92093 USA. {\tt\small rorosco@ucsd.edu}}}

\begin{document}

\maketitle
\thispagestyle{empty}
\pagestyle{empty}


\begin{abstract}
Anytime a robot manipulator is controlled via visual feedback, the transformation between the robot and camera frame must be known. However, in the case where cameras can only capture a portion of the robot manipulator in order to better perceive the environment being interacted with, there is greater sensitivity to errors in calibration of the base-to-camera transform. A secondary source of uncertainty during robotic control are inaccuracies in joint angle measurements which can be caused by biases in positioning and complex transmission effects such as backlash and cable stretch. In this work, we bring together these two sets of unknown parameters into a unified problem formulation when the kinematic chain is partially visible in the camera view. We prove that these parameters are non-identifiable implying that explicit estimation of them is infeasible. To overcome this, we derive a smaller set of parameters we call Lumped Error since it lumps together the errors of calibration and joint angle measurements. A particle filter method is presented and tested in simulation and on two real world robots to estimate the Lumped Error and show the efficiency of this parameter reduction.
\end{abstract}


\section{Introduction}

Anytime a robot manipulator is being controlled via visual feedback, an important coordinate transform must be known: the orientation and translation between the robot and the camera frame.
This transforms positions and velocities of the robot defined by its kinematics, such as its end-effector position, into the camera's frame of reference where feedback and trajectories are often defined.
Typically this relative transform is calibrated for by placing markers, such as ArUco \cite{garrido2014automatic}, on the robot, identifying them in the image frame, and solving the homogeneous linear system for the \textit{base-to-camera} transform \cite{fassi2005hand}.
However, in the case where cameras can only observe a portion of the robot manipulator, there is greater sensitivity to errors in calibration of the base-to-camera transform due to the limited range of motions the robot can take on when collecting data \cite{tsai1989new}.
This situation arises when the camera is positioned to perceive the robotic tool (e.g. gripper) and the environment or objects being manipulated with rather than the entire kinematic chain.
An example scenario is shown in Fig. \ref{fig:cover_figure}.
This comes up frequently in object grasping and manipulation tasks \cite{levine2018learning, mahler2019learning} and small-scale manipulations such as robotic surgery with the da Vinci\textregistered{} Surgical System.

\begin{figure}[t]
	\centering
	\vspace{2mm}
	\includegraphics[trim=0cm 14.25cm 27cm 0cm, clip, width=8.5cm]{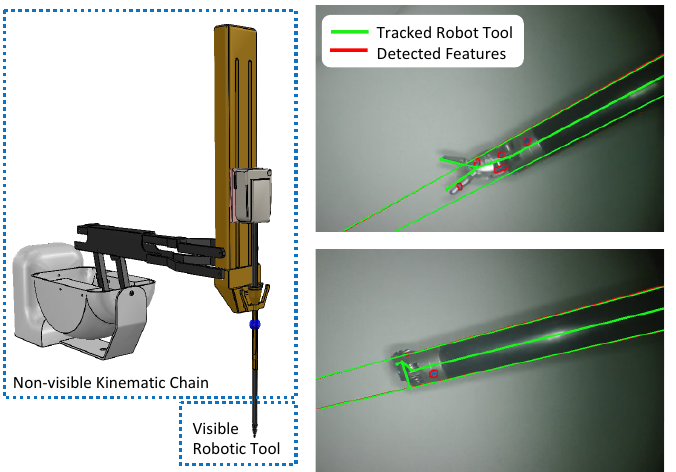}
	\caption{Precise robotic manipulation utilizes visual information with the sensor positioned to observe the environment and objects of interest rather than the entire kinematic chain. As such, it is challenging to track the robotic tool due to partially visible kinematic chain. In this work we derive and track a smaller set of parameters called Lumped Error for effective robotic tool tracking in such scenarios.
	The right two images show the re-projected, tracked robot tool and its corresponding insertion shaft using the proposed Lumped Error parameter reduction technique.
	}
    \label{fig:cover_figure}
\end{figure}

A secondary source of uncertainty that can occur during robotic control is the  inaccuracies in joint angle measurements.
Errors in joint angle measurements are caused by biases in positioning, drifting in readings, and complex transmission effects such as cable stretch and backlash.
Similar to finding the base-to-camera transform, this is typically solved through calibration where a separate sensor, such as a camera, collects ground truth measurements and compares against the joint angle readings \cite{pradeep2014calibrating}.
For non-constant errors, such as cable stretch, explicit dynamic modelling has been conducted \cite{miyasaka2020modeling} and data-driven approaches with neural networks \cite{hwang2020efficiently}.
These methods however are challenging to apply outside of a lab setting due to the need for additional sensors or calibration steps.
Furthermore, the calibration parameters can degrade over time through irreversible effects from transmission wear-and-tear and mechanical creep.

A strong motivational example for scenarios with challenging base-to-camera calibration and errors in joint angle measurements are surgical robotic-endoscopic platforms \cite{hagn2010dlr, dimaio2011vinci, sport_surgical}.
The endoscopes are designed to only capture a small working space for higher operational precision.
Surgical robotic platforms also typically use cable-drives to enable low-profile robotic tools hence resulting in joint angle measurement error.
Furthermore, the bases of the surgical manipulators are adjusted regularly depending on the type of procedure and to fit each patients anatomy.
There is a significant amount of previous literature from the surgical robotics community tackling these two problems separately, and we unify these two problems and present a solution which also generalizes to other robot manipulators.

\subsection{Contributions}

In this work, we demonstrate the ability to track robotic tools from visual observations that only show part of the kinematic chain under the conditions of uncertainty in base-to-camera transform and joint angle measurements.
To this end, we present the following novel contributions: 
\begin{enumerate}
    \item a novel problem formulation which proves direct estimation of all the described parameters is infeasible since they are non-identifiable,
    \item the first approach to unify these uncertainties into a smaller set of parameters which are identifiable and compensates for all the uncertainties,
    \item an extension of tracking under simultaneously moving robotic tools and cameras.
\end{enumerate}
We coin the reduced set of parameters as \textit{Lumped Error}.
To track the Lumped Error, a tracking algorithm based on a particle filter is presented.
The particle filter uses visual features from the tracked robotic tool to continuously update the belief of the Lumped Error.
The visual features used in our implementation are detected using markers, edge detectors for geometric primitives such as cylinders, and learned point features.
For experimentation, the presented particle filter was evaluated both in simulation and on real world robotic data using the da Vinci Research Kit (dVRK) \cite{DVRK}, a widely used surgical robotics research platform  with a total of 10 Degrees of Freedom (DoF) and a gripper across both the endoscopic (i.e. surgical robotic camera arm) and robotic manipulator kinematic chains, and the 7 DoF arm on Rethink Robotic's Baxter robot.
From this set of experiments, the joint angle disturbances include simulated noise, cable stretch from the dVRK, and backlash from the Baxter arm.
Lastly, we summarized our previous work and their results which tracked the Lumped Error for applications in control of surgical robotics, such as autonomous suction and suture needle manipulation, highlighting the impact this method already has in the surgical robotic community.
This summary is written in Appendix \ref{appendix:previou_work}.
Overall, these results show that estimation of the Lumped Error is efficient and yields precise and accurate robotic tool tracking.

\subsection{Related Work}
Integration of visual observations with robotic manipulators and handling inaccurate joint angle measurements is not a new concept.
Therefore, this related work section is split into different categories to cover the wide-range of solutions presented by previous groups for robotic tracking.
A special focus is given to surgical robotics as the challenge of surgical robotic tool tracking would be a direct application of the presented work.

\subsubsection{Base-to-camera estimation}

A common approach to calibrating the base-to-camera transform is rigidly attaching a marker whose pose can be directly estimated from visual data (e.g. ArUco \cite{garrido2014automatic}, ARTag \cite{fiala2005artag}, AprilTag \cite{olson2011apriltag}, and STag \cite{benligiray2019stag}), collect multiple images of the marker, and solve the homogeneous linear system \cite{fassi2005hand, park1994robot}.
To relieve the heavy reliance on 3D pose reconstruction, which can often be inaccurate from 2D images, markers have been attached to robot manipulators to collect 2D keypoints and the camera-to-base transform is estimated with Solve-PnP \cite{lepetit2009epnp}.
Deep learning approaches have been applied to detect 2D keypoints on robotic manipulators to remove the need for markers \cite{lambrecht2019towards, lambrecht2019robust, lee2020camera, lu2020robust}.
However, these calibration methods do not consider errors in joint angle measurements and instead make the assumption that the robot kinematic chain is located exactly at the joint angle readings.

Zhong et al. proposed an interactive method to maximize the accuracy in calibration for remote center of motion (RCM) robots \cite{zhong2020hand} which are typical for laparoscopic robots.
Similarly, Zhao et al. defined a kinematic remote-center coordinate system (KCS) which absorbs all the error in the transform from the camera frame to the base of an RCM robot \cite{zhao2015efficient}.
Tracking the KCS is a common technique in surgical robotics where the updates come from markers \cite{super}, learned features \cite{reiter2012feature, original_rcs, lu2020super}, silhouette matching \cite{Hao}, or online template matching \cite{template_matching_EKF}.
All of these estimation methods do not explicitly consider the effects of joint angle errors.
In fact, we show that the Lumped Error is mathematically equivalent to tracking the KCS implying that these methods are compensating for both the base-to-camera transform and joint angle errors.
Furthermore, in this work we generalize Lumped Error to other robotic manipulators.
In Appendix \ref{appendix:previou_work}, we summarize our previous work in control of surgical robotics which utilized tracking KCS, or equivalently Lumped Error, to highlight the impact this method already has on the surgical robotics community.

\subsubsection{Joint measurement error estimation}

Using fiducial markers to collect data, Pastor et al. applied a data-driven approach to estimating the joint angle error \cite{pastor2013learning}.
Meanwhile Wang et al. used markers to estimate the joint angle offsets in real-time via inverse kinematics \cite{wang2013online}.
From the perspective of surgical robotics, errors of joint angle readings due to transmissions effects has largely been studied in the context of cable drives.
Miyasaka et al. explicitly modelled the physical effects of cable transmissions such as friction and hystresis \cite{miyasaka2020modeling}.
Learning based approaches have been applied in the form of neural networks for direct estimation of cable stretch \cite{hwang2020efficiently, peng2019real} and Gaussian processes for compensation \cite{mahler2014learning}.
From visual data, Unscented Kalman Filter \cite{haghighipanah2016unscented} and neural networks \cite{cable_error_comp} were applied to estimate the effects of cable stretch. 
These techniques however are impractical to apply outside of a lab setting due to the need for additional sensors or calibration steps.
In addition, calibration parameters can degrade over time due to mechanical effects such as cable creep which is when cable stretch varies irreversibly through usage.

\subsubsection{Combined base-to-camera and joint estimation}
Joint calibration techniques have been proposed which optimize for both joint angle offsets and base-to-camera transformations \cite{pradeep2014calibrating, le2009joint}.
To handle dynamic uncertainties, such as non-constant joint angle errors, real-time estimation by combining iterative closest point from depth sensing and Kalman Filtering have been proposed \cite{krainin2010manipulator}.
A probabilistic approach has also been proposed where the observation models are grounded in physical parameters hence making it easier to tune the hyper parameters \cite{cifuentes2016probabilistic}.
These works largely focus on integration of sensors into real-time estimation.
Instead of this, we look into parameter reductions for the case of partially visible kinematic chains.
Therefore, the proposed Lumped Error parameter reduction can aid these efforts by reducing the total number of parameters that need to be estimated in the case of a partially visible kinematic chain.
Nonetheless, we propose a particle filter to estimate the Lumped Error which relies only on image data; meanwhile, the efforts above rely on depth sensing which is not as readily available on all robotic platforms such as the da Vinci\textregistered{} Surgical System.

A separate and popular approach to controlling a robot from visual feedback without the base-to-camera transform is through online jacobian estimation \cite{shen2003asymptotic}.
These visual servoing techniques can even compensate for kinematic inaccuracies and joint angle measurement errors \cite{cheah2003approximate}.
While these techniques are sufficient to control the end-effector in the camera frame, they do not describe the rest of the kinematic chain in the camera frame.

\subsubsection{Eye-in-hand configuration}
Another consideration in the case of a robotic camera arm is the problem of \textit{eye-in-hand} calibration \cite{shiu1989calibration}.
This particular visual-robotic challenge is not considered here but included for the sake of completeness.
Zhang et al. developed a computationally efficient methods using dual quaternions \cite{zhang2017computationally}.
Adjoint transformations from twist motions have also been applied to converge to solutions with high accuracy \cite{calibrate_2, pachtrachai2018adjoint}.
In the case of RCM robots, reduction of the computational complexity has been found \cite{pachtrachai2019hand, calibrate_1}.
Similar to the previously described visual servoing techniques, the robot camera arm can also be controlled through online jacobian estimation \cite{piepmeier2003uncalibrated, verghese2019model}.

\section{Methods}

The problem formulation for base-to-camera and joint angle measurement errors with camera interaction is first explained in this section.
To overcome the described challenges for a partially observed robot from visual observations, a Lumped Error is derived and then extended to the eye-in-hand case.
Finally, our proposed method for tracking the Lumped Error with a particle filter approach is described.

\begin{figure}[t]
	\centering
	\vspace{2mm}
	\includegraphics[trim=0cm 3.8cm 13.3cm 0cm clip, width=8.5cm]{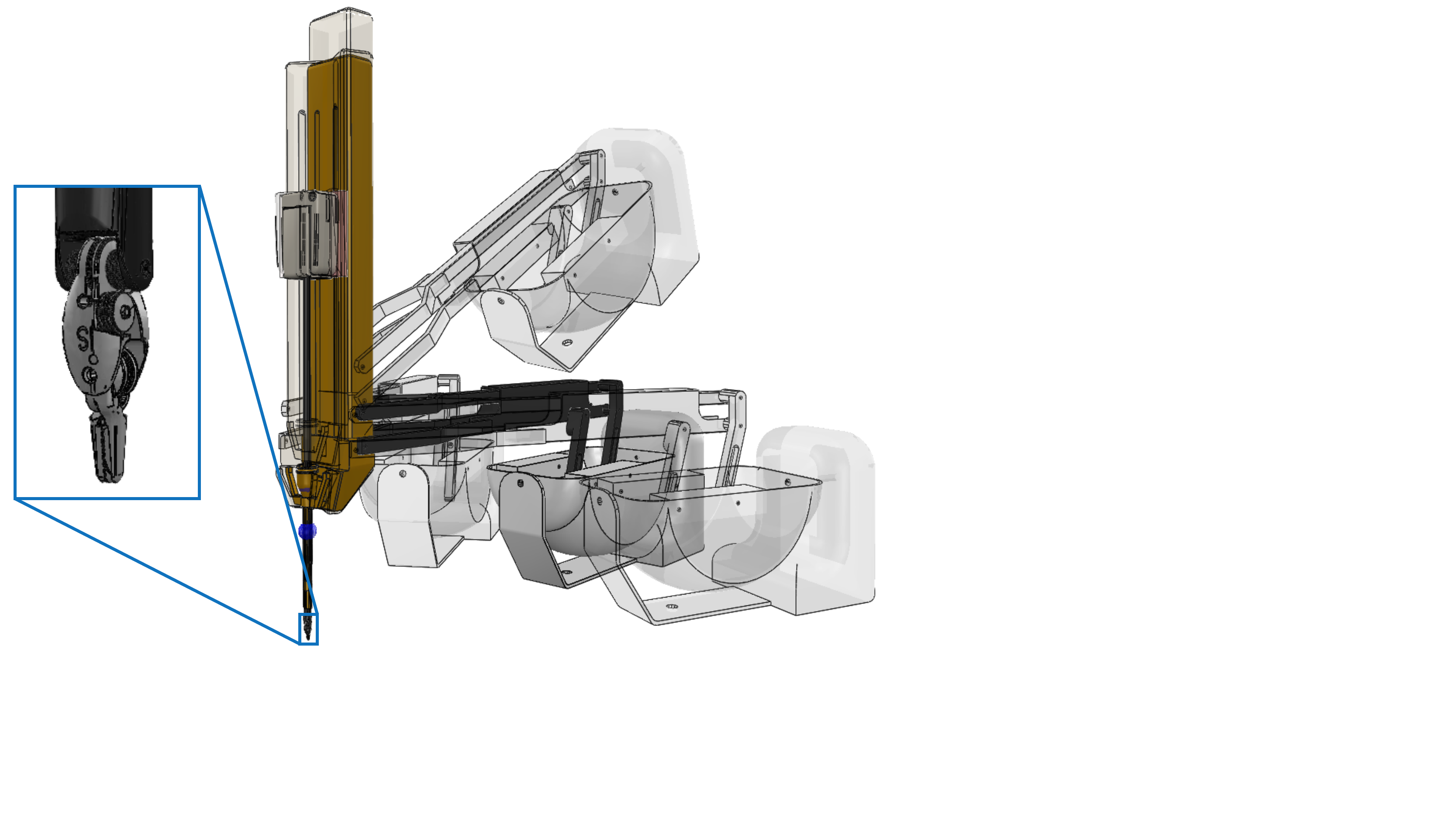}
	\caption{Given an image of the robot tool, as shown in blue, and not the whole kinematic chain, multiple solutions exist for joint angles and base-to-camera transform errors. Examples of these solutions are shown with the transparent kinematic chains. This implies that it is infeasible to estimate these unknowns directly when the kinematic chain is partially visible.
	}
    \label{fig:PSMMult}
\end{figure}

\subsection{Problem Formulation}
The 3D geometry of a robotic tool can be fully described in the stationary camera frame through a base-to-camera transform and forward kinematics.
A single point, $\mathbf{o}^j \in \mathbb{R}^3$, on the $j$-th link of a robotic tool can be transformed to the stationary camera frame by:
\begin{equation}
    \overline{\mathbf{o}}^{c}_t = \mathbf{T}^{c}_b \prod \limits_{i=1}^{j} \mathbf{T}_i^{i-1}(q^i_t) \overline{\mathbf{o}}^j
    \label{eq:perfect_fk}
\end{equation}
at time $t$ where $\mathbf{T}^c_b \in SE(3)$ is the base-to-camera transform and $\mathbf{T}_i^{i-1}(q^i_t) \in SE(3)$ is the $i$-th joint transform with joint angle $q^i_t$.
The overline operator ($\overline{\cdot}$) defines the homogeneous representation of a 3D point (e.g $\overline{\mathbf{o}} = \begin{bmatrix} \mathbf{o} & 1 \end{bmatrix}^\top$).
Therefore, all that is necessary to describe a robotic manipulator in the camera frame is the base-to-camera transform and joint angles.
Typically, the base-to-camera transform can be calibrated for, and the joint angles can be found from encoder readings.
The issue with applying this approach directly to scenarios where the camera only captures images with a portion of the kinematic chain is that small errors in calibration or joint angles will be exacerbated in the image frame.

Therefore, let $\tilde{q}^1_t, \dots, \tilde{q}^{n_j}_t$ be the joint angle measurements and $e^1_t, \dots, e^{n_j}_t$ are the measurement errors, such that:
\begin{equation}
    q^i_t = \tilde{q}^i_t + e^i_t
\end{equation}
for all $i = 1, \dots, n_j$.
No distribution is assumed for the errors, $e^i_t$.
For example, the error could be constant bias from absolute position error or non-constant with hysteresis effects from cable stretch. 
Combining with (\ref{eq:perfect_fk}), the robotic tool can be described in the camera frame by:
\begin{equation}
    \overline{\mathbf{o}}^{c}_t = \mathbf{T}^{c}_{b-} \mathbf{T}^{b-}_{b} \prod \limits_{i=1}^{j} \mathbf{T}_i^{i-1}(\tilde{q}^i_t + e^i_t) \overline{\mathbf{o}}^j
    \label{eq:not_perfect_fk}
\end{equation}
where the true base-to-camera transform is broken into $\mathbf{T}^{c}_{b-} \in SE(3)$ and $\mathbf{T}^{b-}_{b} \in SE(3)$ which are measured from an initial calibration and the error in the calibration respectively.
Therefore, in order to correctly describe the robotic tool in the camera frame, both the joint angle errors, $e^i_t$, and the error in base-to-camera transform, $\mathbf{T}^{b-}_{b}$, need to be estimated.
Let $n_j$ be the total number of joint angles and the $SE(3)$ error transform, $\mathbf{T}^{b-}_{b}$, be estimated with an axis-angle and a translation vector, resulting a total of $n_j + 6$ parameters to estimate. 

Explicit estimation for the joint angles and base-to-camera transform is not possible when only a portion of the kinematic chain is visible in the camera frame.
This is since one cannot distinguish where the source of error is coming from, joint angles or the base-to-camera transform calibration.
For example, a surgical tool is considered partially visible with regards to the endoscopic camera.
The endoscopes narrow field only has visual information of the tool-tip and not the base nor joints preceding the articulated wrist resulting in multiple viable solutions to estimating $\mathbf{T}^{b-}_b$ and $e^i_t$.
An example of this is shown in Fig. \ref{fig:PSMMult} where the joints part of the RCM are not visible to the endoscopic camera, but the joints on the gripper are. 

This relates to the concept of \textit{identifiability} \cite{parameter_identification}.
Identifiability is concerned with the existence of a unique inverse association with regards to the parameters estimated from observations.
Fig. \ref{fig:PSMMult} shows examples of there not being a unique association from the image of the surgical tool from an endoscope to the base-to-camera transform and errors in joint angles.
These instances are denoted as \textit{observationally equivalent}.
Parameters are only considered identifiable if there are no observational equivalences.

\begin{claim}
\label{claim:inf_solutions}
When only using the camera data for observations, then the error in base-to-camera transform, $\mathbf{T}^{b-}_b$, and errors in the first $n_b$ joint angles, $e^1_t, \dots, e^{n_b}_t$, are not identifiable if all the kinematic links preceding joint $n_b$ are out of the camera frame.
\end{claim}

\begin{proof}

Let \textit{Modified Denavit-Hartenberg Parameters} be used to define each forward kinematic joint transform.
Therefore: $\mathbf{T}_i^{i-1}(q^i_t) = \mathbf{T}_x(\alpha^i, a^i) \mathbf{T}_z(\theta^i, d^i)$ where
\begin{align*}
  \mathbf{T}_x(\alpha^i, a^i) &=
    \begin{bmatrix}
        1 & 0 & 0 & a^i \\
        0 & cos(\alpha^i) & -sin(\alpha^i) & 0 \\
        0 & sin(\alpha^i) & cos(\alpha^i) & 0 \\
        0 & 0 & 0 & 1
    \end{bmatrix} \\
    \mathbf{T}_z(\theta^i, d^i) &= 
    \begin{bmatrix}
        cos(\theta^i) & -sin(\theta^i) & 0 & 0 \\
        sin(\theta^i) & cos(\theta^i) & 0 & 0 \\
        0 & 0 & 1 & d^i \\
        0 & 0 & 0 & 1
    \end{bmatrix} 
\end{align*}
and $q^i_t$ is plugged into $\theta^i$ or $d^i$ for a revolute and prismatic joint respectively.
A revolute joint transform with a joint angle of $\omega + \psi \in \mathbb{R}$, $\mathbf{T}_i^{i-1}(\omega + \psi)$, can be expanded using the Modified Denavit-Hartenberg Parameters:
\begin{align}
    \label{eq:start_of_error_extraction}
    &\mathbf{T}_x(\alpha^i, a^i) \mathbf{T}_{z}(\omega + \psi, d^i)\\
    &\mathbf{T}_{x}(\alpha^i, a^i)\mathbf{T}_{z}(\omega, 0) \mathbf{T}^{-1}_{x}(\alpha^i, a^i)\mathbf{T}_{x}(\alpha^i, a^i)\mathbf{T}_{z}(\psi, d^i)\\
    &\mathbf{T}_i(\omega) \mathbf{T}_i^{i-1}(\psi)
\end{align}
where $\mathbf{T}_i(\omega) = \mathbf{T}_{x}(\alpha^i, a^i)\mathbf{T}_{z}(\omega, 0) \mathbf{T}^{-1}_{x}(\alpha^i, a^i)$. Likewise for a prismatic joint transform:
\begin{align}
    &\mathbf{T}_x(\alpha^i, a^i) \mathbf{T}_{z}(\theta^i, \omega + \psi)\\
    &\mathbf{T}_{x}(\alpha^i, a^i)\mathbf{T}_{z}(0, \omega) \mathbf{T}^{-1}_{x}(\alpha^i, a^i)\mathbf{T}_{x}(\alpha^i, a^i)\mathbf{T}_{z}(\theta^i, \psi)\\
    &\mathbf{T}_i(\omega) \mathbf{T}_i^{i-1}(\psi) \label{eq:end_of_error_extraction}
\end{align}
where $\mathbf{T}_i(\omega) = \mathbf{T}_{x}(\alpha^i, a^i)\mathbf{T}_{z}(0, \omega) \mathbf{T}^{-1}_{x}(\alpha^i, a^i)$.
Note that the same notation, $\mathbf{T}_i(\omega)$, is used for rotational and prismatic joints for simplified notation in the coming equations.

Using these expansions, we will show using induction that portions of the joint angle errors can be expanded out as follows:
\begin{equation}
    \prod \limits_{i=1}^{n_b} \mathbf{T}_i^{i-1}(\tilde{q}^i_t + e^i_t) = \mathbf{T}^{n_b} \prod \limits_{i=1}^{n_b} \mathbf{T}_i^{i-1}(\tilde{q}^i_t + \beta_i e^i_t)
    \label{eq:result_of_lump_proof}
\end{equation}
where
\begin{multline}
\label{eq:analyical_lump}
\mathbf{T}^{n_b} = 
   \prod \limits_{k=1}^{n_b}  \Big( \prod \limits_{i=1}^{k-1} \mathbf{T}_i^{i-1}(\tilde{q}^i_t + \beta_i e^i_t) \Big) \mathbf{T}_{k}((1-\beta_i)e^k_t) \\ \Big( \prod \limits_{i=1}^{k-1} \mathbf{T}_i^{i-1}(\tilde{q}^i_t+ \beta_i e^i_t) \Big)^{-1}
\end{multline}
and $\beta_i \in \mathbb{R}$ for $i=1,\dots,n_b$ is an arbitrary portion of the joint angle error not to be lumped into $\mathbf{T}^{n_b}$.
For the base case of $n_b = 1$ in (\ref{eq:result_of_lump_proof}), the error of the joint angle can be pulled out:
\begin{equation}
  \mathbf{T}_1^0(\tilde{q}^i_t + e^i_t) = \mathbf{T}_1((1 - \beta_1) e^i_t) \mathbf{T}_1^0(\tilde{q}^i_t + \beta_1 e^i_t)
\end{equation}
using (\ref{eq:start_of_error_extraction}) - (\ref{eq:end_of_error_extraction}). 

Assuming (\ref{eq:result_of_lump_proof}) holds true for $n_b = m$, then for $n_b = m+1$  the left-hand-expression from (\ref{eq:result_of_lump_proof}) can be rewritten as:
\begin{equation}
    \mathbf{T}^{m} \prod \limits_{i=1}^{m} \mathbf{T}_i^{i-1}(\tilde{q}^i_t + \beta_i e^i_t) \mathbf{T}_{m+1}^{m}(\tilde{q}^{m+1}_t + e^{m+1}_t)
\end{equation}
which expands to
\begin{multline}
    \mathbf{T}^{m} \prod \limits_{i=1}^{m} \mathbf{T}_i^{i-1}(\tilde{q}^i_t + \beta_i e^i_t) \mathbf{T}_{m+1}((1 - \beta_{m+1}) e^{m+1}_t)\\ \mathbf{T}^{m}_{m+1}(\tilde{q}^{m+1}_t + \beta_{m+1} e^{m+1}_t)
\end{multline}
using (\ref{eq:start_of_error_extraction}) - (\ref{eq:end_of_error_extraction}). Expanding the expression one more time:
\begin{multline}
    \mathbf{T}^{m} \prod \limits_{i=1}^{m} \mathbf{T}_i^{i-1}(\tilde{q}^i_t + \beta_i e^i_t) \mathbf{T}_{m+1}((1 - \beta_{m+1}) e^{m+1}_t)\\ \Big( \prod \limits_{i=1}^{m} \mathbf{T}_i^{i-1}(\tilde{q}^i_t + \beta_i e^i_t) \Big)^{-1} \prod \limits_{i=1}^{m} \mathbf{T}_i^{i-1}(\tilde{q}^i_t + \beta_i e^i_t)\\
    \mathbf{T}^{m}_{m+1}(\tilde{q}^{m+1}_t + \beta_{m+1} e^{m+1}_t)
\end{multline}
which is equivalent to (\ref{eq:result_of_lump_proof}). Therefore (\ref{eq:result_of_lump_proof}) holds for $n_b = 1, 2, \dots$ by mathematical induction.

Let $P(\mathbf{y}_t| \mathbf{T}^{b-}_b$, $e^1_t, \dots, e^{n_j}_t)$ be a proper probability distribution and the observation model of some feature $\mathbf{y}$ from the surgical tool in the camera frame parameterized by all the unknowns in the kinematic chain described in (\ref{eq:not_perfect_fk}).
Since $P(\mathbf{y}_t| \cdot)$ cannot describe a feature from the kinematic links preceding joint $n_b$, the observation model for some feature $\mathbf{y}$ can be re-parameterized to:
\begin{equation}
    \label{eq:many_to_one}
    P \Big(\mathbf{y}_t| \mathbf{T}^{c}_{b-} \mathbf{T}^{b-}_{b} \prod \limits_{i=1}^{n_b} \mathbf{T}_i^{i-1}(\tilde{q}^i_t + e^i_t), e^{n_b+1}_t, \dots, e^{n_j}_t \Big)
\end{equation}
The equality in (\ref{eq:result_of_lump_proof}) implies that the observation is not a one-to-one mapping from the parameter space $(\mathbf{T}^{b-}_{b}$, $e^1_t, \dots, e^{n_j}_t)$ to the camera observation (output of $P(\mathbf{y}_t| \cdot)$). In fact, for each observation generated by $P(\mathbf{y}_t| \cdot)$, there are an infinite solutions for the inverse mapping which are spanned by $\beta_1, \dots, \beta_{n_b}$. Since there are infinite observational equivalencies, the parameters are not identifiable.

\end{proof}

The equality in (\ref{eq:result_of_lump_proof}), which causes the lack of identifiability, can be interpreted as moving the joint errors from the kinematic chain to the base transform of the robot.
Lack of identifiability implies undesirable properties for parameter estimation such as rank deficiency in the Fischer Information Matrix \cite{parameter_identification}.
Furthermore, it shows the inability to estimate both errors in joint angles and base-to-camera transform.

\subsection{Lumped Error Transform for Estimation}
Due to Claim \ref{claim:inf_solutions}, it is infeasible to estimate all of the error parameters described in (\ref{eq:not_perfect_fk}) when only using camera data.
Therefore, we propose a parameter reduction technique where all the errors of the first $n_b$ joints are lumped together with the error in base-to-camera transform.
Hence, we call it the \textit{Lumped Error} Transform.

Using (\ref{eq:result_of_lump_proof}), (\ref{eq:not_perfect_fk}) can be re-written as:
\begin{equation}
    \label{eq:final_estimation_equation}
    \overline{\mathbf{o}}^{c}_t = \mathbf{T}^{c}_{b-} \mathbf{T}^{b-}_{n_b} (\mathbf{w}_t, \mathbf{b}_t)\prod \limits_{i=1}^{n_b}  \mathbf{T}_i^{i-1}(\tilde{q}^i_t) \prod \limits_{i=n_b+1}^{j}  \mathbf{T}_i^{i-1}(\tilde{q}^i_t + e^i_t)\overline{\mathbf{o}}^j
\end{equation}
where $\mathbf{T}^{b-}_{n_b}(\mathbf{w}_t, \mathbf{b}_t) \in SE(3)$ is the Lumped Error transform of all the first $n_b$ joint errors and the base-to-camera transform calibration error, $\mathbf{T}_b^{b-}$, and it is parameterized by an orientation, $\mathbf{w}_t \in \mathbb{R}^3$, and translation, $\mathbf{b}_t \in \mathbb{R}^3$.
The Lumped Error analytical solution from the joint angle errors and error in base-to-camera transform is $\mathbf{T}^{b-}_{n_b}(\mathbf{w}_t, \mathbf{b}_t) = \mathbf{T}_b^{b-} \mathbf{T}^{n_b}$ where $\mathbf{T}^{n_b}$ is defined in (\ref{eq:analyical_lump}) with $\beta_i = 0$ for $i=1,\dots,n_b$.

Intuitively, the Lumped Error transform is virtually adjusting the base of the kinematic chain for the robot in the camera frame.
The virtual adjustments are done to fit the error of the first $n_b$ joint angles and any error in base-to-camera transform.
The Lumped Error transform removes the many to one mapping shown in (\ref{eq:many_to_one}).
Furthermore, it is a significant reduction of the parameters that need to be estimated for robotic tool tracking.
With a total of $n_j$ joints and the $SE(3)$ error transforms, $\mathbf{T}_{b-}^b$ and $\mathbf{T}^{b-}_{n_b}(\mathbf{w}_t, \mathbf{b}_t)$, being estimated using axis-angle and a translation vector, then (\ref{eq:not_perfect_fk}) has $n_j + 6$ parameters to estimate while (\ref{eq:final_estimation_equation}) has $n_j-n_b + 6$ parameters. 

Even with this parameter reduction, it can still be challenging to constrain all of the parameters with image observations.
For example, from a single image frame, 4 pixel point detections are required to constrain the Lumped Error transform \cite{p3p_solution} and additional point detections would be needed for the joint errors $e^{n_b +1}_t, \dots, e^{n_j}_t$.
Therefore, we propose the following simplification to (\ref{eq:final_estimation_equation}) if there is not an abundance of features:
\begin{equation}
    \label{eq:simplification}
    e^{i}_t \approx 0 \text{ for } i = n_b+1, \dots, n_j
\end{equation}
With this simplification, only the Lumped Error transform needs to be estimated.
This simplification can be done in situations where the error from joints $n_b+1, \dots, n_j$ does not propagate through the kinematic chain dramatically.
In cases where the camera focuses on an articulated wrist or gripper as shown in Fig. \ref{fig:PSMMult}, this is an acceptable assumption as their link lengths are short hence reducing their sensitivity to error.

The resulting expression when combining the simplification in (\ref{eq:simplification}) with (\ref{eq:final_estimation_equation}) is equivalent to what previous literature in robotic surgical tool tracking described as the KCS which was developed for RCM based robots \cite{zhao2015efficient}.
Therefore, the KCS tracking formulation not only corrects for the error in base-to-camera transform, but also joint angle errors.

The Lumped Error can also be moved to the right hand-side of the first $n_b$ joint transforms in (\ref{eq:final_estimation_equation}) giving the following expression:
\begin{equation}
        \label{eq:right_handed_lumped_error}
        \overline{\mathbf{o}}^{c}_t = \mathbf{T}^{c}_{b-}  \prod \limits_{i=1}^{n_b} \mathbf{T}_i^{i-1}(\tilde{q}^i_t)  \mathbf{T}^{n_b, b-}_{n_b+1}(\mathbf{w}_t,\mathbf{b}_t) \prod \limits_{i=n_b+1}^{j}  \mathbf{T}_i^{i-1}(\tilde{q}^i_t + e^i_t)\overline{\mathbf{o}}^j
\end{equation}
where the right hand-side Lumped Error is:
\begin{equation}
    \label{eq:right_hand_lumped_analytical}
    \mathbf{T}^{n_b, b-}_{n_b+1}(\mathbf{w}_t,\mathbf{b}_t) = \Big( \prod \limits_{i=1}^{n_b} \mathbf{T}_i^{i-1}(\tilde{q}^i_t) \Big)^{-1} \mathbf{T}^{b-}_{n_b}(\mathbf{w}_t,\mathbf{b}_t) \prod \limits_{i=1}^{n_b} \mathbf{T}_i^{i-1}(\tilde{q}^i_t)
\end{equation}
which is equivalent to the tracking method proposed by Hao et. al. \cite{Hao} and shows their method compensates for both errors in base-to-camera transform and joint angle errors.

\subsection{Extension to Robotic Camera Arm}

In the case of eye-in-hand, the constant true base-to-camera transform, $\mathbf{T}^{c}_{b} = \mathbf{T}^{c}_{b-} \mathbf{T}^{b-}_{b}$, described in (\ref{eq:not_perfect_fk}) is replaced with a kinematic chain as follows:
\begin{equation}
    \overline{\mathbf{o}}^{c}_t = \mathbf{T}^c_{c_n} \Big( \prod \limits_{i=1}^{n}  \mathbf{T}^{c_{i-1}}_{c_{i}}(q^{c_i}_t) \Big)^{-1} \mathbf{T}^{c_b}_{b} \prod \limits_{i=1}^{j} \mathbf{T}_i^{i-1}(\tilde{q}^i_t + e^i_t) \overline{\mathbf{o}}^j
    \label{eq:fk_with_robotic_camera}
\end{equation}
where $\mathbf{T}^c_{c_n} \in SE(3)$ is the static transform from the final joint to the camera frame, $\mathbf{T}^{c_{i-1}}_{c_{i}}(q^{c_i}_t) \in SE(3)$ is the $i$-th joint transform of the camera arm with joint angle $q^{c_i}_t$, and $\mathbf{T}^{c_b}_{b} \in SE(3)$ is the base-to-base transform (i.e. transform from the base of the robotic tool to the base of the robotic camera arm).

Calibration of the base-to-base transform is even more challenging than calibrating the base-to-camera transform since the kinematic chain is extended by the camera arm.
Joint angle errors are also still assumed. 
Let $\tilde{q}^{c_i}_t$ and $e^{c_i}_t$ be the joint angle measurement and measurement error respectively for joint angle $c_i$ on the camera arm.
The base to base transform $\mathbf{T}^{c_b}_{b}$ is split into the calibrated base to base transform, $\mathbf{T}^{c_b}_{b-}$, and the error in calibration, $\mathbf{T}^{b-}_{b}$.
Therefore (\ref{eq:fk_with_robotic_camera}) is rewritten as:
\begin{multline}
    \overline{\mathbf{o}}^{c}_t = \mathbf{T}^c_{c_n} \Big( \prod \limits_{i=1}^{n}  \mathbf{T}^{c_{i-1}}_{c_{i}}(\tilde{q}^{c_i}_t + e^{c_i}_t) \Big)^{-1} \mathbf{T}^{c_b}_{b-} \mathbf{T}^{b-}_{b} \\ \prod \limits_{i=1}^{j} \mathbf{T}_i^{i-1}(\tilde{q}^i_t + e^i_t) \overline{\mathbf{o}}^j
    \label{eq:not_perfect_fk_with_robotic_camera}
\end{multline}
The kinematic links from the camera arm are typically not visible in the camera frame.
Therefore, the same non-identifiability issue from Claim \ref{claim:inf_solutions} extends to joint angle errors $e^{c_i}_t$ for $i=1,\dots c_n$.
To solve this issue, the Lumped Error from (\ref{eq:final_estimation_equation}) is applied to the camera arm's kinematic chain.
This results in:
\begin{multline}
    \overline{\mathbf{o}}^{c}_t = \mathbf{T}^c_{c_n} \Big( \prod \limits_{i=1}^{n}  \mathbf{T}^{c_{i-1}}_{c_{i}}(\tilde{q}^{c_i}_t) \Big)^{-1} \mathbf{T}^{c_b}_{c_n}(\mathbf{w}^c_t,\mathbf{b}^c_t)^{-1} \mathbf{T}^{c_b}_{b-} \\ \mathbf{T}^{b-}_{n_b}(\mathbf{w}_t,\mathbf{b}_t)  \prod \limits_{i=1}^{n_b}  \mathbf{T}_i^{i-1}(\tilde{q}^i_t) \prod \limits_{i=n_b+1}^{j}  \mathbf{T}_i^{i-1}(\tilde{q}^i_t + e^i_t)\overline{\mathbf{o}}^j
    \label{eq:intermediate_lumped_error_endoscopic_case}
\end{multline}
where $\mathbf{T}^{c_b}_{c_n}(\mathbf{w}^c_t,\mathbf{b}^c_t)$ analytical expression from joint angles is described in (\ref{eq:result_of_lump_proof}) with $\beta_i = 0$ for $i = 1,\dots, n$. 
Continuing further, (\ref{eq:intermediate_lumped_error_endoscopic_case}) can be reduced to a single unknown pose parameterized by orientation and translation vectors $\mathbf{w}^l_t, \mathbf{b}^l_t \in \mathbb{R}^3$ respectively and unknown joint errors $e^i_t$ for $i=n_b+1,\dots, n_j$. 
The new expression is: 
\begin{multline}
    \overline{\mathbf{o}}^{c}_t = \mathbf{T}^c_{c_n} \Big( \prod \limits_{i=1}^{n}  \mathbf{T}^{c_{i-1}}_{c_{i}}(\tilde{q}^{c_i}_t) \Big)^{-1}  \mathbf{T}^{c_b}_{b-} \mathbf{T}^{c_n}_{n_b} (\mathbf{w}^l_t,\mathbf{b}^l_t) \\ \prod \limits_{i=1}^{n_b} \mathbf{T}_i^{i-1}(\tilde{q}^i_t) \prod \limits_{i=n_b+1}^{j} \mathbf{T}_i^{i-1}(\tilde{q}^i_t + e^i_t)\overline{\mathbf{o}}^j
    \label{eq:lumped_error_endoscopic_case}
\end{multline}
where 
\begin{equation}
    \label{eq:explicit_endoscope_lumped_error}
    \mathbf{T}^{c_n}_{n_b} (\mathbf{w}^l_t,\mathbf{b}^l_t) = \big( \mathbf{T}^{c_b}_{b-} \big)^{-1} \mathbf{T}^{c_b}_{c_n}(\mathbf{w}^c_t,\mathbf{b}^c_t)^{-1} \mathbf{T}^{c_b}_{b-} \mathbf{T}^{b-}_{n_b}(\mathbf{w}_t,\mathbf{b}_t)
\end{equation}

The Lumped Error that would be estimated in this case, $\mathbf{T}^{c_n}_{n_b} (\mathbf{w}^l_t,\mathbf{b}^l_t)$, holds similar properties to the previous case of the stationary camera Lumped Error.
The base of the robotic manipulator relative to the camera arm base is virtually adjusted to compensate for the error in the first $n_b$ joint readings in it and all the joint readings in the robotic camera arm.
This Lumped Error also reduces the number of parameters that need to be estimated to $n_j - n_b + 6$ while in (\ref{eq:not_perfect_fk_with_robotic_camera}) there are $n_j + n +6$ unknown parameters.
For even fewer parameters to estimate, the simplification in (\ref{eq:simplification}) can be applied resulting in only 6 parameters.

\subsection{Particle Filter for Tracking of Lumped Error}

\begin{algorithm}[t]
    \caption{Particle Filter to Track Lumped Error}
     \label{alg:pf}
    \SetKwInOut{Input}{Input}
    \SetKwInOut{Output}{Output}
    \Input{Initial base-to-camera transform $\mathbf{T}^{c}_{b-}$}
    \Output{Estimated Lumped Error and Observable Joint Errors $\hat{\mathbf{w}}_{t}, \hat{\mathbf{b}}_{t}, \hat{\mathbf{e}}_t$}
    Initialize particle list $P_{0|0} = \{\alpha_{0|0}^{(p)}, \hat{\mathbf{w}}^{(p)}_{0|0}, \hat{\mathbf{b}}^{(p)}_{0|0}, \hat{\mathbf{e}}^{(p)}_{0|0}\}_{p=1}^N$ \\
    \tcp{Initialize particle distribution}
    \label{alg_pf:begin_initialize}
    \For{particle $p \in P_{0|0}$}{
        $\begin{bmatrix}
        \hat{\mathbf{w}}^{(p)}_{0|0}, \hat{\mathbf{b}}^{(p)}_{0|0} \end{bmatrix}^\top \sim \mathcal{N} ( \mathbf{0}, \mathbf{\Sigma}_{\mathbf{w},\mathbf{b},0} )$\\
        $\alpha_{0|0}^{(p)} \leftarrow \mathcal{G} \left( \begin{bmatrix} \hat{\mathbf{w}}^{(p)}_{0|0}, \hat{\mathbf{b}}^{(p)}_{0|0} \end{bmatrix}^\top, \mathbf{\Sigma}_{\mathbf{w},\mathbf{b},0}\right)$\\
        $\hat{\mathbf{e}}^{(p)}_{0|0} \sim \mathcal{U}(-\mathbf{a}_{\hat{\mathbf{e}}}, \mathbf{a}_{\hat{\mathbf{e}}})$\\
    }
    $\{ \alpha^{(p)}_{0|0} \}_{p=1}^N \leftarrow normalizeWeights \left( \{ \alpha^{(p)}_{0|0} \}_{k=1}^N \right)$\\
    \label{alg_pf:end_initialize}

    \tcp{Main Loop}
    \While{image and joint readings, $(\mathbf{I}_{t}$, $\tilde{\mathbf{q}}_{t})$, arrive}{
    \tcp{Predict}
    Initialize new particle list $P_{t|t-1}$\\
    \label{alg_pf:begin_predict}
    \For{particle $p \in P_{t|t-1}$}{
        $q \sim P_{t-1|t-1}$ weights $\{\alpha^{(1)}_{t-1|t-1} ,\dots , \alpha^{(N)}_{t-1|t-1}\}$\\
        $\begin{bmatrix} \hat{\mathbf{w}}^{(p)}_{t|t-1}, \hat{\mathbf{b}}^{(p)}_{t|t-1} \end{bmatrix}^\top \sim \mathcal{N} \left( \begin{bmatrix} \hat{\mathbf{w}}^{(q)}_{t-1|t-1}, \hat{\mathbf{b}}^{(q)}_{t-1|t-1} \end{bmatrix}^\top, \mathbf{\Sigma}_{\mathbf{w},\mathbf{b},t} \right)$\\
        $\alpha^{(p)}_{t|t-1} \leftarrow \mathcal{G} \left( \begin{bmatrix} \hat{\mathbf{w}}^{(p)}_{t|t-1}, \hat{\mathbf{b}}^{(p)}_{t|t-1} \end{bmatrix}^\top, \mathbf{\Sigma}_{\mathbf{w},\mathbf{b},t}\right) $ \\
        $\hat{\mathbf{e}}^{(p)}_{t|t-1} \sim \mathcal{N} \left( \hat{\mathbf{e}}^{(q)}_{t-1|t-1}, \mathbf{\Sigma}_{\hat{\mathbf{e}},t} \right)$\\
        $\alpha^{(p)}_{t|t-1} \leftarrow \alpha^{(p)}_{t|t-1} \cdot \mathcal{G} \left( \hat{\mathbf{e}}^{(p)}_{t|t-1}, \mathbf{\Sigma}_{\hat{\mathbf{e}},t} \right) $\\
    }
    \label{alg_pf:end_predict}
    \tcp{Update}
    $\mathbf{m}_{t} \leftarrow detectRobotPointFeatures(\mathbf{I}_{t})$\\
    $\boldsymbol{\rho}_{t}, \boldsymbol{\phi}_{t} \leftarrow detectRobotEdgeFeatures(\mathbf{I}_{t})$\\
    \For{particle $p \in P_{t|t-1}$}{
        $\hat{\mathbf{m}}_{t} \leftarrow projPoints( \hat{\mathbf{w}}^{(p)}_{t|t-1}, \hat{\mathbf{b}}^{(p)}_{t|t-1}, \hat{\mathbf{e}}^{(p)}_{t|t-1}, \tilde{\mathbf{q}}_{t})$\\
       \label{alg_pf:update_proj_point}
        $A_m, \mathbf{C}^m \leftarrow associatePoints(\mathbf{m}_{t}, \hat{\mathbf{m}}_{t})$\\
        \label{alg_pf:update_point_associate}
        $\alpha^{(p)}_{t|t-1} \leftarrow \alpha^{(p)}_{t|t-1} \cdot pointObsModel(A_m, \mathbf{C}^m)$\\
        \label{alg_pf:update_point_obs_mod}
        $\hat{\boldsymbol{\rho}}_{t}, \hat{\boldsymbol{\phi}}_{t} \leftarrow projEdges( \hat{\mathbf{w}}^{(p)}_{t|t-1}, \hat{\mathbf{b}}^{(p)}_{t|t-1}, \hat{\mathbf{e}}^{(p)}_{t|t-1}, \tilde{\mathbf{q}}_{t})$\\
        \label{alg_pf:update_proj_edge}
        $A_l, \mathbf{C}^l \leftarrow associateEdges([\boldsymbol{\rho}_{t}, \boldsymbol{\phi}_{t}],[ \hat{\boldsymbol{\rho}}_{t}, \hat{\boldsymbol{\phi}}_{t}])$\\
        \label{alg_pf:update_edge_associate}
        $\alpha^{(p)}_{t|t-1} \leftarrow \alpha^{(p)}_{t|t-1} \cdot edgeObsModel(A_l, \mathbf{C}^l)$\\
        \label{alg_pf:update_edge_obs_mod}
    }
    $P_{t|t} \leftarrow P_{t|t-1}$\\
    $\{ \alpha^{(p)}_{t|t} \}_{k=1}^N \leftarrow normalizeWeights \left( \{ \alpha^{(p)}_{t|t} \}_{k=1}^N \right)$\\
    \If{$numEffectiveParticles(P_{t|t}) > N_{eff}$}{
        $P_{t|t} \leftarrow stratifyResampling(P_{t|t})$
    }
    $\begin{bmatrix} \hat{\mathbf{w}}_{t}, \hat{\mathbf{b}}_{t}, \hat{\mathbf{e}}_t \end{bmatrix}^\top = \sum \limits_{p=1}^N \alpha^{(p)}_{t|t} \begin{bmatrix} \hat{\mathbf{w}}^{(p)}_{t|t}, \hat{\mathbf{b}}^{(p)}_{t|t}, \hat{\mathbf{e}}^{(p)}_{t|t} \end{bmatrix}^\top$\\
    }
 \end{algorithm}

The result in (\ref{eq:final_estimation_equation}) reduced the number of parameters that are required to be estimated to the Lumped Error transform, $\mathbf{T}_{n_b}^{b-}(\hat{\mathbf{w}}_t, \hat{\mathbf{b}}_t)$, and joint errors: $\hat{\mathbf{e}}_t := \begin{bmatrix} \hat{e}^{n_b+1}_t & \dots &\hat{e}^{n_j}_t \end{bmatrix}^{\top}$.
By these reductions, one can use previously developed methods of parameter estimation to track it such as the Extended Kalman Filter, Unscented Kalman Filter, or a particle filter with updates from camera images.
For our approach, we utilized a particle filter because of its flexibility to model the posterior probability density function with a finite-number of samples \cite{particle_filter_robotics} rather than using parametric model such as a Kalman Filter.
It also has found recent success in estimating poses \cite{robust_particle_filter} which is needed here for the Lumped Error transform.
The coming sections describe tracking of the parameters by defining the motion models and observation models.
The last section covers the few modifications necessary for the eye-in-hand case.
An outline of the proposed particle filter is shown in Algorithm \ref{alg:pf} and specific parameter values used in the experiments are described in Appendix \ref{appendix:particle_filter_details}.

\subsubsection{Motion Model}
The joint angle errors are initialized from a uniform distribution and have a motion model of additive zero mean Gaussian noise:
\begin{align}
    \hat{\mathbf{e}}_{0} \sim \mathcal{U}(-\mathbf{a}_{\hat{\mathbf{e}}}, \mathbf{a}_{\hat{\mathbf{e}}}) &&
    \hat{\mathbf{e}}_{t+1} \sim \mathcal{N} ( \hat{\mathbf{e}}_{t}, \mathbf{\Sigma}_{\hat{\mathbf{e}},t+1} )
    \label{eq:pf_joint_motion_model}
\end{align}
where $\mathbf{a}_{\hat{\mathbf{e}}} \in \mathbb{R}^{n_j - n_b}$ describes the bounds of constant joint angle error and $\mathbf{\Sigma}_{\hat{\mathbf{e}},t+1} \in \mathbb{R}^{(n_j - n_b) \times (n_j - n_b)}$ is a covariance matrix.
The initialization is done to capture joint angle biases, and a Weiner Process is chosen for the motion model due to its ability to generalize over a large number of random processes.

Let the tracked Lumped Error, $\mathbf{T}_{n_b}^{b-}(\hat{\mathbf{w}}_t, \hat{\mathbf{b}}_t)$ , be represented by an axis angle vector, $\hat{\mathbf{w}}_t \in \mathbb{R}^3$, and translation vector, $\hat{\mathbf{b}}_t \in \mathbb{R}^3$.
Their initialization and motions are defined as:
\begin{equation}
\label{eq:lumped_error_prediction}
\begin{split}
    \begin{bmatrix}
    \hat{\mathbf{w}}_{0}, \hat{\mathbf{b}}_{0}  \end{bmatrix}^\top &\sim \mathcal{N} ( \mathbf{0}, \mathbf{\Sigma}_{\mathbf{w},\mathbf{b},0} ) \\ \begin{bmatrix}
    \hat{\mathbf{w}}_{t+1}, \hat{\mathbf{b}}_{t+1}  \end{bmatrix}^\top &\sim \mathcal{N} (  \begin{bmatrix}
    \hat{\mathbf{w}}_{t}, \hat{\mathbf{b}}_{t}  \end{bmatrix}^\top,  \mathbf{\Sigma}_{\mathbf{w},\mathbf{b},t+1} )
\end{split}
\end{equation}
where $\mathbf{\Sigma}_{\mathbf{w},\mathbf{b},t} \in \mathbb{R}^{6 \times 6}$ is a covariance matrix.
A Weiner Process is once again chosen for the same reason as the joint angle error motion model.
Integration of the initial distribution and motion model in the particle filter is shown in lines \ref{alg_pf:begin_initialize} to \ref{alg_pf:end_initialize} and \ref{alg_pf:begin_predict} to \ref{alg_pf:end_predict} respectively in Algorithm \ref{alg:pf}.

\subsubsection{Observation Model}
To update the Lumped Error from images, features need to be detected and a corresponding observation model for them must be defined.
The coming observation models will generalize for any point or edge features.
Let $\mathbf{m}_t$ be a list of detected point features in the image frame from the projected robot tool.
By following the standard camera pin-hole model combined with (\ref{eq:final_estimation_equation}), the camera projection equation for the $k$-th point is:
\begin{multline}
    \label{eq:marker_camera_projection}
    \hat{\mathbf{m}}_k(\hat{\mathbf{w}}_t,\hat{\mathbf{b}}_t, \hat{\mathbf{e}}_t) = \frac{1}{s} \mathbf{K} \mathbf{T}^{c}_{b-} \mathbf{T}^{b-}_{n_b}(\hat{\mathbf{w}}_t,\hat{\mathbf{b}}_t) \prod \limits_{i=1}^{n_b}  \mathbf{T}_i^{i-1}(\tilde{q}^i_t) \\ \prod \limits_{i=n_b+1}^{j_k}  \mathbf{T}_i^{i-1}(\tilde{q}^i_t + \hat{e}^i_t)\overline{\mathbf{p}}^{j_k}
\end{multline}
where $\frac{1}{s} \mathbf{K}$ is the camera projection operator with intrinsic matrix $\mathbf{K}$ and known location $\mathbf{p}^{j_k}$ on joint link $j_k$. 

Similarly, let the paired lists $\boldsymbol{\rho}_t, \boldsymbol{\phi}_t$ be the parameters describing the detected edges in the image from the projected robot tool.
The parameters describe an edge in the image frame using the Hough Transform \cite{hough_transform}, so the $k$-th pair, $\rho^k_t$ and $\phi^k_t$ parameterize the $k$-th detected edge with the following equation:
\begin{equation}
    \label{eq:hough_shaft_parameters}
    \rho^k_t = u \text{ cos}(\phi^k_t) + v \text{ sin}(\phi^k_t)
\end{equation}
where $(u,v)$ are pixel coordinates.
Let the projection equations for the $i$-th edge be $\big(\hat{\rho}^i(\hat{\mathbf{w}}_t,\hat{\mathbf{b}}_t, \hat{\mathbf{e}}_t), \hat{\phi}^i(\hat{\mathbf{w}}_t,\hat{\mathbf{b}}_t, \hat{\mathbf{e}}_t)\big)$.
These projection equations will need to be defined based on the geometry of the robot.
An example of a cylindrical shape is shown in Appendix \ref{appendix:shaft_projection}.
Furthermore, Chaumette derived the projection equations for a multitude of geometric primitives and can be referred to for additional shapes \cite{project_cylinder}.
The point and edge projections are computed on lines \ref{alg_pf:update_proj_point} and \ref{alg_pf:update_proj_edge} respectively in Algorithm \ref{alg:pf}.

From the lists of detected features, there may be false detections, and they need to be associated with the correct point position ($\mathbf{p}^{j_i}$) or edge on the robot.
To accomplish this, a cost matrix, $\mathbf{C}^m$, is generated between the detected and projected features.
For the $k$-th detected point feature and $i$-th projected point, the cost is:
\begin{equation}
    C^m_{k,i} = \gamma_m || \mathbf{m}^k_t - \hat{\mathbf{m}}_i(\hat{\mathbf{w}}_t,\hat{\mathbf{b}}_t, \hat{\mathbf{e}}_t) ||^2 
\end{equation}
where $\gamma_m$ is a tuned parameter. Likewise a cost matrix, $\mathbf{C}^l$, is computed for the edges, and the $k$-th detected edge and $i$-th projected edge the cost is:
\begin{equation}
    C^l_{k,i} = \gamma_{\rho} | \rho^k_t -   \hat{\rho}_i ( \hat{\mathbf{w}}_t,\hat{\mathbf{b}}_t, \hat{\mathbf{e}}_t)| + \gamma_{\phi} | \phi^k_t -   \hat{\phi}_i ( \hat{\mathbf{w}}_t,\hat{\mathbf{b}}_t, \hat{\mathbf{e}}_t )|
\end{equation}
where $\gamma_{\rho}$ and $\gamma_{\phi}$ are tuned parameters. 

A greedy matching technique is used to make associations between the detected and projected features because of the computational efficiency.
The costs are sorted from lowest to highest, and the first $(k,i)$ pair is matched and added to the set $A_m$ or $A_l$ for points and edges respectively.
All subsequent costs associated with either detection $k$ or projection $i$ are removed from the sorted list.
This is repeated until a maximum cost of $C^m_{max}$ or $C^l_{max}$ is reached for points and edges respectively.
By limiting the maximum cost for association, false detections can be filtered out.
This association technique is conducted on lines \ref{alg_pf:update_point_associate} and \ref{alg_pf:update_edge_associate} in Algorithm \ref{alg:pf} for points and edges respectively.

The observation model wraps the associations and their costs into a probability function dependent on the state, so the filter can update the states properly.
For the list of point features, the probability is: 
\begin{equation}
    P( \mathbf{m}_{t}  | \hat{\mathbf{w}}_t, \hat{\mathbf{b}}_t, \hat{\mathbf{e}}_t ) \propto (n_m - |A_m|) e^{-C^m_{max}} + \sum \limits_{ k,i  \in A_m } e^{-C^m_{k,i}}
    \label{eq:pf_observation_model_markers}
\end{equation}
where there are a total of $n_m$ detectable point features on the robot.
Similarly, the probability of the list of detected edges is:
\begin{equation}
    P( \boldsymbol{\rho}_{t}, \boldsymbol{\phi}_{t} |\hat{\mathbf{w}}_t, \hat{\mathbf{b}}_t, \hat{\mathbf{e}}_t)  \propto (n_l - |A_l|)e^{-C^l_{max}}  + \sum \limits_{ k,i \in A_l } e^{-C^l_{k,i}}
\end{equation}
where there are a total of $n_l$ detectable edge features on the robot.
The probability distributions can be viewed as a summation of Gaussians centered about the projected features.
The individual Gaussian probabilities are bounded and clipped by the maximum cost for association.
Clipping the Gaussians is preferred since in the cases of missed feature detections, the posterior probability from the filter does not go to zero. 
An additional advantage of using a particle filter for tracking the Lumped Error is that these observation models do not need to be normalized.
In this case, finding the normalization factor would be challenging due to the matching complexity and clipping of Gaussians.
These observation models update the particle filter on lines \ref{alg_pf:update_point_obs_mod} and \ref{alg_pf:update_edge_obs_mod} in Algorithm \ref{alg:pf} for points and edges respectively.

\subsubsection{Modifications for Eye-in-Hand Configuration}

The explicitly tracked joint errors, $\hat{\mathbf{e}}_t$, remains the same since they are still the only joints visible in the camera frame.
However, the tracked pose is now $\mathbf{T}^{c_n}_{n_b}(\hat{\mathbf{w}}^l_t, \hat{\mathbf{b}}^l_t)$, described in (\ref{eq:lumped_error_endoscopic_case}).
The tracked parameters $\hat{\mathbf{w}}^l_t, \hat{\mathbf{b}}^l_t \in \mathbb{R}^3$ represent the Lumped Error as axis-angle and translation vectors respectively and have the same additive zero mean Gaussian noise as described in (\ref{eq:lumped_error_prediction}).
The feature detection, association, and observation models all remain the same.
The only change required is modifying the camera projection equations. 
The camera projection equation for the $i$-th marker is changed from (\ref{eq:marker_camera_projection}) to:
\begin{equation}
\begin{split}
     \hat{\mathbf{m}}_k(\hat{\mathbf{w}}^l_t, \hat{\mathbf{b}}^l_t, \hat{\mathbf{e}_t}) = \frac{1}{s} \mathbf{K} & \mathbf{T}^c_{c_n} \Big( \prod \limits_{i=1}^{n}  \mathbf{T}_{c_{i-1}}^{c_{i}}(\tilde{q}^{c_i}_t) \Big)^{-1}    \mathbf{T}^{c_b}_{b-} \\ \mathbf{T}_{n_b}^{c_n} (\hat{\mathbf{w}}^l_t, \hat{\mathbf{b}}^l_t) &\prod \limits_{i=1}^{n_b}  \mathbf{T}_i^{i-1}(\tilde{q}^i_t)  \prod \limits_{i=n_b+1}^{j_k}  \mathbf{T}_i^{i-1}(\tilde{q}^i_t + \hat{e}^i_t)\overline{\mathbf{p}}^{j_k}
\end{split}
\end{equation}
by combining (\ref{eq:lumped_error_endoscopic_case}) with the camera pin-hole model.
A similarly simple modification is required for the projected edges $\big(\hat{\rho}^i(\hat{\mathbf{w}}_t,\hat{\mathbf{b}}_t, \hat{\mathbf{e}}_t), \hat{\phi}^i(\hat{\mathbf{w}}_t,\hat{\mathbf{b}}_t, \hat{\mathbf{e}}_t)\big)$.
The example shown in Appendix \ref{appendix:shaft_projection} for cylindrical shapes includes the modifications necessary.

\section{Experiments and Results}

Since surgical robotic tool tracking is a direct application of this work, the proposed particle filter was used to track the Lumped Error in a simulated scene of a da Vinci\textregistered{} Surgical System and on a real world dVRK \cite{DVRK}.
The uncertainties of joint angles on the dVRK system are so prevalent that results relying only on base-to-camera calibration and not accounting for joint angle error were intentionally omitted in previous work due to such poor results \cite{cable_error_comp}.
Furthermore, in our own previously equivalent work, we experimented using either calibrated base-to-camera transform or active tracking to grasp chicken tissue detected in the camera frame \cite{super}.
Using the calibrated base-to-camera transform, the surgical tool was unable to grasp the chicken tissue.
Meanwhile with active tracking, the surgical tool was able to repeatedly grasp the chicken tissue.
The last experiment is tracking a partially visible Baxter robot arm which has significant backlash transmission effects.
These set of tests show the effectiveness of the proposed parameter reduction technique by comparing against different parameter sets.

\subsection{da Vinci Simulated Scene Setup}
A simulated scene in V-REP \cite{rohmer2013v} was developed based on the da Vinci robot model constructed in Fontanelli et al. \cite{vrep_simulator}.
The robotic tool and camera arm simulated were a Patient Side Manipulator (PSM) with a Large Needle Driver and an Endoscopic Camera Manipulator (ECM) respecitvely from the da Vinci\textregistered{} Surgical System.
The PSM has 6 DoF and an additional gripper joint. 
The ECM is stereoscopic and has 4 DoF.
The first joint link visible in the endoscopic camera frame was after the $n_b = 4$ joint as expected when operating with a da Vinci\textregistered{} Surgical System.

Small blue spheres were placed as markers along the kinematic links near the gripper to be used as point features to update the particle filter.
The blue markers were detected using standard color segmentation from OpenCV \cite{opencv_library}.
Each camera image was first converted to the Hue, Saturation, and Value (HSV) color space.
Hand-tuned lower and upper bounds for each HSV channel was then applied to the image resulting in a segmented binary image.
The segmented binary image was then clustered into distinct contours from which the centroids are estimated.
The list of centroids, $\mathbf{m}_t$, were considered detected pixel coordinate features potentially from projected points on the surgical tool.
The edges of the projected cylindrical insertion shaft of the PSM tool were also used to update the particle filter and detected using standard OpenCV functionality \cite{opencv_library}.
Each pixel potentially associated with the edges were detected using Canny edge detector \cite{canny_edge_detector}.
The pixels were further classified into distinct edges using the Hough Transform \cite{hough_transform} with parameters $\rho^k_t$ and $\phi^k_t$ to fit (\ref{eq:hough_shaft_parameters}).
The simulated scene and a corresponding camera image with the detected features is shown in Fig. \ref{fig:simuated_scene}.

\begin{figure} 
    \vspace{2mm}
    \centering
    \includegraphics[width=0.95\linewidth]{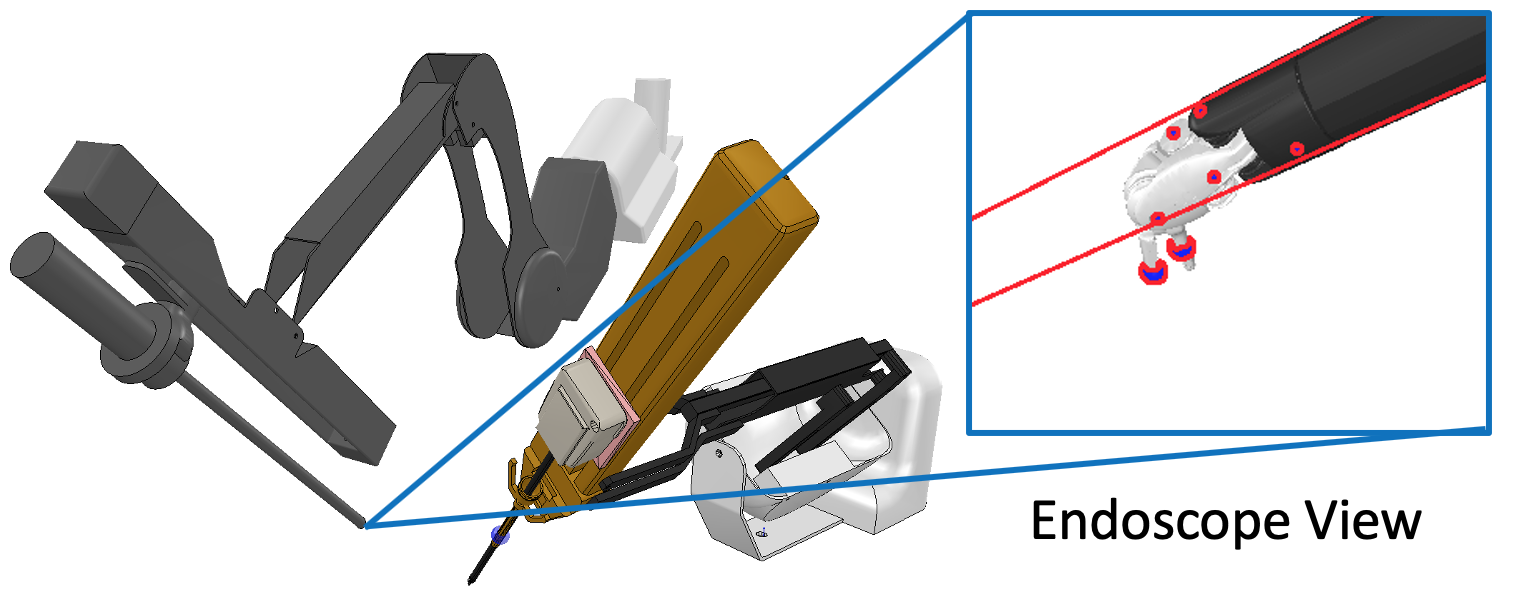}
    \caption{Simulated scene in V-REP \cite{vrep_simulator} of a Patient Side Manipulator (PSM) and Endoscopic Camera Manipulator (ECM) from a da Vinci\textregistered{} Surgical System. Blue markers are placed on PSM's gripper and detected for the particle filter as shown in red in the endoscopic view. Similarly, the detected edges of the insertion shaft are highlighted with red lines and are also used by the particle filter.}
    \label{fig:simuated_scene}
\end{figure}

\begin{figure*}
    \centering
    \vspace{2mm}
    \begin{subfigure}{1.0\linewidth}
    \centering
        \begin{subfigure}{0.48\textwidth}
        \includegraphics[width=\linewidth]{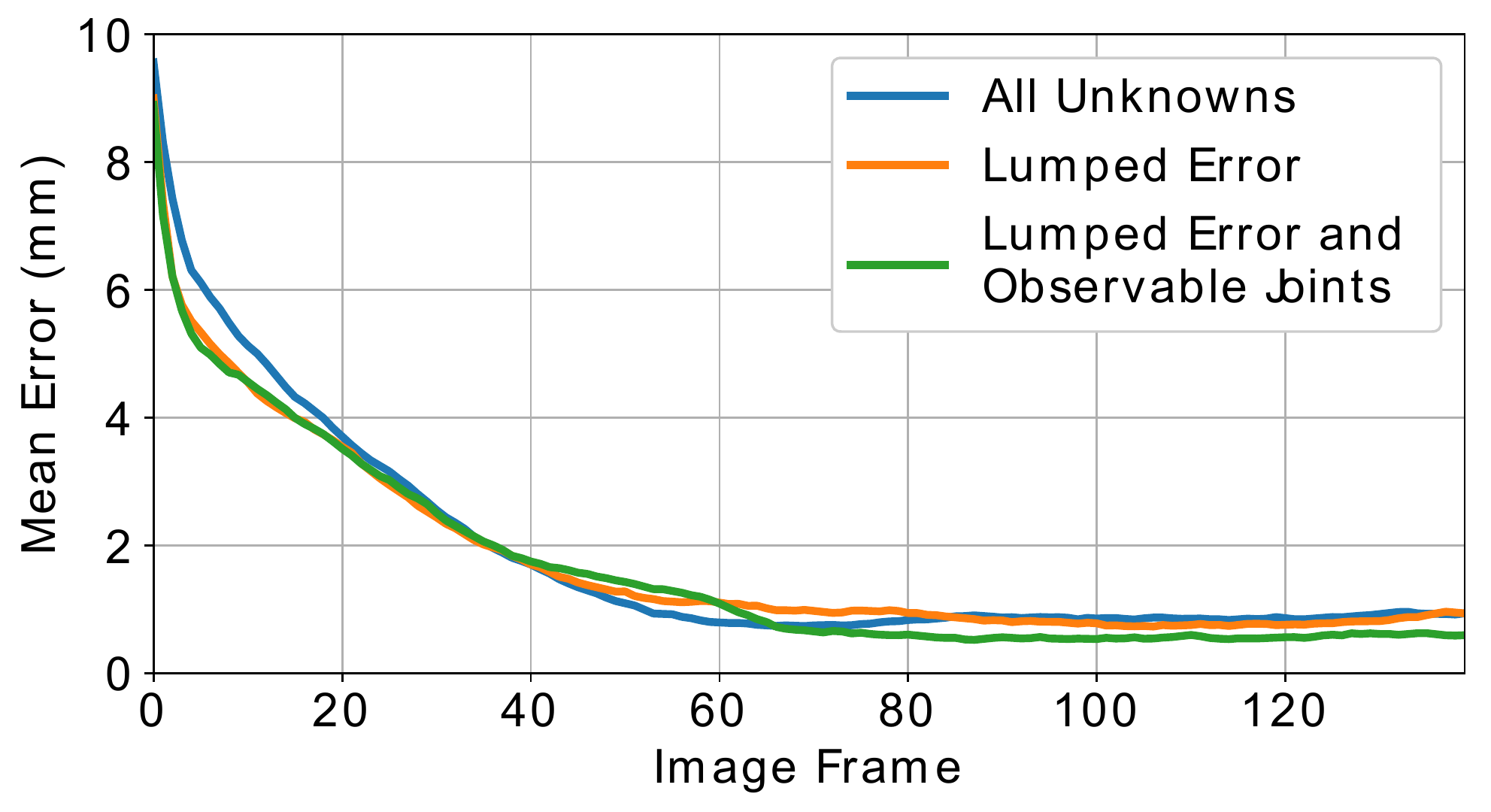}  
        \end{subfigure}
        \hspace{10pt}
        \begin{subfigure}{0.48\textwidth}
        \includegraphics[width=\linewidth]{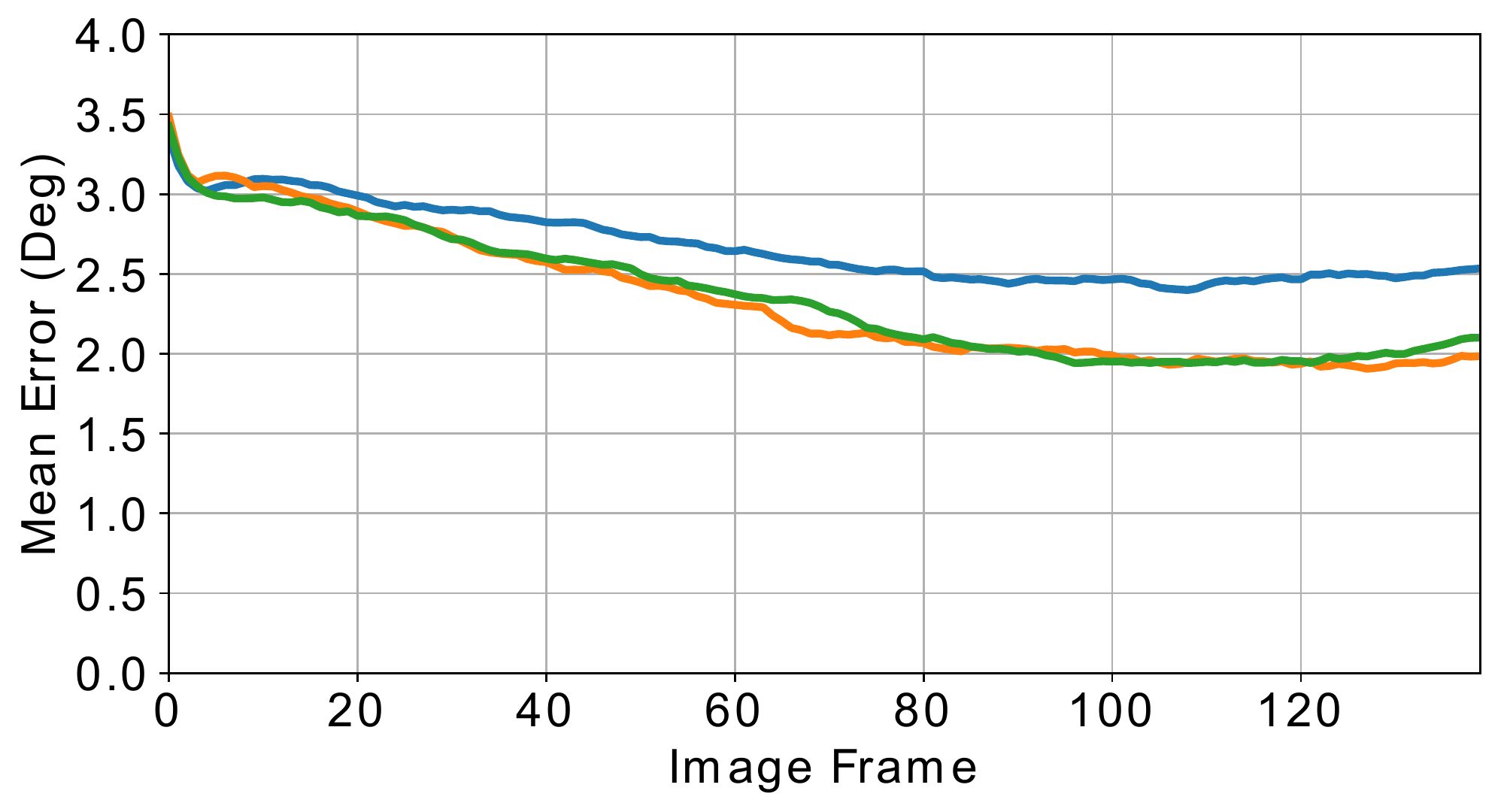}  
        \end{subfigure}
    \end{subfigure}
    
    \vspace{1mm}
    
        \begin{subfigure}{1.0\linewidth}
    \centering
        \begin{subfigure}{0.48\textwidth}
        \includegraphics[clip,width=\linewidth]{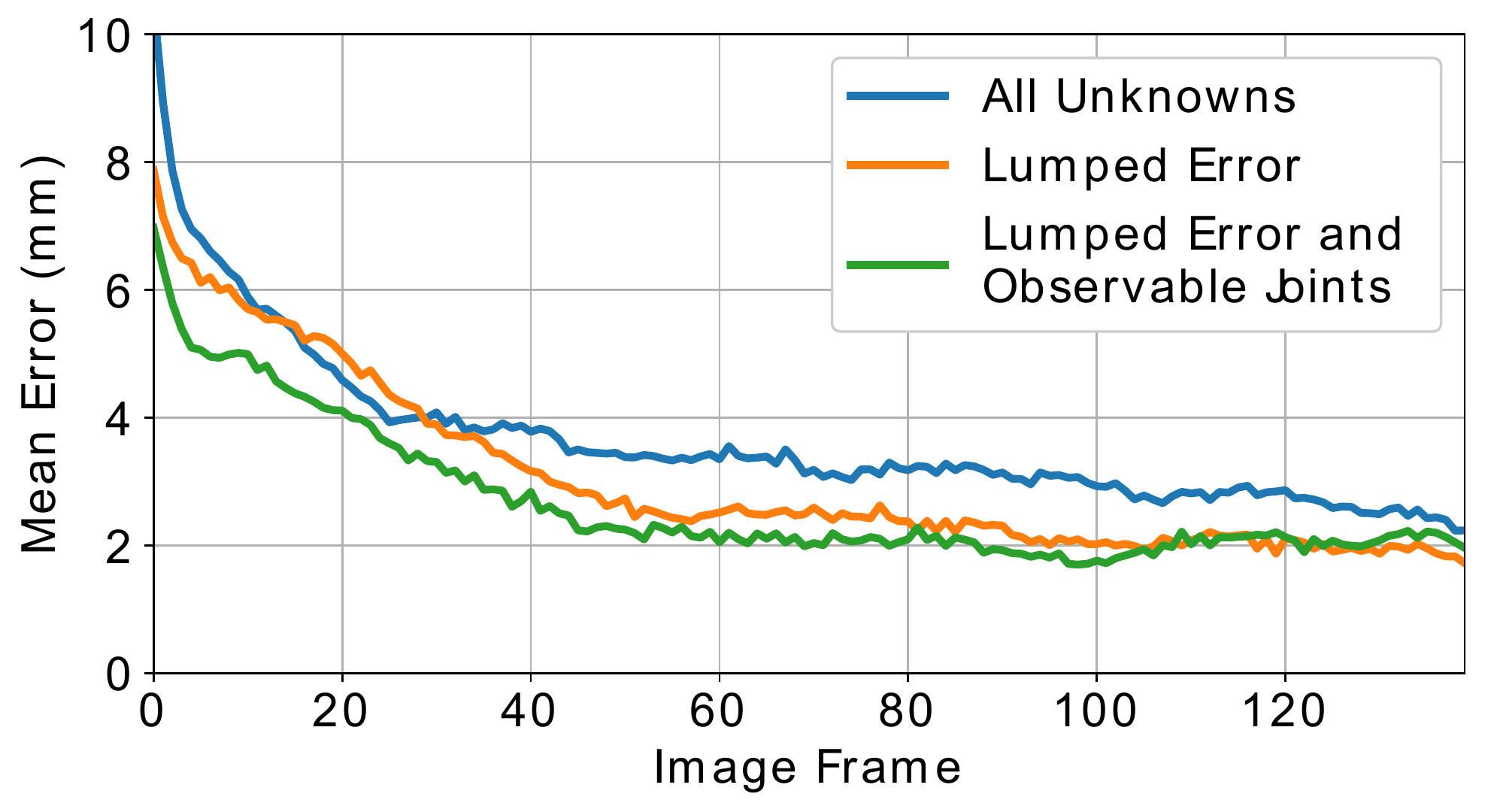}  
        \end{subfigure}
        \hspace{10pt}
        \begin{subfigure}{0.48\textwidth}
        \includegraphics[clip,width=\linewidth]{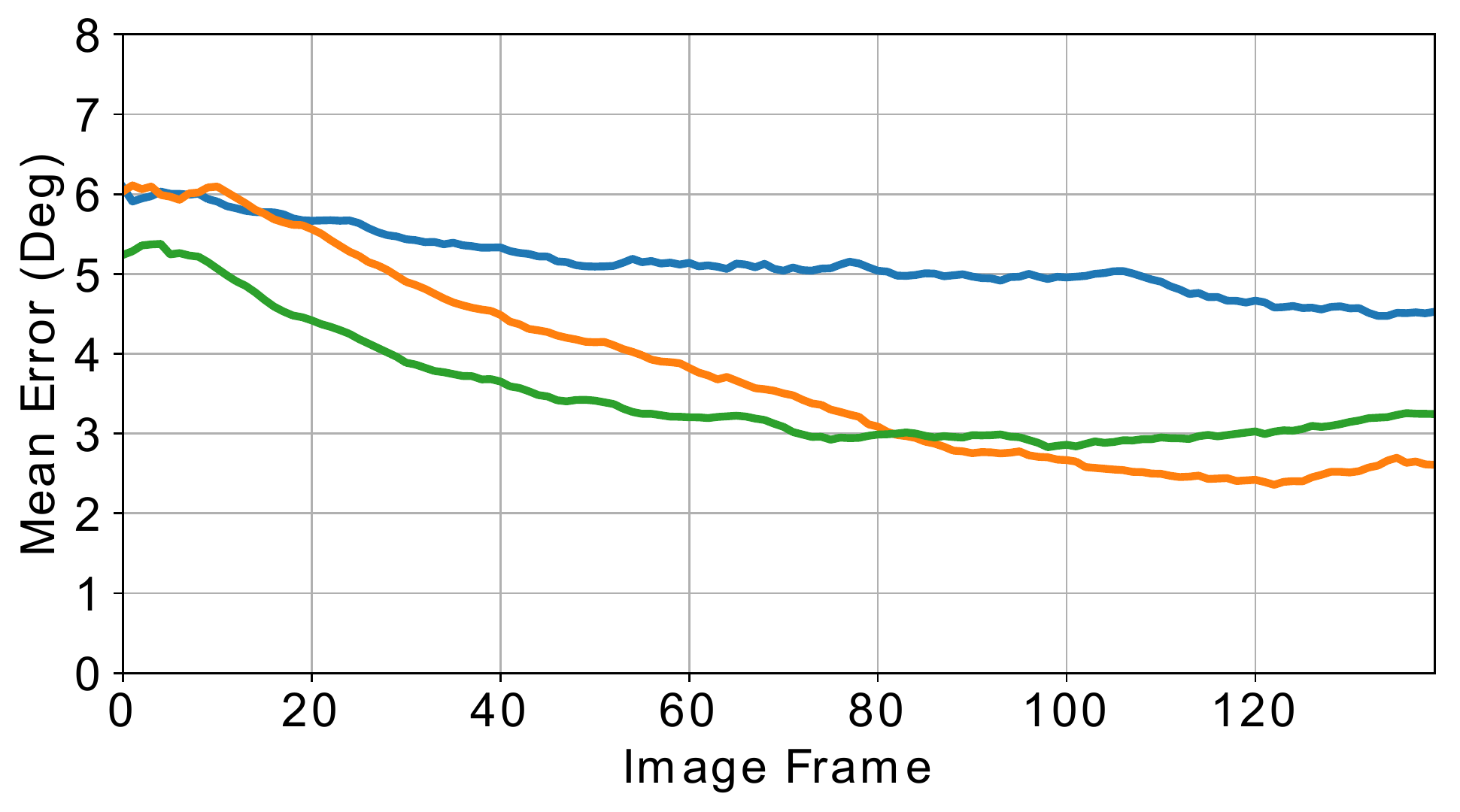}  
        \end{subfigure}
    \end{subfigure}
    \caption{
    Mean end-effector pose error in the camera frame over time under various tracking configurations from simulated da Vinci scene.
    The top and bottom row of plots are measured from the stationary and eye-in-hand cases respectively.
    These mean error trends are calculated with 50 trials and shows that tracking the Lumped Error results in a lower end-effector orientation error compared to tracking all unknowns.}
    \label{fig:pf_over_time_plots}
\end{figure*}

The error in calibration, $\mathbf{T}^{b-}_b$ was done by sampling from zero mean Gaussian in its axis angle and translation vector representations:
\begin{equation}
    \begin{bmatrix}
    \mathbf{w}^{b-}, \mathbf{b}^{b-}  \end{bmatrix}^\top \sim \mathcal{N} ( \mathbf{0}, \mathbf{\Sigma}^{b-}_{\mathbf{w},\mathbf{b}} )
    \label{eq:simulated_hand_eye_noise_1}
\end{equation}
Therefore, the initial calibration given to the filter was set to
\begin{equation}
\begin{split}
    \mathbf{T}^c_{b-} &=  \mathbf{T}^c_{b} \Big( \mathbf{T}^{b-}_b(\mathbf{w}^{b-}, \mathbf{b}^{b-}) \Big)^{-1}  \\ \mathbf{T}^{c_b}_{b-} &= \mathbf{T}^{c_b}_{b} \Big( \mathbf{T}^{b-}_b(\mathbf{w}^{b-}, \mathbf{b}^{b-}) \Big)^{-1}
    \label{eq:simulated_hand_eye_noise_2}
\end{split}
\end{equation}
for the stationary camera and eye-in-hand cases respectively where $\mathbf{T}^c_{b}$ and $\mathbf{T}^{c_b}_{b}$ were given by the simulator.

The joint error for the PSM was simulated as a summation between a uniformly sampled bias at the start of each trial and linear cable stretch.
Written explicitly, the error for joint angle $i$ was defined as:
\begin{equation}
    e^i_t = e^i_b + e^i_c q^i_t
    \label{eq:simulated_joint_noise_1}
\end{equation}
where $e^i_b \sim \mathcal{U}(-a^{i,b}_e, a^{i,b}_e)$, $e^i_c$ is the linear cable stretch coefficient, and $q^i_t$ is the correct joint angle from the PSM.
Similarly, joint error $c_i$ for the ECM was defined as:
\begin{equation}
    e^{c_i}_t = e^{c_i}_b + e^{c_i}_{l,t}
    \label{eq:simulated_joint_noise_2}
\end{equation}
where $e^{c_i}_b \sim \mathcal{U}(-a^{c_i,b}_e, a^{c_i,b}_e)$ was sampled once at the start of each trial and $e^{c_i}_{l,t} \sim \mathcal{N}(0, \sigma_{c_i,l}^2)$ sampled at every time step to simulate the uncertainties in the robotic endoscopes joint angles.

The PSM arms configuration was set via V-REP's inverse kinematics. 
Its position moved along a preset cyclical trajectory added with a small, random sample from a zero mean Gaussian with standard deviation of 1mm.
The gripper joint opened and closed at a similar cyclical rate.
Likewise, the four joint angles of the ECM were set to move in a cyclical pattern in the eye-in-hand case.
The orientation of PSM end-effector instead takes a random walk starting at a preset value by rotating an additional uniformly sampled rotation at every time step.
Note that this represents the most complex scenario where every part of the robot (manipulator, gripper, camera) are continuously moving on independent paths hence testing the proposed tracking method in a larger variety of scenarios including occlusions of features.
Additional details and parameter values are described in Appendix \ref{appendix:simulation_detials}.

\begin{figure}[t]
  \centering
    \begin{subfigure}{0.17\textwidth}
        \includegraphics[width=\linewidth]{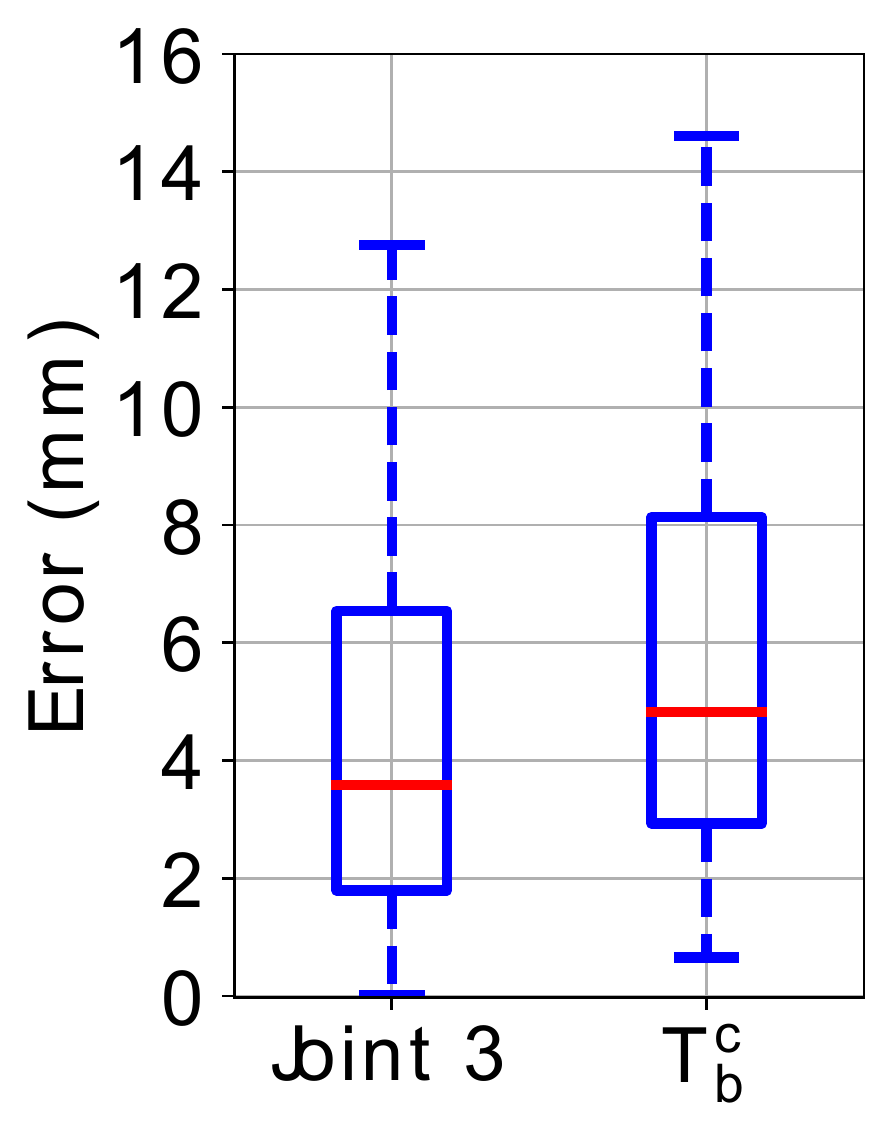}
    \end{subfigure}
    \begin{subfigure}{0.31\textwidth}
        \includegraphics[width=\linewidth]{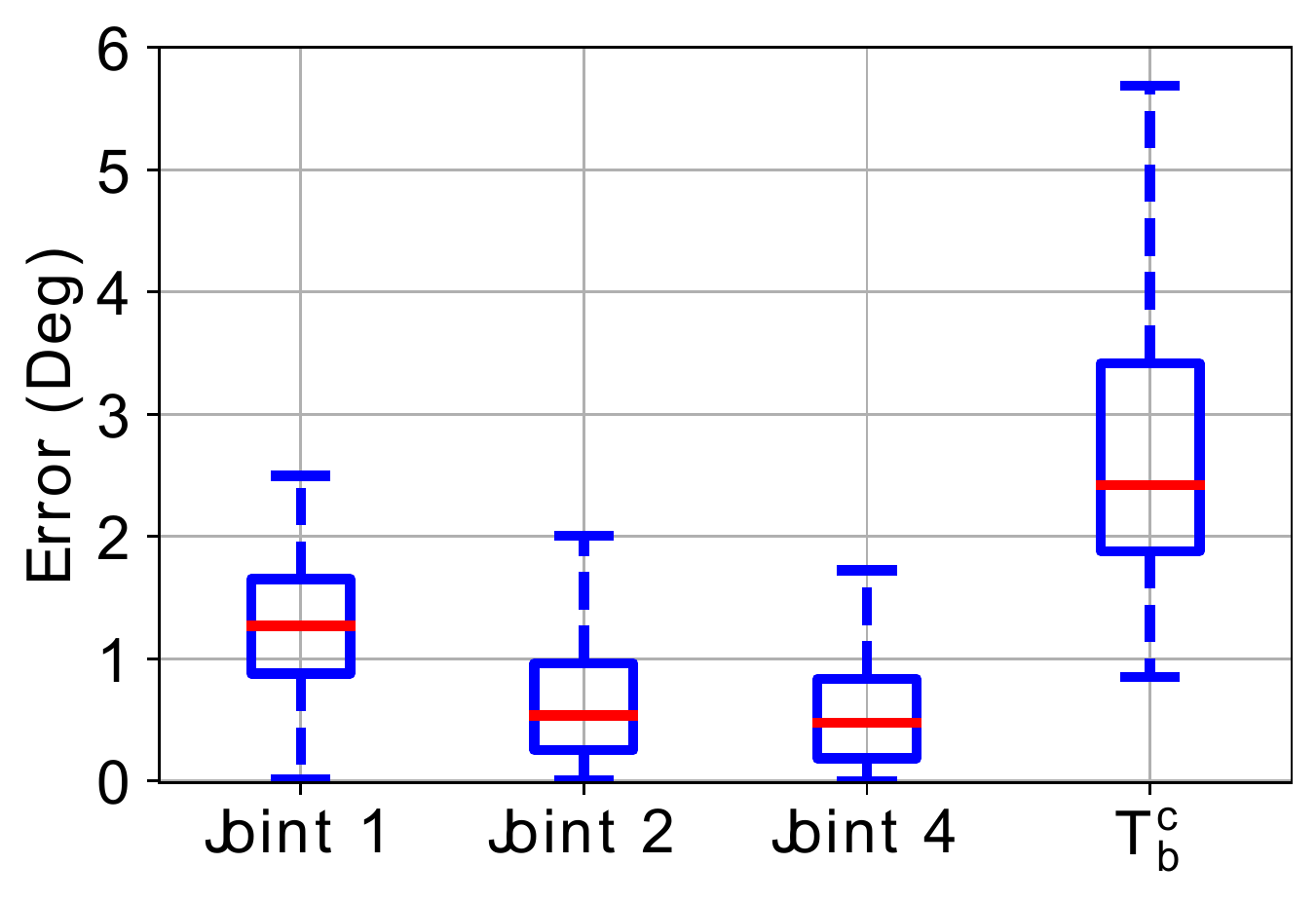}
    \end{subfigure}
    \caption{Distribution of first 4 joint angle errors, whose preceding kinematic links are never in the camera frame, and the stationary camera to base transform, $\mathbf{T}^{c}_b$, error when explicitly estimating them in the simulated da Vinci scene.
    Errors up to 14mm and 5 degrees highlight the inability to estimate these unknown values explicitly due to the parameters being non-identifiable as shown in Claim \ref{claim:inf_solutions}.
    }
    \label{fig:errors_when_tracking_everything_stationary}
\end{figure}

\begin{figure*}
    \centering
    \begin{subfigure}{1.0\textwidth}
      \centering
        \begin{subfigure}{0.3\textwidth}
           \includegraphics[width=\linewidth]{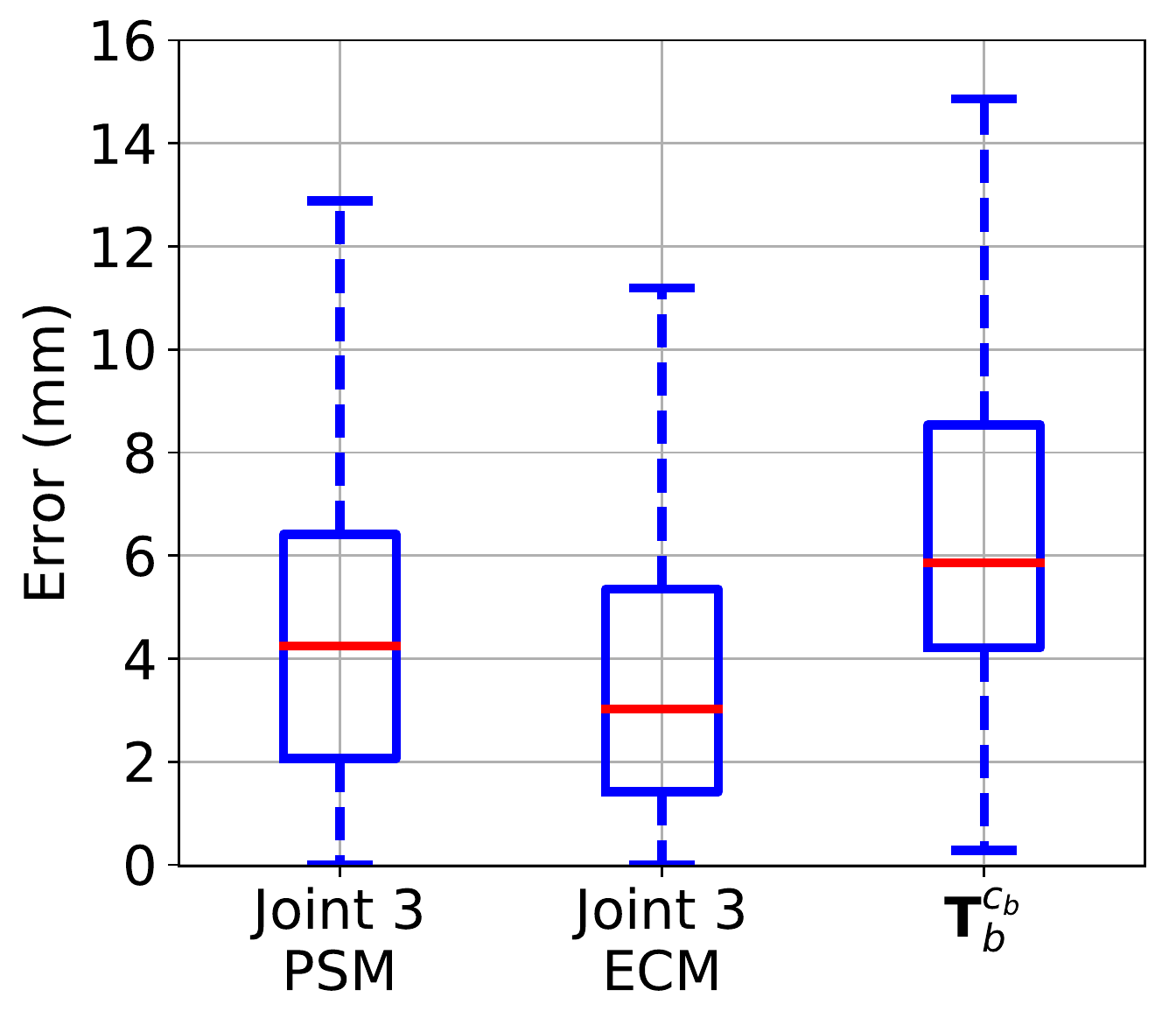}  
        \end{subfigure}
        \begin{subfigure}{0.6\textwidth}
            \includegraphics[width=\linewidth]{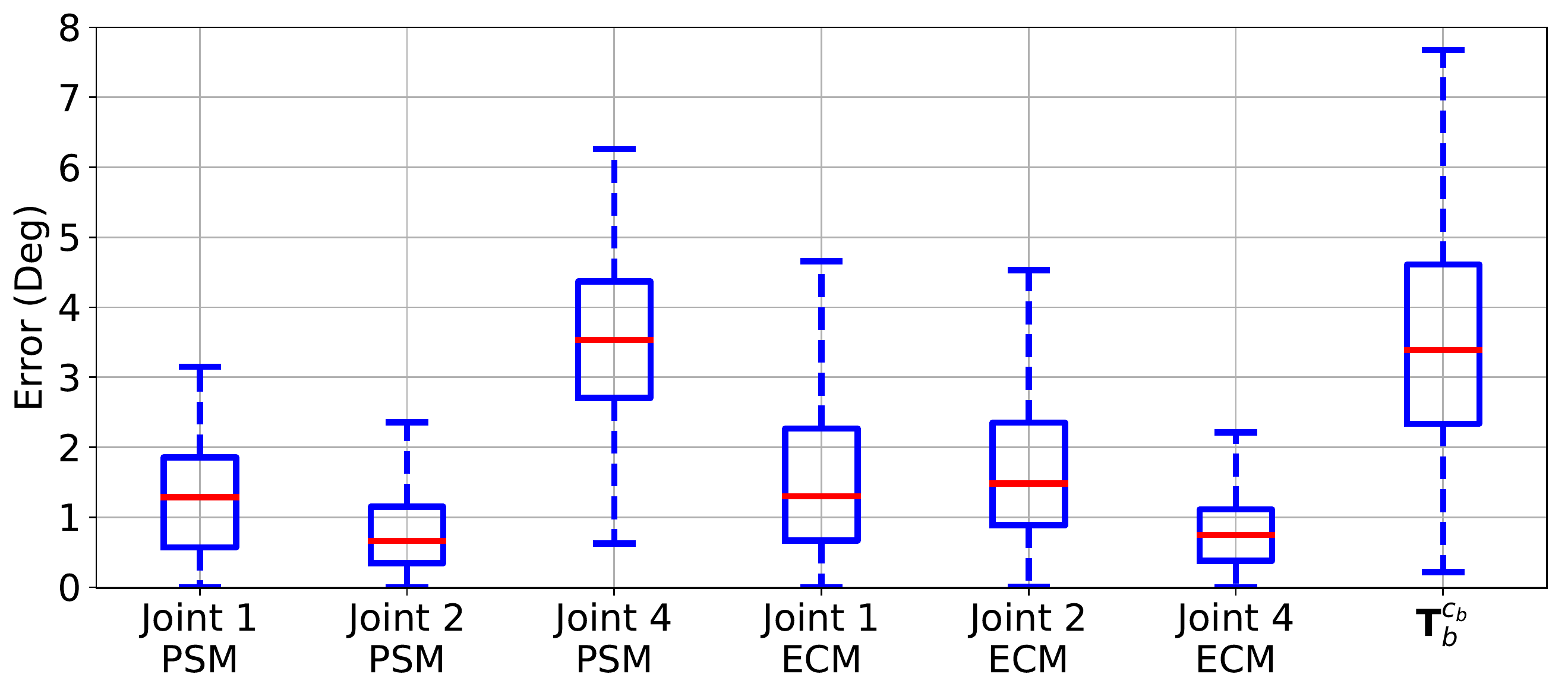}  
        \end{subfigure}
    \end{subfigure}
    \caption{
    Box plots of the tracked joint angle errors, whose preceding kinematic links are never in the camera frame, and the base of camera arm to base of robotic tool transform, $\mathbf{T}^{c_b}_b$, error when explicitly estimating all unknowns in the simulated da Vinci scene.
    Errors up to 14mm and 7 degrees highlight the inability to estimate these unknown values explicitly due to the parameters being non-identifiable as shown in Claim \ref{claim:inf_solutions}.
    }
    \label{fig:errors_when_tracking_everything_moving_ecm}
\end{figure*}

\begin{figure*}
    \centering
        \begin{subfigure}{0.45\textwidth}
        \includegraphics[width=\linewidth]{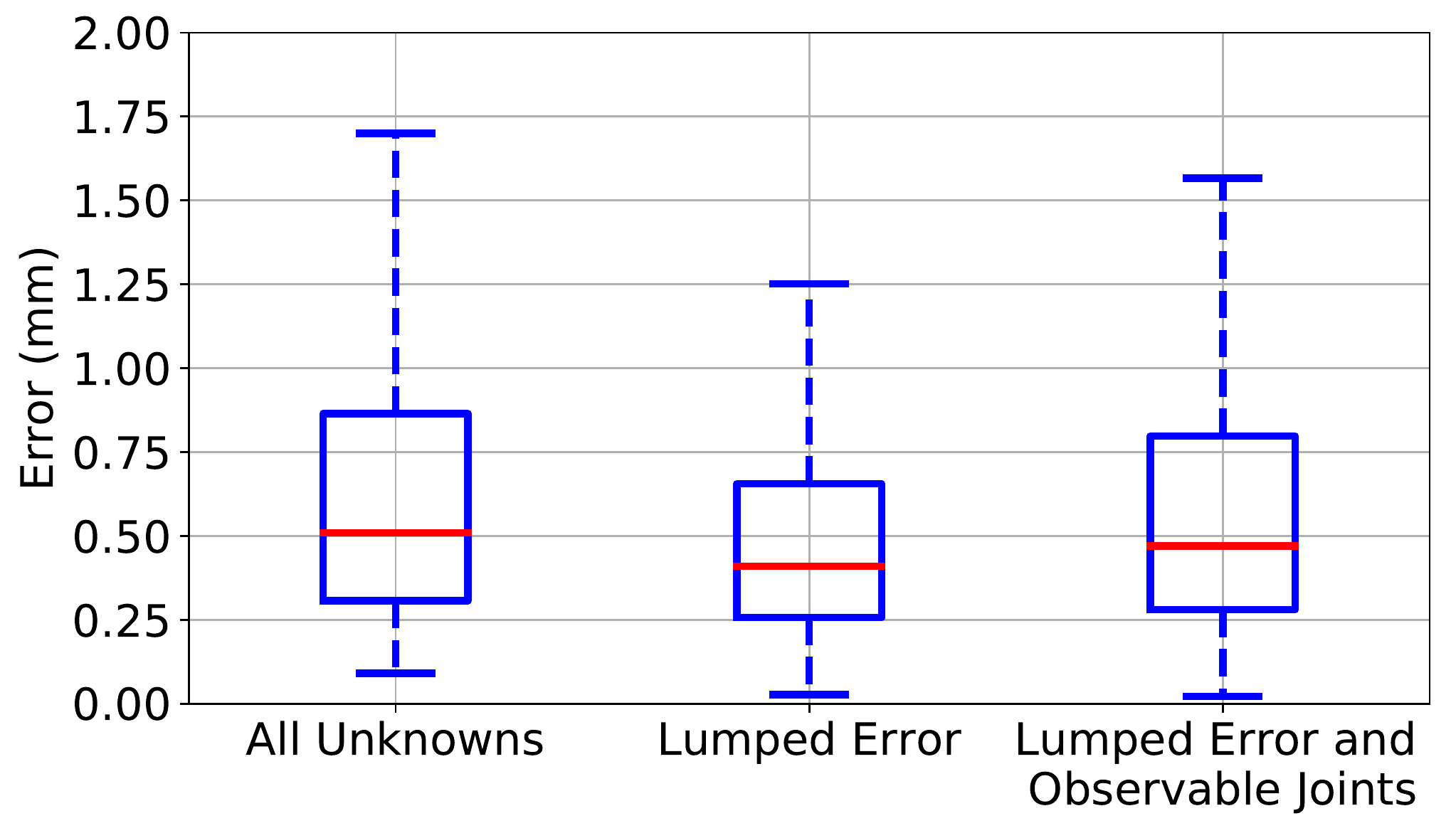} 
        \caption{End-effector position error}
    \end{subfigure}
    \hspace{10pt}
    \begin{subfigure}{0.45\textwidth}
        \includegraphics[width=\linewidth]{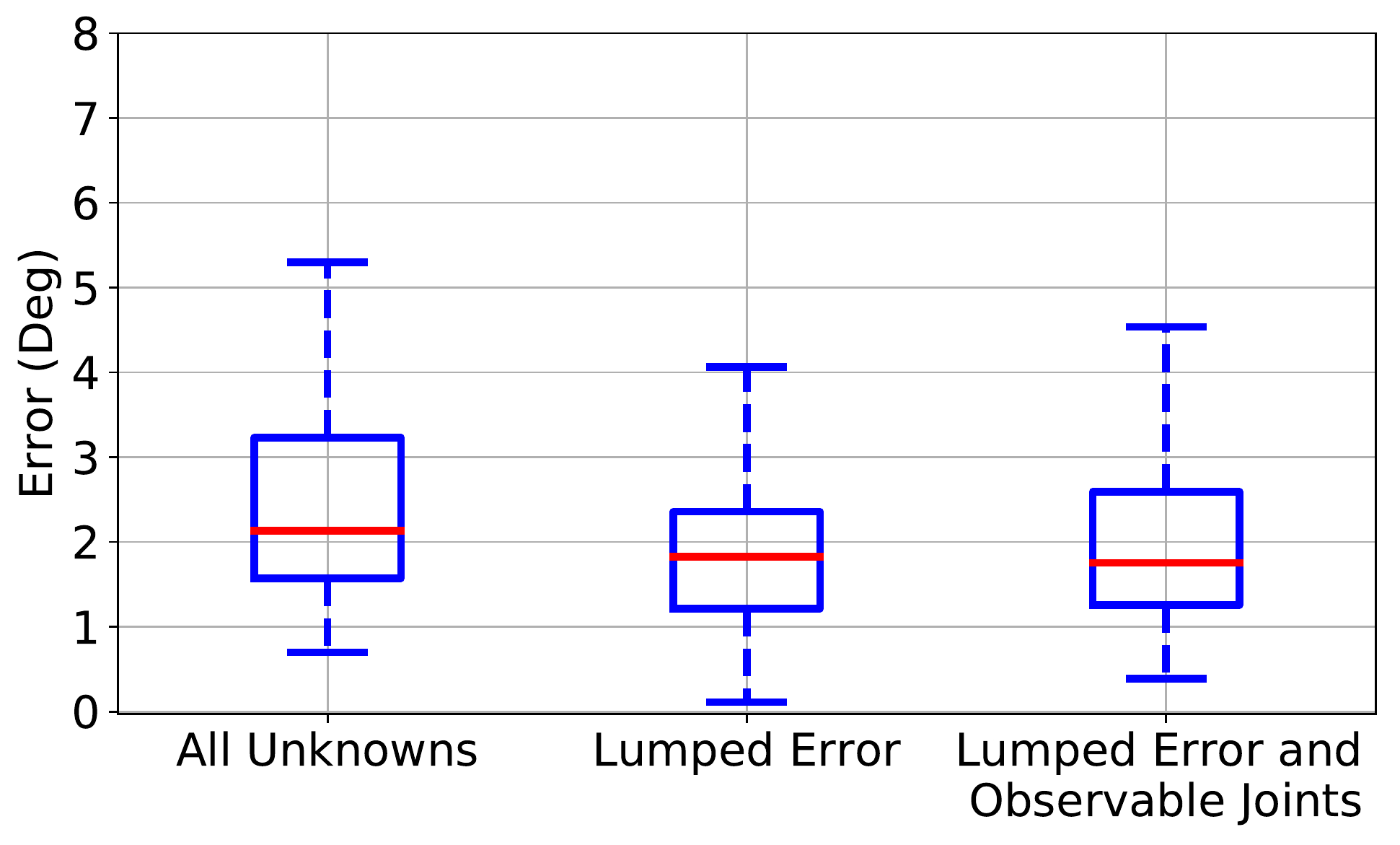} 
        \caption{End-effector orientation error}
    \end{subfigure}
    \begin{subfigure}{0.32\textwidth}
        \includegraphics[width=\linewidth]{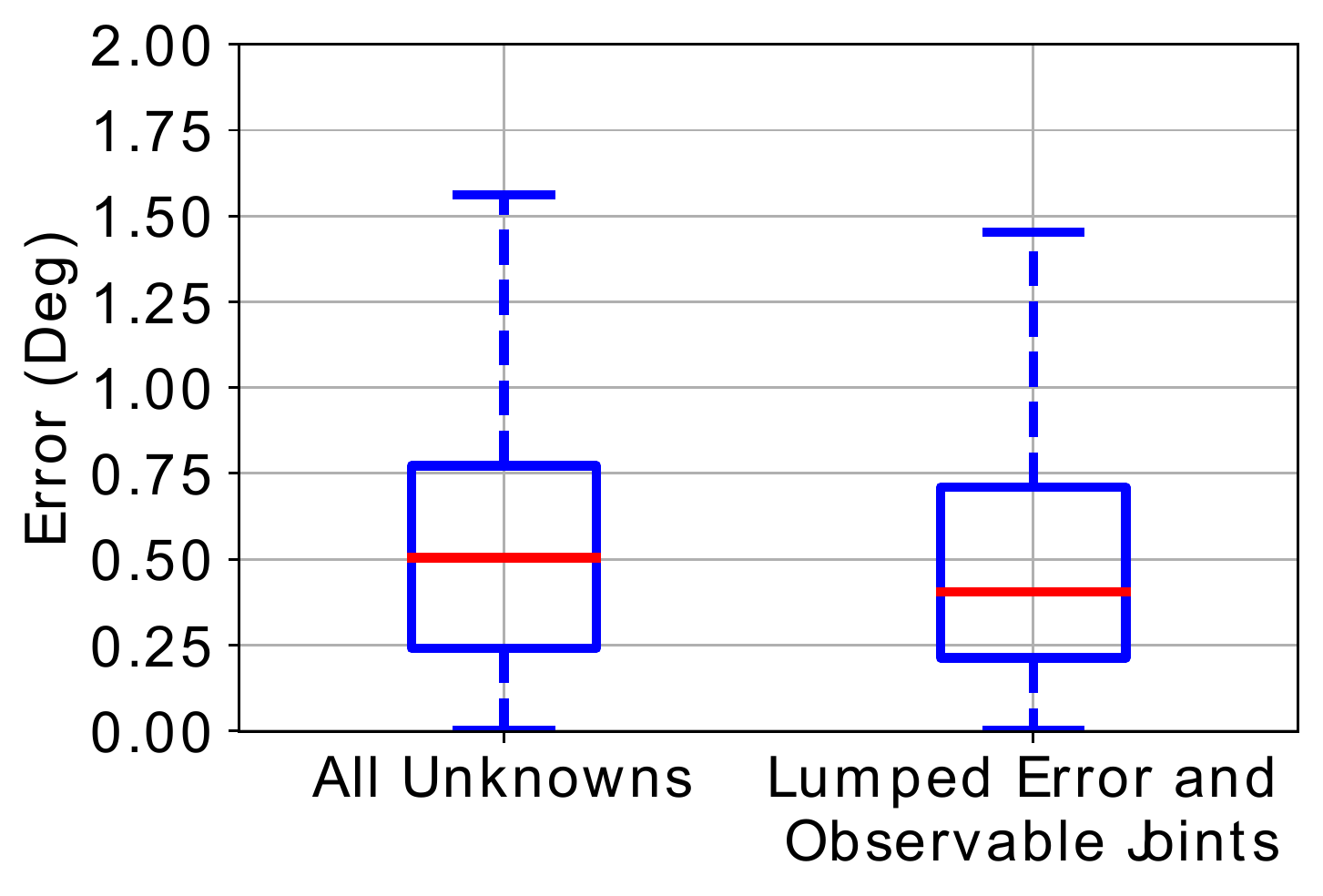} 
        \caption{Joint 5 error}
    \end{subfigure}
    \begin{subfigure}{0.32\textwidth}
        \includegraphics[width=\linewidth]{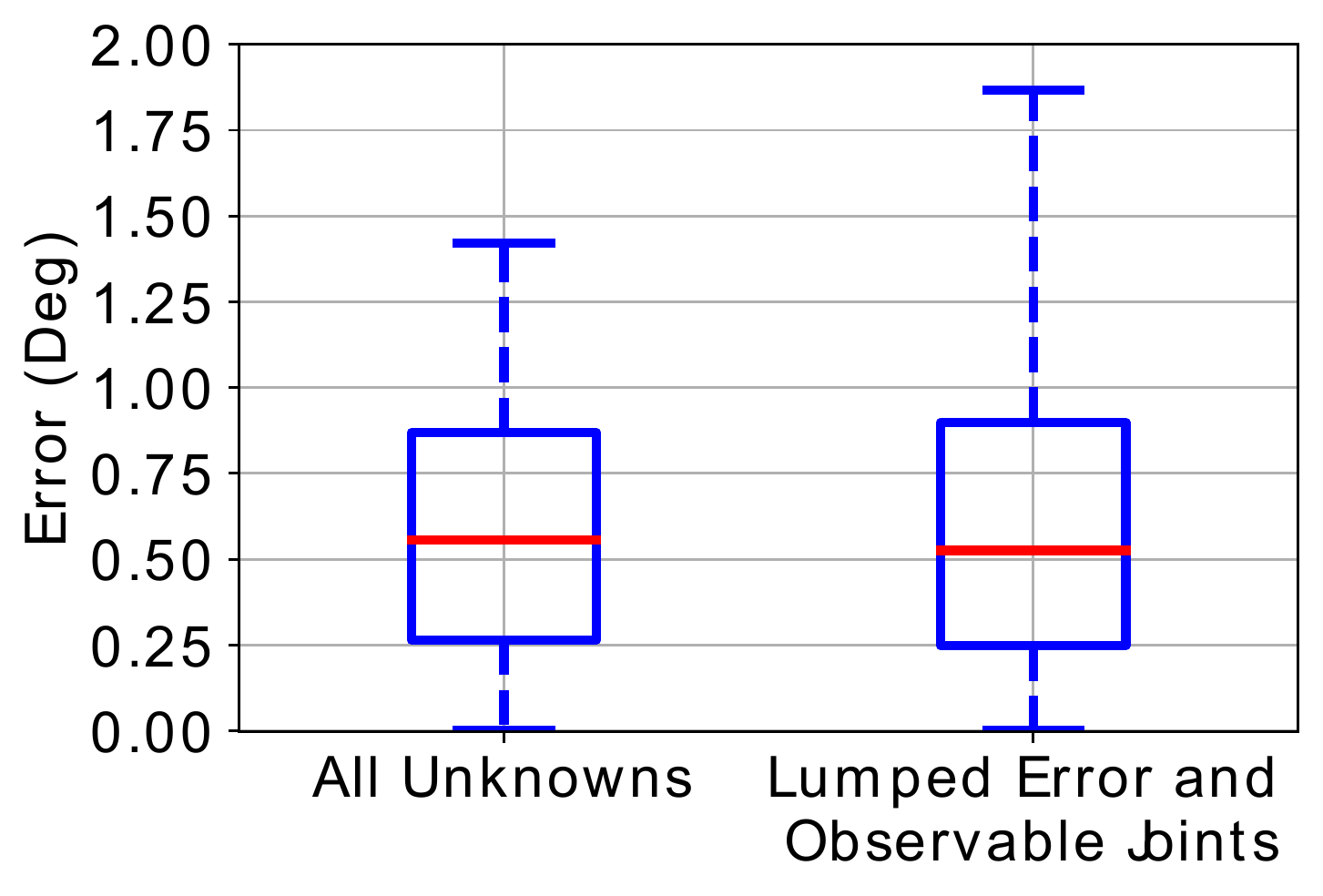}  
        \caption{Joint 6 error}
    \end{subfigure}
    \begin{subfigure}{0.32\textwidth}
        \includegraphics[width=\linewidth]{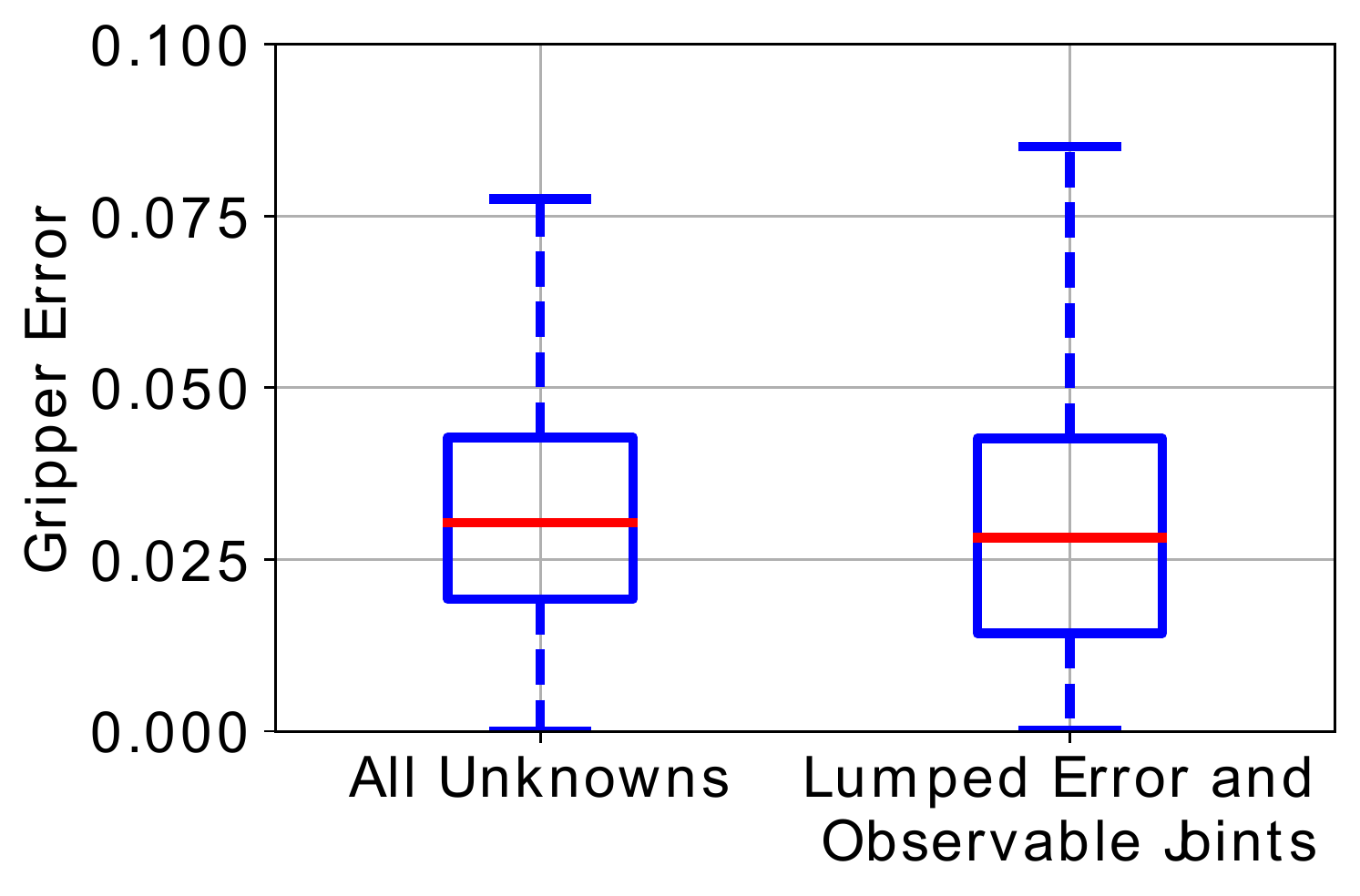} 
        \caption{Joint 7 error}
    \end{subfigure}
    \caption{Box plots of converged tracking performance under various configurations for the particle filter in simulated da Vinci  for stationary camera.
    As is evident from the box plots, the Lumped Error results in better converged end-effector pose error compared to tracking all unknowns.
    Meanwhile, no significant difference in observable joint angle error is seen.
    }
    \label{fig:stationary_converged_pf_box_plots}
\end{figure*}
\begin{figure*}
    \vspace{2mm}
    \centering
    \begin{subfigure}{0.45\textwidth}
        \includegraphics[width=\linewidth]{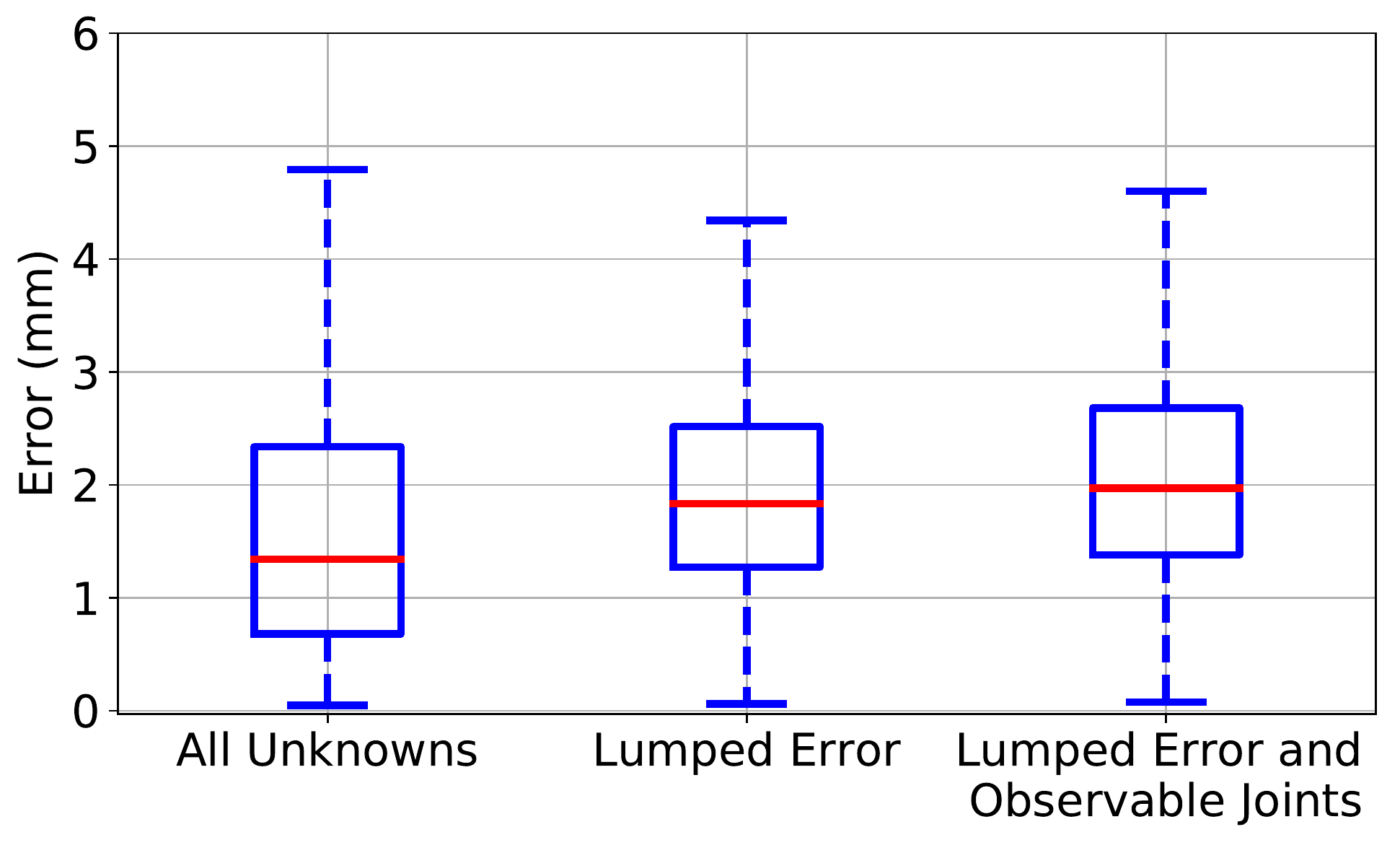}
        \caption{End-effector position error}
    \end{subfigure}
    \hspace{10pt}
    \begin{subfigure}{0.45\textwidth}
        \includegraphics[width=\linewidth]{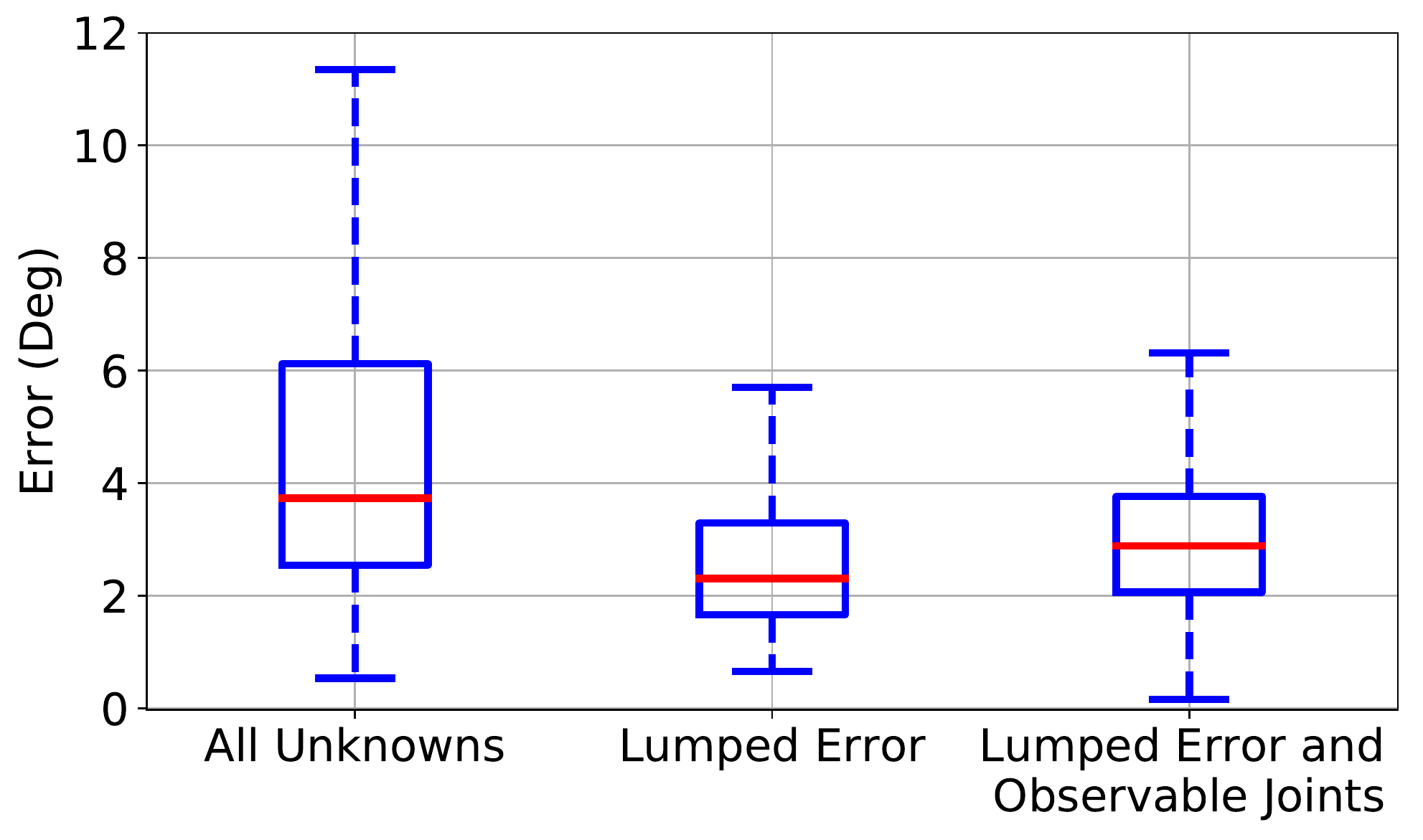} 
        \caption{End-effector orientation error}
    \end{subfigure}
    \begin{subfigure}{0.32\textwidth}
        \includegraphics[width=\linewidth]{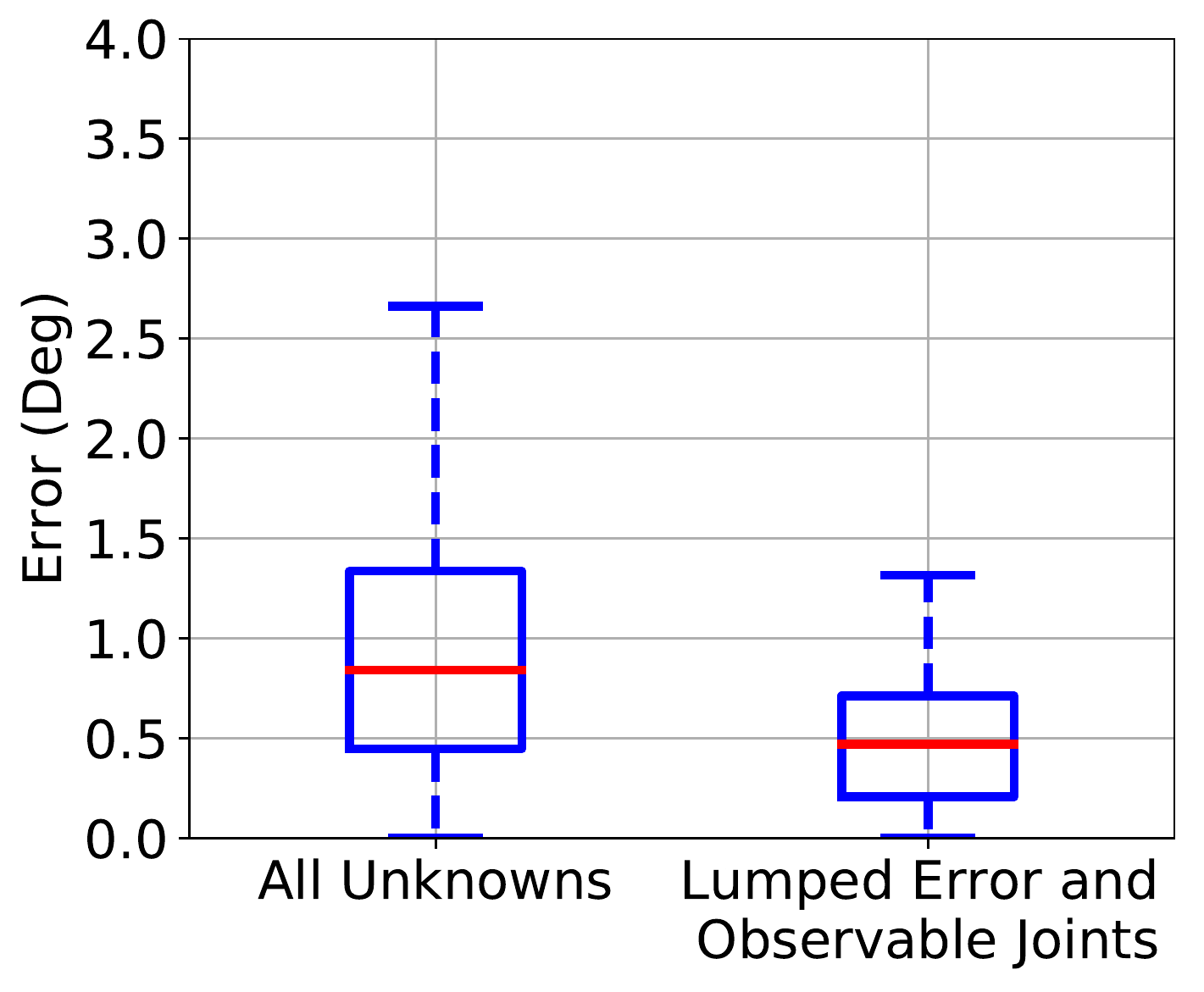} 
        \caption{Joint 5 error}
    \end{subfigure}
    \begin{subfigure}{0.32\textwidth}
        \includegraphics[width=\linewidth]{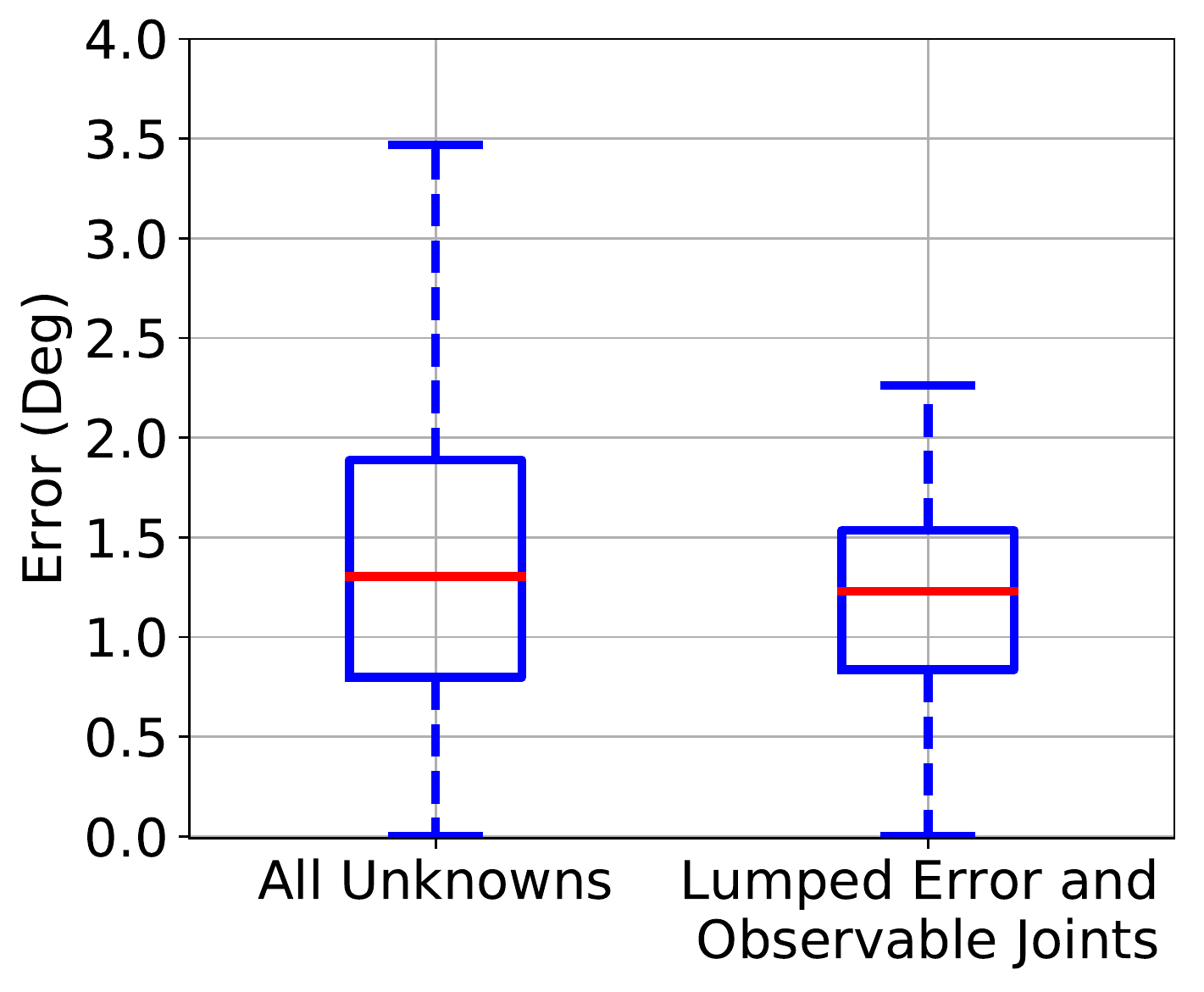}  
        \caption{Joint 6 error}
    \end{subfigure}
    \begin{subfigure}{0.32\textwidth}
        \includegraphics[width=\linewidth]{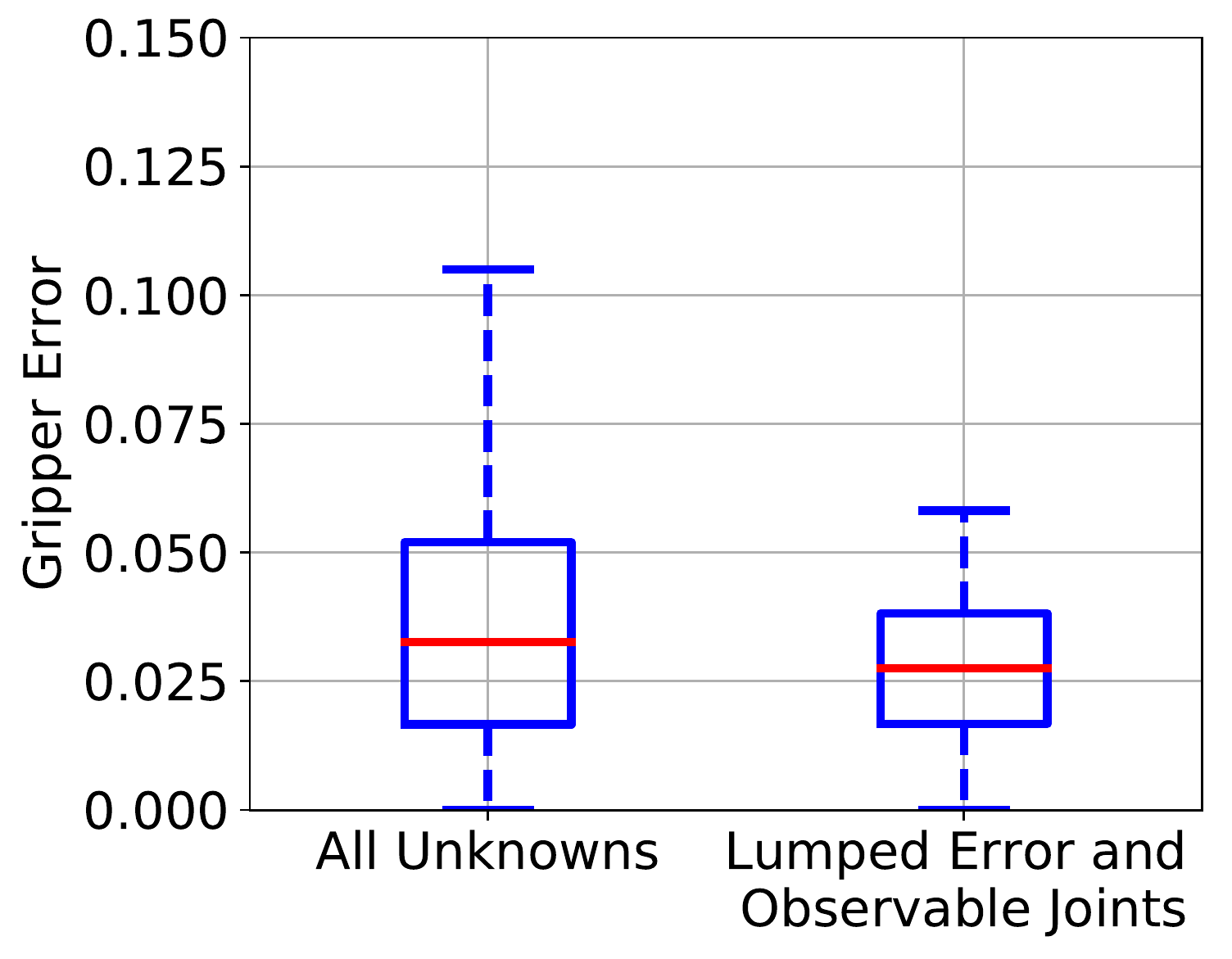}  
        \caption{Joint 7 error}
    \end{subfigure}
    \caption{Box plots of converged tracking performance under various configurations for the particle filter in simulated da Vinci scene for eye-in-hand configuration.
    As is evident from the box plots, the Lumped Error results in better converged end-effector orientation error and observable joint angle errors.  
    }
    \label{fig:moving_camera_converged_pf_box_plots}
\end{figure*}

\subsection{Tracking Lumped Error in da Vinci Simulated Scene}
The particle filter configurations evaluated were:
\begin{enumerate}
    \item \textit{All Unknowns}: tracking all joint angle errors and base-to-camera transform or base-to-base in the stationary and eye-in-hand case respectively. Done by setting $n_b = 0$ in the particle filter.
    \item \textit{Lumped Error}: applying (\ref{eq:simplification}) to the particle filter
    \item \textit{Lumped Error and Observable Joints}: no modifications to the described particle filter
\end{enumerate}
Both stationary camera and moving camera arm scenarios were tested.
Each configuration was repeated 50 times to test for consistent performance.

To evaluate the effectiveness of pose or transform estimation, the error was calculated at time $t$ as:
\begin{align}
   \epsilon_{\mathbf{b}} = || \mathbf{b}_t - \hat{\mathbf{b}}_t || && \epsilon_{\mathbf{w}} =|| \mathbf{w}^r_t||
\end{align}
where $\mathbf{w}^r_t$ is the axis angle representation of $\mathbf{R}_t (\hat{\mathbf{R}}_t)^{-1}$, $\mathbf{b}_t \in \mathbb{R}^3$ and $\mathbf{R}_t \in SO(3)$ are the ground truth translation vector and rotation matrix respectively, and $\hat{\mathbf{b}}_t \in \mathbb{R}^3$ and $\hat{\mathbf{R}}_t \in SO(3)$ are the tracked translation vector and rotation matrix respectively.
The $i$th joint angle error was computed as:
\begin{equation}
   \epsilon_{q^i} = | \hat{q}^i_t - q^i_t |j
\end{equation}

The mean end-effector pose error plots are shown in Fig. \ref{fig:pf_over_time_plots} for both stationary camera and eye-in-hand cases.
Fig. \ref{fig:errors_when_tracking_everything_stationary} shows the distributions of errors for the non-identifiable joint angles and base-to-camera transform in the stationary camera case when explicitly tracking all unknown parameters.
Fig. \ref{fig:errors_when_tracking_everything_moving_ecm} shows the distributions of errors in the eye-in-hand configuration for the non-identifiable joint angles and base-to-base transform when explicitly tracking all unknown parameters. 
These values were calculated across 40 time steps from all 50 trials after 100 time steps to give time for the particle filter to converge.
These errors have a large spread even though the particle filter is still able to sufficiently track the end-effector as seen in the mean end-effector pose error plots in Fig. \ref{fig:pf_over_time_plots}.
This supports Claim \ref{claim:inf_solutions} by showing that it is infeasible to explicitly track all the unknown parameters from partially visible robotic tools since they do not converge to their true values.

The distribution of end-effector tracking errors after giving the particle filter time to converge in the same manner as previously described are shown in Fig. \ref{fig:stationary_converged_pf_box_plots} and \ref{fig:moving_camera_converged_pf_box_plots} for stationary camera and moving camera arm respectively.
The converged distributions of error show little difference in end-effector positional error.
However, end-effector orientation error was clearly improved by using the Lumped Error estimation for both the stationary camera and robotic camera arm cases.
For the observable joints, joints $5, 6, 7$, the Lumped Error tracking method showed no significant difference in error between the stationary camera and eye-in-hand cases.
Meanwhile, when explicitly tracking all unknowns, the error in observable joints was significantly worse in the eye-in-hand case.

\subsection{Tracking Lumped Error on dVRK}

Two one-minute segments of encoder readings and stereoscopic data from the endoscope on an ECM was captured from a dVRK \cite{DVRK}.
The stereoscopic camera system used the standard dVRK endoscopic lens and has a resolution of 1920 by 1080 pixels at 30FPS. 
In both sequences, a single PSM arm was teloperated with gripper in full view of the stereoscopic camera.
In the first sequence, the PSM arm travelled a total distance of 48mm, and the ECM was stationary.
In the second sequence, the PSM arm travelled a total distance of 49mm, and the ECM arm joint angles were set to sinusoidal patterns similar to the previous simulation experiment resulting in 35mm for a total distance travelled.
The PSM arm had blue colored markers painted on.
The markers and edges of the projected cylindrical insertion shaft were detected in the same manner as the simulated scene.
The initial calibration, $\mathbf{T}^c_{b-}$ and $\mathbf{T}^{c_b}_{b-}$, were computed using OpenCV's solvePnP \cite{opencv_library} with manually set associations of the markers.
On a deployed, and fully assembled da Vinci\textregistered{} Surgical System, we envision that these initial transformations are computed from the set up joints, which connect the ECM with the PSM arm.
However, the dVRK by default does not come with set up joints, which is why we elected to the sovlePnP with manually set associations for initialization.

From both sequences, 20 evenly distributed images were manually annotated using the VGG labeller \cite{dutta2016via}.
These labels, $\mathbf{I}_G$, were considered ground truth and Intersection over Union (IoU) was used as the metric in this experiment:
\begin{equation}
    \frac{\mathbf{I}_R \cap \mathbf{I}_G}{\mathbf{I}_R \cup \mathbf{I}_G}
\end{equation}
where $\mathbf{I}_R$ was a generated mask using our previously developed rendering procedure \cite{sarpd}.
The generated mask was rendered using the tracked parameters.
The distributions of IoU are shown in Fig. \ref{fig:IoU_results_dvrk} for both stationary and moving camera arm cases.
Similar to the simulation results, the Lumped Error clearly performed the best.
Examples of surgical tool renderings on top of the image feed are shown in Fig. \ref{fig:images_of_best_and_worst_iou_dvrk}. 

\begin{figure*}
  \centering
  \vspace{2mm}
    \begin{subfigure}{0.45\textwidth}
        \includegraphics[width=\linewidth]{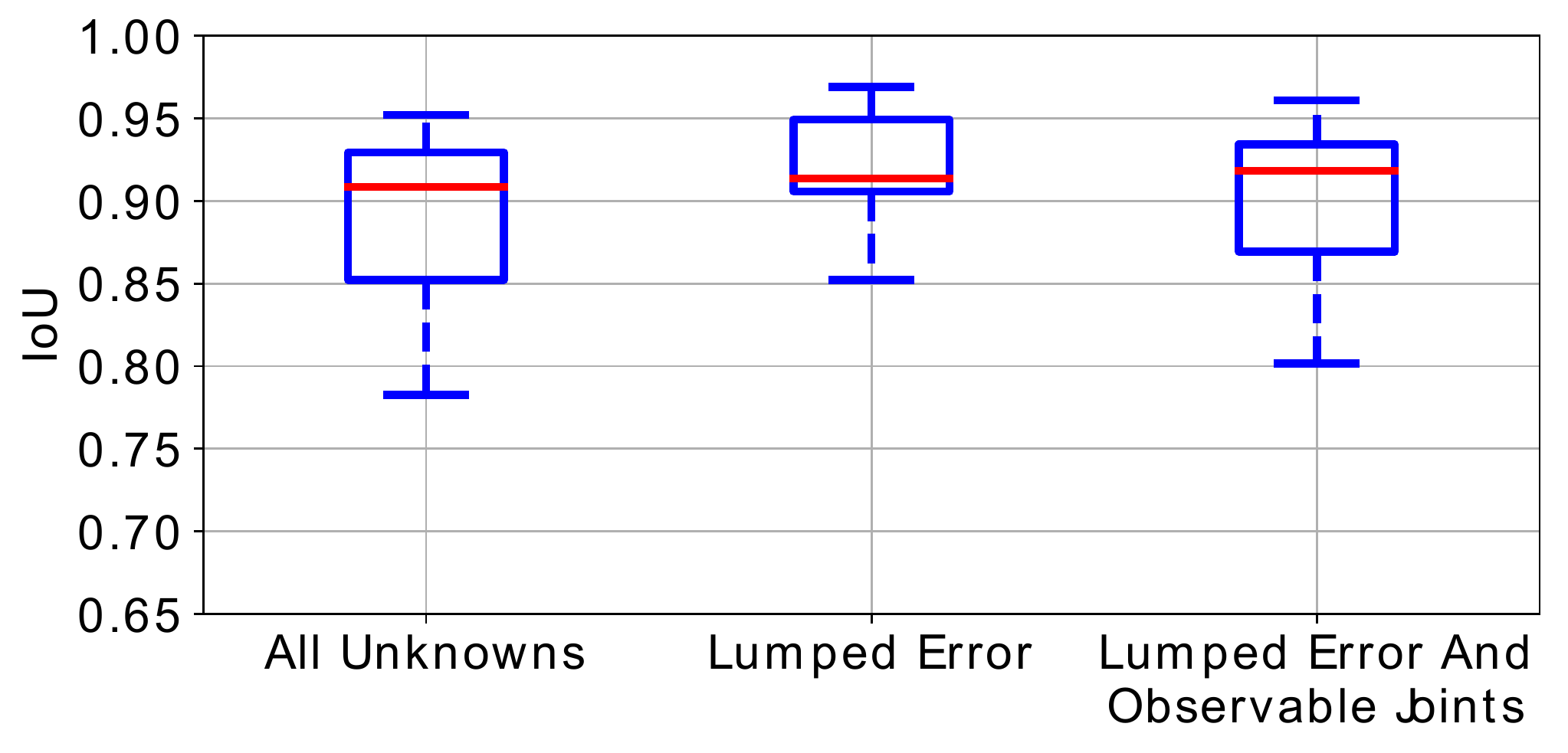}
    \end{subfigure}
    \hspace{10pt}
    \begin{subfigure}{0.45\textwidth}
        \includegraphics[width=\linewidth]{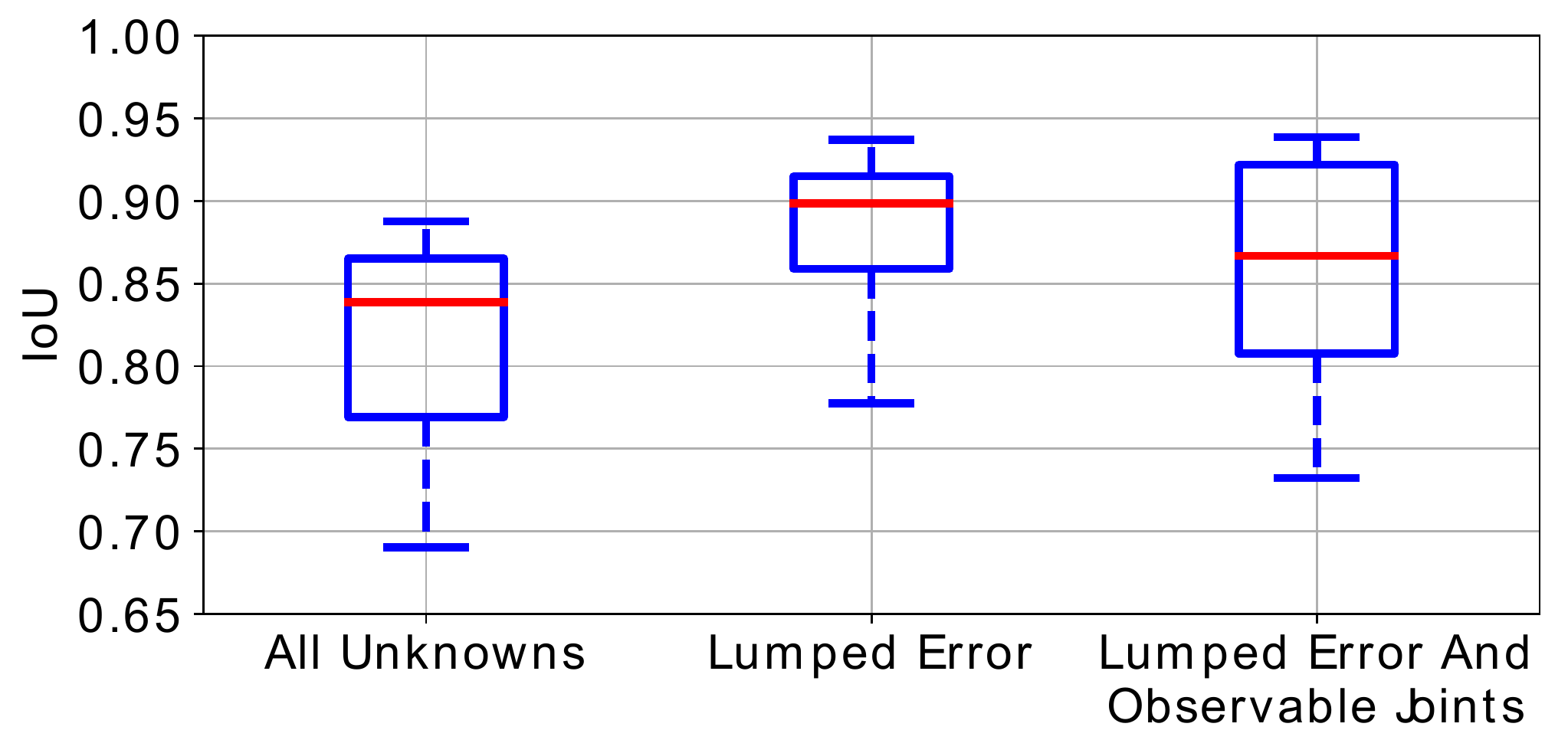}
    \end{subfigure}
    \caption{Distribution of Intersection over Union (IoU) between manual annotations and re-projected rendering from tracked values under various particle filter configurations on the dVRK \cite{DVRK}. The left and right plots correspond to stationary camera and eye-in-hand cases respectively. The plots show tracking the Lumped Error yields better tool tracking rather than explicitly tracking all unknowns.}
    \label{fig:IoU_results_dvrk}
\end{figure*}

\begin{figure*}
    \centering
    \begin{subfigure}{0.49\linewidth}
    \centering
    \begin{subfigure}{0.325\textwidth}
        \includegraphics[width=\linewidth]{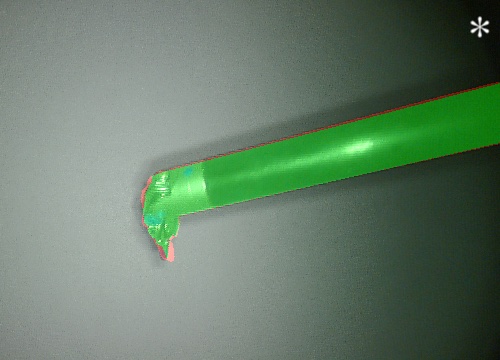}
    \end{subfigure}
    \begin{subfigure}{0.325\textwidth}
        \includegraphics[width=\linewidth]{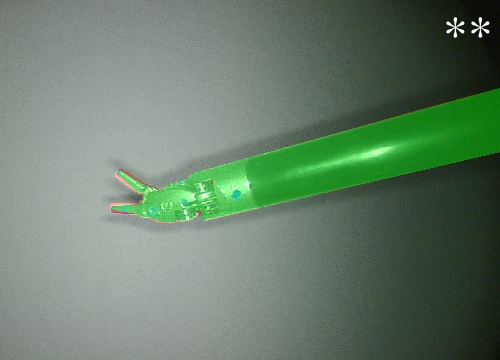}
    \end{subfigure}
    \begin{subfigure}{0.325\textwidth}
        \includegraphics[width=\linewidth]{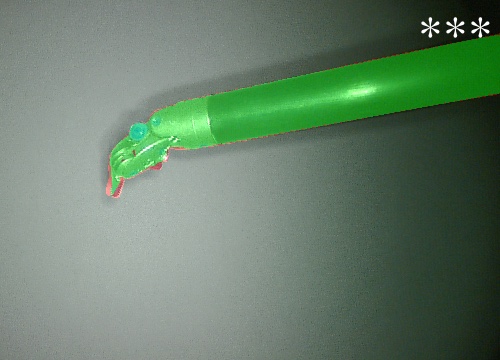}
    \end{subfigure}
    \caption{Best IoU for stationary camera: 0.94$^{*}$, 0.96$^{**}$, 0.96$^{***}$}
    \vspace{5pt}
    \end{subfigure}
    \begin{subfigure}{0.49\linewidth}
    \centering
    \begin{subfigure}{0.325\textwidth}
        \includegraphics[width=\linewidth]{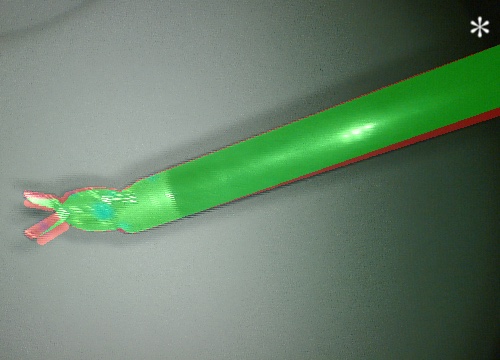}
    \end{subfigure}
    \begin{subfigure}{0.325\textwidth}
        \includegraphics[width=\linewidth]{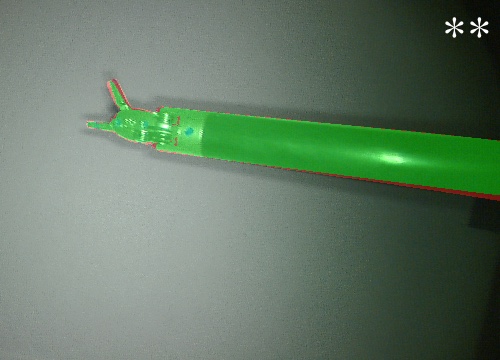}
    \end{subfigure}
    \begin{subfigure}{0.325\textwidth}
        \includegraphics[width=\linewidth]{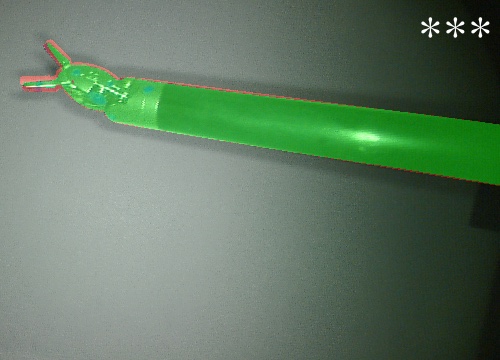}
    \end{subfigure}
    \caption{Best IoU for moving camera arm: 0.89$^{*}$, 0.94$^{**}$, 0.94$^{***}$}
    \vspace{5pt}
    \end{subfigure}
    \begin{subfigure}{0.49\linewidth}
    \centering
    \begin{subfigure}{0.325\textwidth}
        \includegraphics[width=\linewidth]{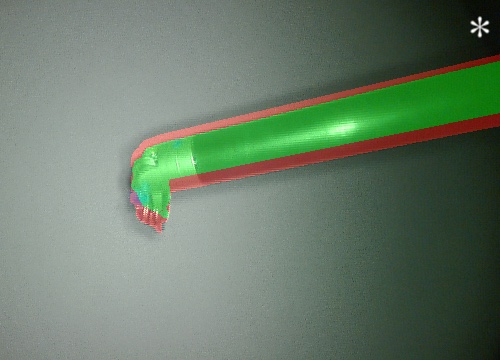}
    \end{subfigure}
    \begin{subfigure}{0.325\textwidth}
        \includegraphics[width=\linewidth]{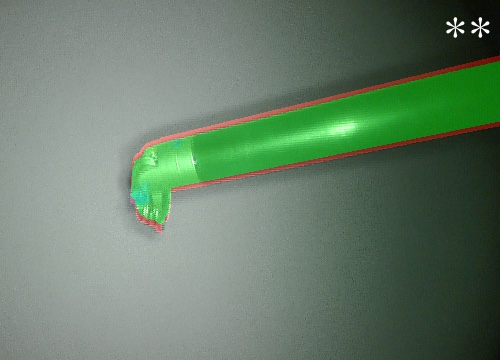}
    \end{subfigure}
    \begin{subfigure}{0.325\textwidth}
        \includegraphics[width=\linewidth]{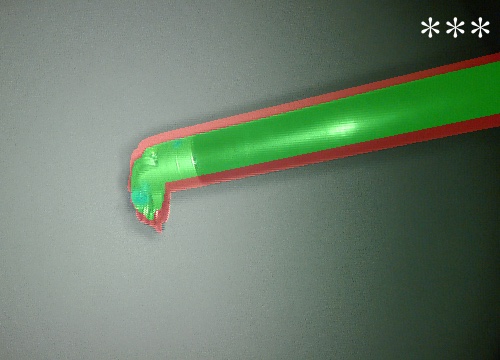}
    \end{subfigure}
    \caption{Worst IoU for stationary camera: 0.69$^{*}$, 0.85$^{**}$, 0.70$^{***}$}
    \end{subfigure}
    \begin{subfigure}{0.49\linewidth}
    \centering
    \begin{subfigure}{0.325\textwidth}
        \includegraphics[width=\linewidth]{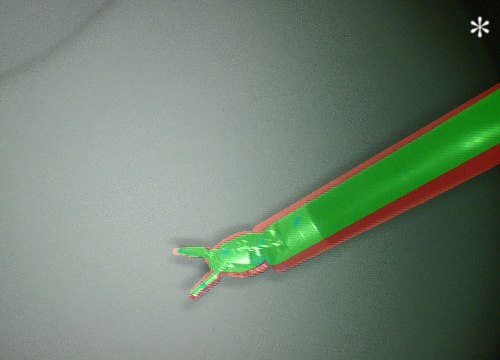}
    \end{subfigure}
    \begin{subfigure}{0.325\textwidth}
        \includegraphics[width=\linewidth]{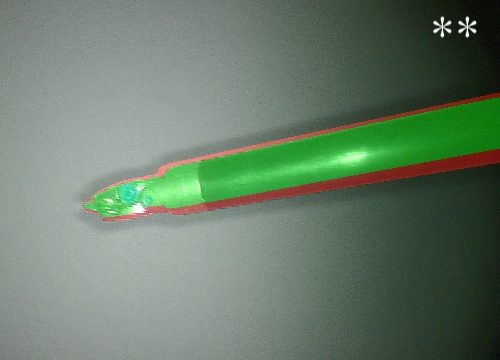}
    \end{subfigure}
    \begin{subfigure}{0.325\textwidth}
        \includegraphics[width=\linewidth]{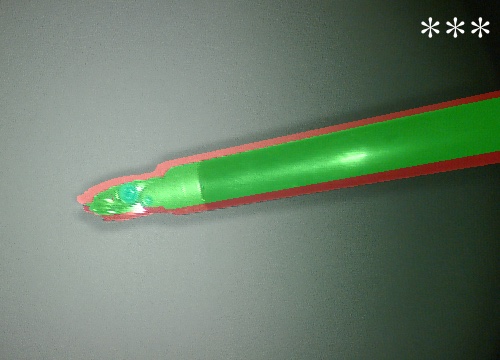}
    \end{subfigure}
    \caption{Worst IoU for moving camera arm: 0.69$^*$, 0.76$^{**}$, 0.73$^{***}$}
    \end{subfigure}
    \caption{Images from the best and worst Intersection over Union (IoU) when re-projecting the tracked dVRK \cite{DVRK} surgical tool onto the image. The green and red regions are the intersection and error, union minus intersection, respectively between the re-projection and surgical tool (best seen in color). The tracking conditions for all sets of three, from left to right, are: All Unknowns$^{*}$, Lumped Error$^{**}$, Lumped Error and Observable Joints$^{***}$. The corresponding IoU values are also listed. This shows that the failures when using Lumped Error are substantially less severe than tracking all unknowns.}
    \label{fig:images_of_best_and_worst_iou_dvrk}
\end{figure*}


\begin{figure*}[t]
    \centering
    \vspace{2mm}
    \includegraphics[trim=0cm 3cm 0cm 0cm, clip,width=0.43\textwidth]{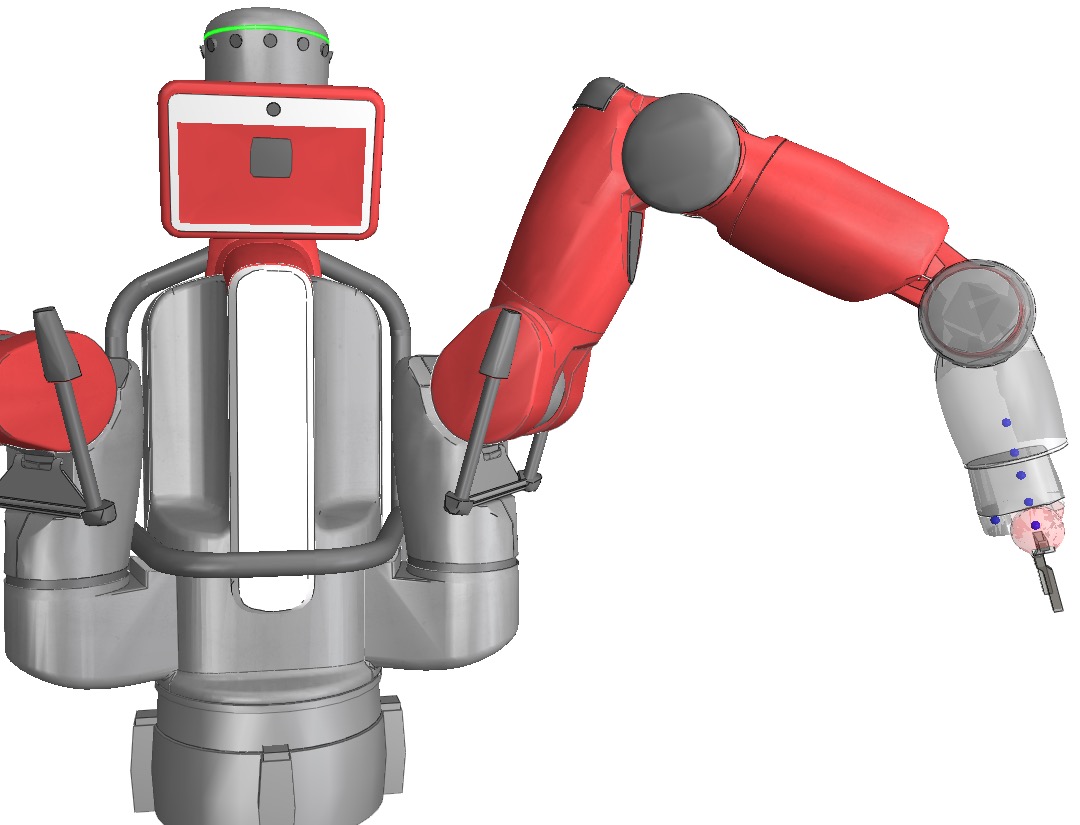}\hspace{5mm}
    \includegraphics[trim=0cm 7.5cm 10cm 3.5cm, clip,width=0.43\textwidth]{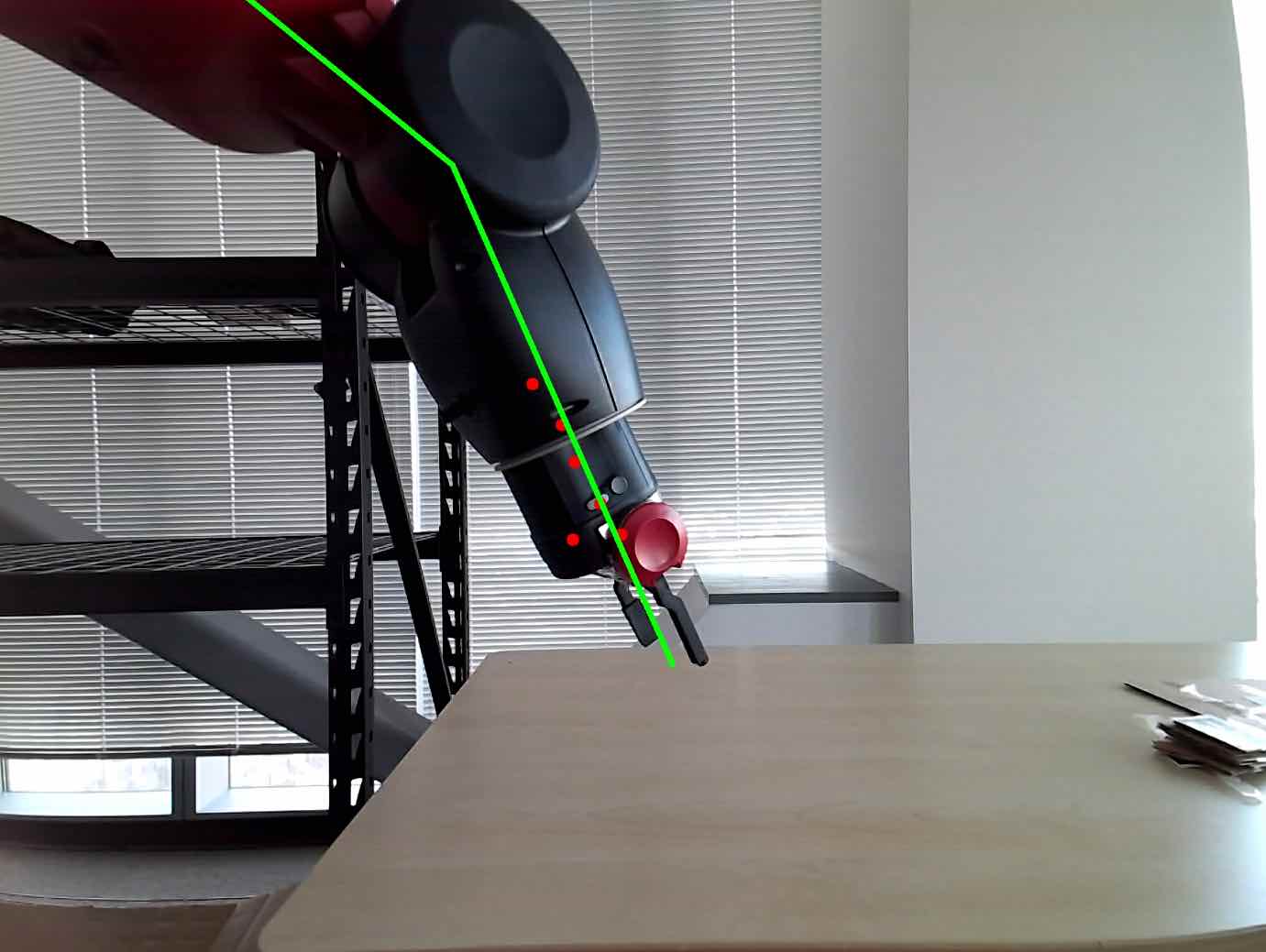}
    \caption{
    A DNN was trained in simulation to detect keypoints on a partially visible Baxter robot arm.
    The keypoint locations on the visible portion of the Baxter arm in this experiment were optimized for such that their detectability and accuracy is maximized \cite{lu2020robust}, and the result is shown with blue circles in the left figure.
    An example of the detections from the Baxter experiment is shown in red on the right figure.
    These detections were used to update the particle filter tracking the Lumped Error of the partially visible Baxter arm.
    A re-projected skeleton of the Baxter using the particle filter is also shown in green on the right figure.
    }
    \label{fig:baxter_keypoint_optimization}
\end{figure*}

\begin{figure*}
    \centering
    \includegraphics[trim=0cm 0cm 0cm 0cm, clip,width=0.48\textwidth]{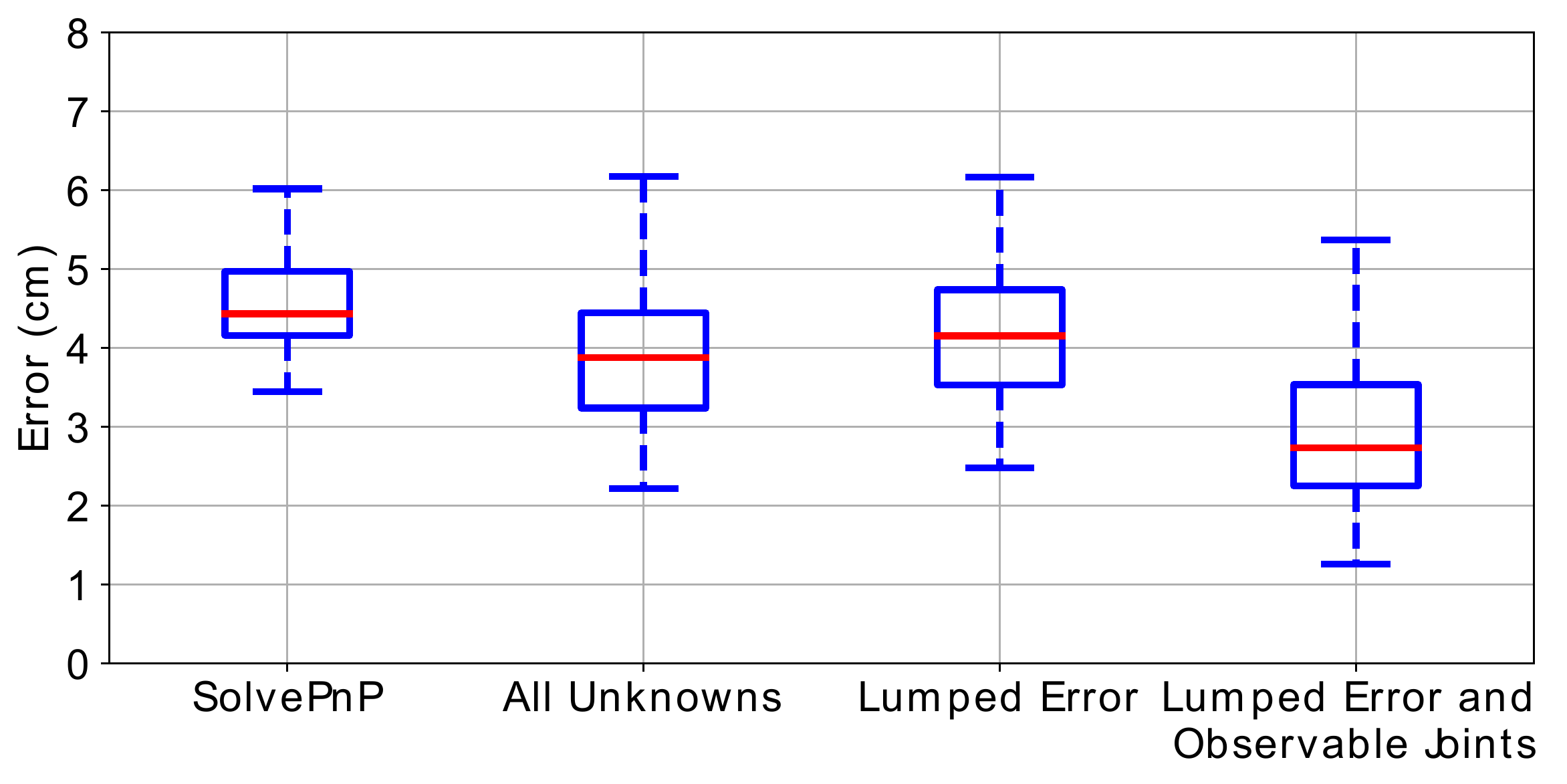}
\hspace{10pt}
    \includegraphics[trim=0cm 0cm 0cm 0cm, clip,width=0.48\textwidth]{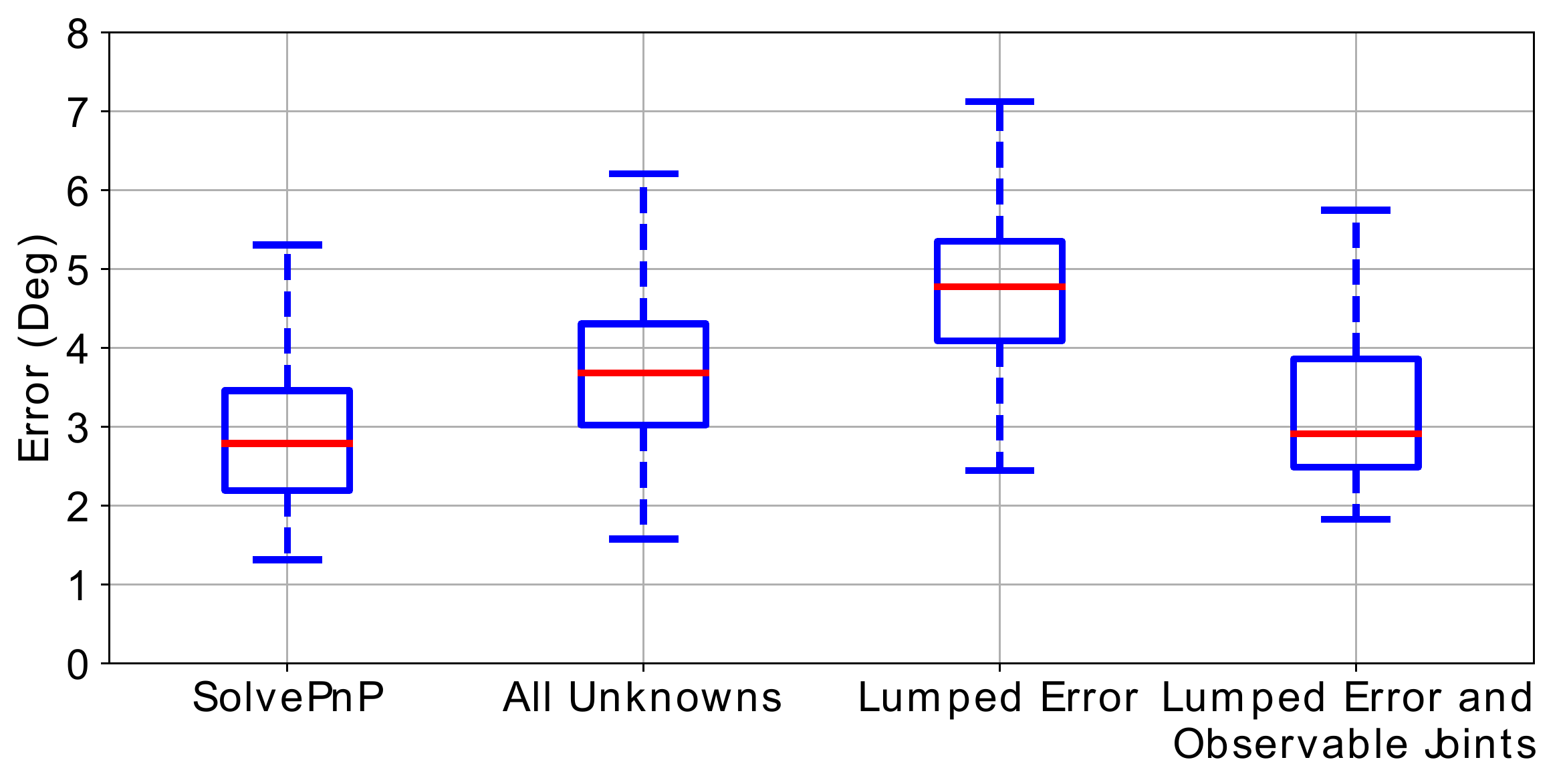}
    \caption{Distribution of position and orientation errors when calibrating for the base-to-camera transform alone with solvePnP and various particle filter configurations compensating for errors in the base-to-camera transform and joint angles from the Baxter robot experiment. Out of the active tracking methods, the Lumped Error parameter reduction technique with the observable joints is the most effective.
    Meanwhile, solvePnP for the static base-to-camera transform performs similar to Lumped Error and Observable joints in orientation error, but performs significantly worse in positional error.
    }
    \label{fig:baxter_box_plot_results}
\end{figure*}

\subsection{Tracking Lumped Error on Baxter}

A 77 second video segment was recorded of a 7 DoF arm from the Baxter robot with corresponding joint reading data.
The first joint link consistently visible in the image frames is after the $n_b = 6$ joint, and the end-effector moved a total distance of 5.25m during the segment.
The video was captured on Microsoft's Azure Kinect camera which has and RGB camera and a depth camera.
In this experiment, the particle filter which tracked this robotic arm only used the mono-RGB camera data.
Meanwhile the depth images were used to evaluate performance of tracking the robotic tool.

The features used to update the particle filter for this experiment were detected using a Deep Neural Network (DNN) rather than markers as used in the previous experiments.
This was done to show the flexibility of the presented particle filter with regards to the features used to update it.
The DNN choosen was DeepLabCut \cite{mathis2018deeplabcut}, and it was trained to detect and optimize feature points in simulation using our previously developed method \cite{lu2020robust}.
The resulting feature points which were detected by the DNN are shown in Fig. \ref{fig:baxter_keypoint_optimization}.
The DeepLabCut detections also provide direct associations for the feature points, $A_m$, and a confidence value, $\eta^k_t \in [0,1]$, for each detected feature $k$.
To integrate this with the Lumped Error tracking, the point feature observation model in (\ref{eq:pf_observation_model_markers}) was modified to:
\begin{equation}
    P( \mathbf{m}_{t}  | \hat{\mathbf{w}}_t, \hat{\mathbf{b}}_t, \hat{\mathbf{e}}_t ) \propto \sum \limits_{ k,i  \in A_m } \eta^k_t e^{-\gamma_m ||\mathbf{m}^k_{t} - \hat{\mathbf{m}}_i(\hat{\mathbf{w}}_t, \hat{\mathbf{b}}_t, \hat{\mathbf{e}}_t )||}
    \label{eq:pf_observation_model_markers_DNN}
\end{equation}
which removes the association step in line \ref{alg_pf:update_point_associate} from Algorithm \ref{alg:pf}.
It is important to include the DNN's confidence, $\eta^k_t$, in the model because sometimes the detections can be poor and the corresponding update needs to be weighted lower.
In the original observation model, (\ref{eq:pf_observation_model_markers}), this was done in the association step where the maximum cost is thresholded at $C^m_{max}$.

To evaluate tracking performance, the depth images were compared against the reconstructed robotic arm using the tracked parameters.
The reconstructed scene was rendered using a virtual camera and a Baxter robot model in V-REP \cite{rohmer2013v}.
After every image used to update the particle filter, the virtual camera captured a depth image of the reconstructed scene.
The tracking error was defined as the relative transform, $\mathbf{T}(\mathbf{w}^{\epsilon}, \mathbf{b}^{\epsilon}) \in SE(3)$, between the rendered point cloud from the virtual camera, $\mathbf{R}$, and the corresponding point cloud from Kinect Azure's depth camera, $\mathbf{G}$.
This relative transform is calculated by minimizing
\begin{equation}
    \sum \limits_{(\mathbf{r}, \mathbf{g}) \in \mathcal{K}} || \overline{\mathbf{r}} - \mathbf{T}(\mathbf{w}^{\epsilon}, \mathbf{b}^{\epsilon}) \overline{\mathbf{g}} ||
\end{equation}
where $\mathbf{r} \in \mathbf{R}$, $\mathbf{g} \in \mathbf{G}$, and $\mathcal{K}$ is the correspondence set between $\mathbf{R}$ and  $\mathbf{G}$.
The optimization was solved using Open3D's \cite{Zhou2018} implementation of the Iterative Closest Point algorithm \cite{besl1992method}.
To filter out poorly converged values, only the results where the amount of corresponded points relative to the total rendered points, $|\mathcal{K}| / |\mathbf{R}|$, is greater than 0.7 were recorded.
Similar to the dVRK experiments, the initial calibration, $\mathbf{T}^c_{b-}$, was computed using OpenCV's solvePnP \cite{opencv_library} using the detected features and their associations on the first image frame.

For additional comparison, we ran OpenCV's solvePnP implementation \cite{opencv_library} over the first 20 images of the dataset to solve for a static base-to-camera transform, $\mathbf{T}^c_b$.
The 2D detections are the same used by the particle filter.
Their corresponding 3D positions in the base frame of the Baxter robot are generated using forward kinematics with the joint angle readings $\tilde{q}^i_t$.
The resulting $\mathbf{T}^c_b$ and the joint angle readings are used to generate a rendered point cloud $\mathbf{R}$ in V-REP \cite{vrep_simulator}, and the same error metric as previously described is computed.

The distribution of translational errors, $||\mathbf{b}^\epsilon||$, and orientation errors, $||\mathbf{w}^\epsilon||$, for the three different particle filter configurations are plotted in Fig. \ref{fig:baxter_box_plot_results}.
In this case, Lumped Error with observable joints performed the best.
We believe this is since the kinematic links on the Baxter are much larger hence making the simplification from (\ref{eq:simplification}) no longer valid.

\section{Discussion}

From the experimental results and their respective metrics, it is evident that using Lumped Error yields almost always better tool tracking than explicitly estimating all unknowns for both stationary camera and eye-in-hand cases.
Only two experimental metrics, Fig. \ref{fig:moving_camera_converged_pf_box_plots}a and \ref{fig:baxter_box_plot_results}b, performed marginally better when tracking all unknowns.
Nevertheless, these metrics cannot be viewed in isolation, and their respectively paired metrics, Fig. \ref{fig:moving_camera_converged_pf_box_plots}b-e and \ref{fig:baxter_box_plot_results}a, showed significant improvement when tracking the Lumped Error over all unknowns.
Furthermore, the non-identifiable values estimated when tracking all unknowns yielded non-realistic results as shown in Fig. \ref{fig:errors_when_tracking_everything_stationary} and \ref{fig:errors_when_tracking_everything_moving_ecm}.
This is due to there being a single solution to the parameters being estimated when tracking the Lumped Error rather than the infinite set of solutions when tracking all unknowns, as shown in Claim \ref{claim:inf_solutions}.
Due to the infinite set of solutions, large distributions of the parameters being estimated occurred which is seen in Fig. \ref{fig:errors_when_tracking_everything_stationary} and \ref{fig:errors_when_tracking_everything_moving_ecm} when tracking all unknowns.
From the perspective of the particle filter, this is an inefficient usage of particles and detrimental to tracking because the particle filter estimates the posterior probability using a finite number of samples.

Tracking the observable joints when using the Lumped Error showed little difference in end-effector accuracy to not tracking them in the da Vinci simulation.
This supports the validity of applying simplification in (\ref{eq:simplification}) to the da Vinci robot because the link lengths for the observable joints, a dexterous robotic gripper, are short.
However, not tracking the observable joints on the real-world dVRK performed better.
We believe this occurred due to the features not being detected as consistently in the real world as in simulation, hence highlighting the usefulness of the simplification in (\ref{eq:simplification}).
Meanwhile for the Baxter Experiment, tracking the observable joint angles gave better performance than using the simplification.
In contrast to the da Vinci robot, the link lengths for the observable joints in the Baxter experiment are long hence making errors in their joint angles more catastrophic.
This matches with our motivation for the simplification in (\ref{eq:simplification}), and that it should only be used when the errors from the observable joints does not propagate through the kinematic chain drastically (e.g. a robotic gripper).

In addition to being efficient with respect to parameters to estimate, the Lumped Error modelled as a Weiner Process experimentally was found to compensate various distributions of errors.
In the stationary camera arm case in simulation, it captures both linear cable stretch and uniform bias error from joint angles.
For the eye-in-hand case in simulation, tracking the Lumped Error additionally compensated for Gaussian noise from the camera arm.
Meanwhile in the real world experiments, there are significant non-linear cable stretch effects on the PSM arm \cite{hwang2020efficiently} and backlash and hystresis on the Baxter robot and ECM arm, all of which the Lumped Error successfully compensated.

The proposed particle filter was also shown to be effective with a variety of visual features to update the Lumped Error.
The surgical robotic experiments used colored markers and edge detections of a cylindrical shaft as features which also needed to be associated.
The Baxter robot experiment and previously equivalent work, SuPer Deep \cite{lu2020super}, instead used a DNN to extract features with association.
The DNN feature extraction does perform better, as shown in our previous work \cite{lu2020super}, but the marker based approach is still sufficient for precise control as shown in the previous autonomous suction \cite{richter2020autonomous} and needle regrasping \cite{chiu2020bimanual} works.
This discrepancy in performance is due to the improved accuracy in feature detection, which directly improves the accuracy of the particle filters output.
Furthermore, we observed better tracking performance when the visible robotic tool is closer to the camera as the features (learned or markers) are more accurately detected.

The particle filter in all of the experiments shown here and the previous ones just described ran on a Intel\textregistered{} Core$^{\text{TM}}$ i9-7940X Processor and NVIDIA’s GeForce RTX 2080 and yielded a loop rate of 24FPS and 10FPS when using the colored markers and DNN respectively.
Therefore, it is suitable for real-time applications which is further highlighted by our previous control experiments.
These previous works in surgical robotic control and their usage of Lumped Error are discussed in Appendix \ref{appendix:previou_work}.

We also showed that the Lumped Error is mathematically equivalent to a previous, popular surgical tool tracking formulation which claimed to only track the error in the transform from the camera frame to the base of the surgical robot \cite{original_rcs}.
The equivalency implies that this previous formulation actually was compensating for both error in base-to-camera transform and joint angle errors.
The surgical tool tracking experiments presented here focused on the parameter reduction technique.
For performance of surgical tool tracking in surgical environments, refer to our previously developed work which utilized an equivalent tool tracking method \cite{super, lu2020super,richter2020autonomous, chiu2020bimanual}.

Lastly, we want to highlight that the tracking method presented in this work requires knowledge of the kinematic chain through the joint transforms $\mathbf{T}^i_{i-1}(\cdot)$ and camera intrinsics, $\mathbf{K}$.
This is a fair assumption as the kinematic chain for a robot is typically supplied by the manufacturer.
Furthermore, the joint transforms can be calibrated for with high accuracy offline \cite{okamura1996kinematic}.
Camera calibration is also a well studied method \cite{zhang2000flexible} and even implemented in standard vision toolboxes such as OpenCV \cite{opencv_library}.
Nonetheless, accurate calibration of these parameters is required to implement the proposed tracking method in this work.

\section{Conclusion}

In this work, we describe the challenges of tracking robotic tools from visual observations that only show part of the kinematic chain and there is uncertainty in base-to-camera transform and joint angle measurements through a problem formulation that shows it is infeasible to directly estimate all of these unknowns.
A smaller set of parameters, which we coined as Lumped Error, is derived and shown to compensate all the described uncertainties and is identifiable.
Furthermore, Lumped Error is mathematically equivalent to a popular instrument tracking method \cite{original_rcs} hence giving a deeper understanding of how it works.
Tracking the Lumped Error experimentally is shown to efficiently track robotic tools and even can be extended to eye-in-hand configurations.
Through this extension, we successfully track for the first time a surgical robotic tool with a moving endoscope, which contained a total of 10 DoF and a gripper joint.

The proposed tracking method to estimate the Lumped Error used a Weiner Process to model the uncertainty and experimentally was found to be efficient.
In future work, we intend to use the analytical derivation of the Lumped Error to more precisely describe the uncertainty.
In particular, we will use cable stretch models to describe the joint angle error so the uncertainty of the motion model for the Lumped Error appropriately propagates the transmission errors for cable driven robots such as dVRK \cite{DVRK}.
Additionally, we also intend to investigate controllers which utilize the Lumped Error parameter reduction in future work.
The controllers presented in the previous autonomous suction \cite{richter2020autonomous} and suture needle regrasping \cite{chiu2020bimanual} works show great promise to the capabilities of a Lumped Error controller.
These controllers will be investigated from a theoretical perspective where criteria for stability will be defined.

\section*{Acknowledgements}
This work is supported by University of California San Diego's Galvanizing Engineering in Medicine (GEM) grant, the Intuitive Surgical Technology Grant, the National Science Foundation (NSF) under grant numbers 1935329 and 2045803, and the US Army Telemedicine and Advanced Technology Research Center (TATRC) under the Robotic Battlefield Medical Support System project.
F. Richter is supported via the NSF Graduate Research Fellowships. 
The authors would like to thank Intuitive Surgical Inc. for instrument donations, and Simon DiMiao, Omid Maherari, Dale Bergman, and Anton Deguet for their support with the dVRK.

\appendix

\subsection{Previous Works using Lumped Error Tracking for Surgical Robotic Control}
\label{appendix:previou_work}

The Lumped Error simplification has been used to great effect in previous surgical robotic work under the guise of tracking KCS \cite{zhao2015efficient}.
The summary here gives insight into how this tracking method has enabled our previous work in surgical robotic control.
In addition, minor adjustments are described which show how these previous works fit into the unified approach for tool tracking presented here.

\subsubsection{Position Control}

In order to regulate a robotic suction tool along a motion plan to clear the surgical field of blood, a controller which uses Lumped Error was implemented \cite{richter2020autonomous}.
An example of this motion plan is shown in Fig. \ref{fig:blood_and_needle}.
In this automated suction work, the robotic suction tool's Lumped Error was tracked using painted markers for point features and the cylindrical insertion shaft, similar to the surgical tool tracking experiments.
Since the motion plan generates goal positions in the camera frame, which will be denoted as $\mathbf{p}^c_g \in \mathbb{R}^3$, the controller regulates the end-effector of the robotic suction tool in the camera frame.
Let $\mathbf{b}^e_t \in \mathbb{R}^3$ be the incorrect position of the end-effector in the robot base frame which was computed through forward kinematics with the noisy joint angle measurements, $\tilde{q}^i_t$.
The controller iteratively transforms the goal, $\mathbf{p}^c_g$, to the virtual base defined by the Lumped Error, and the error, $\mathbf{d}^e_t \in \mathbb{R}^3$, in the virtual base was computed as:
\begin{equation}
    \mathbf{d}^e_t = \left( \mathbf{T}^c_{b-} \mathbf{T}^{b-}_{n_b} (\hat{\mathbf{w}}_t, \hat{\mathbf{b}}_t) \right)^{-1} \overline{\mathbf{p}}_g^c - \overline{\mathbf{b}}_{t}^e
    \label{eq:position_controller_error}
\end{equation}
This error was then used to update the end-effector position
\begin{equation}
    \overline{\mathbf{b}}^e_{t+1} = \begin{cases} \gamma_s \frac{\mathbf{d}^e_{t}}{||\mathbf{d}^e_{t}||} + \overline{\mathbf{b}}^e_{t}   & \text{if } ||\mathbf{d}^e_{t}|| > \gamma_s\\
    \mathbf{d}^e_{t} + \overline{\mathbf{b}}^e_{t} & \text{if } ||\mathbf{d}^e_{t}|| \leq \gamma_s
    \end{cases}
    \label{eq:position_controller}
\end{equation}
where $\gamma_s$ is the max step size.
The updated end-effector position, $\mathbf{b}^e_{t+1}$, was set on the robotic suction tool using inverse kinematics and joint level regulators which use the noisy joint angle readings, $\tilde{q}^i_t$, as feedback.
These operations were repeated until the error, $||\mathbf{d}^e_t||$, was less than some threshold, and then a new goal position was set from the motion planner.
The resulting motion was an effective controller used to automate clearing of the surgical field from blood for hemostasis.

\begin{figure}[t]
    \centering
    \includegraphics[width=0.49\linewidth]{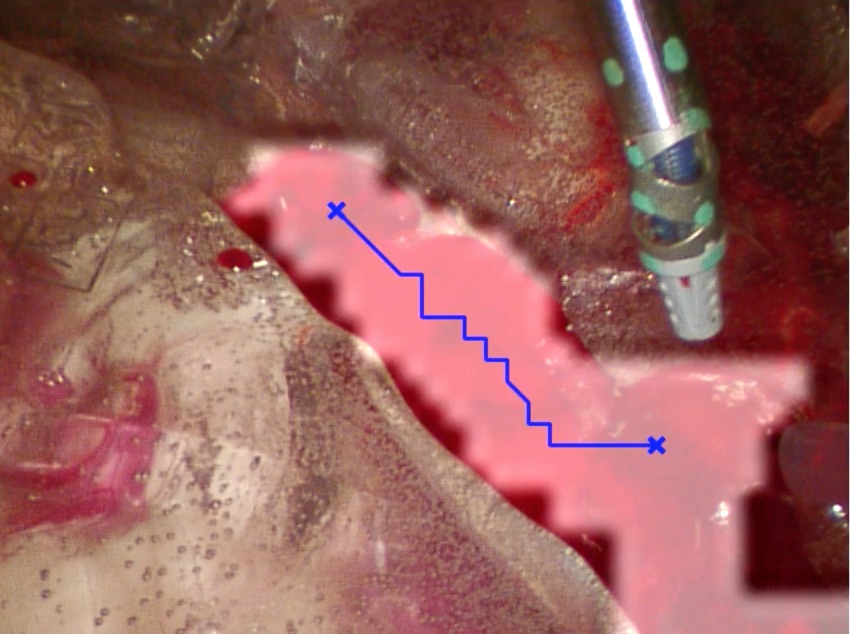}
    \includegraphics[trim=1.65cm 0cm 5cm 0mm, clip, width=0.49\linewidth]{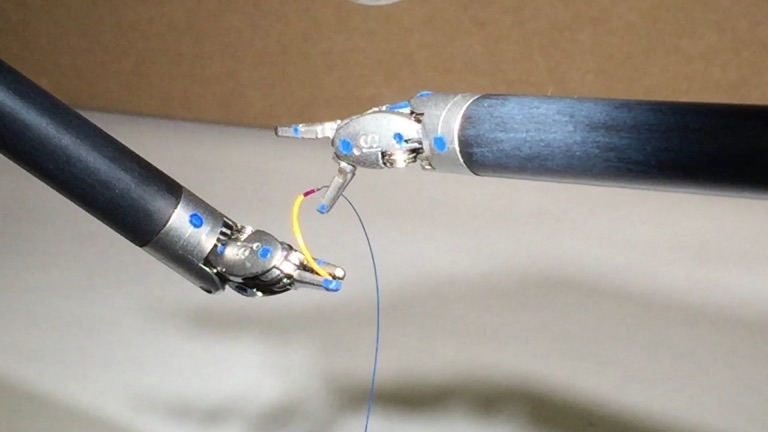}
    \caption{From left to right respectively, the figures show autonomous suction to clear the surgical field of blood for hemostasis \cite{richter2020autonomous} and automated needle regrasping for suture throwing \cite{chiu2020bimanual}. Both of these efforts used a controller to regulate the robotic surgical tools in the camera frame as the goals, blood and suture needles, are detected and tracked in it. The controllers utilize the Lumped Error to accomplish the regulation.}
    \label{fig:blood_and_needle}
\end{figure}

\subsubsection{Orientation Control}

The Lumped Error parameter reduction was also used to regulate robotic Large Needle Drivers along motion plans to conduct suture needle regrasping.
Similar to the previously described autonomous suction task, the motion plan was generated in the camera frame which acts as a bridge between the tracked surgical robotic tools and reconstructed suture needle.
The Lumped Errors of the two surgical robotic tools were tracked using the same point and edge features as described in the surgical tool tracking experiments.
Their positions were regulated using the same controller as described in (\ref{eq:position_controller_error}) and (\ref{eq:position_controller}).
The orientation was regulated in a similar fashion.
Let $\mathbf{R}^c_g \in SO(3)$ and $\mathbf{R}^e_t  \in SO(3)$ be the goal orientation in the camera frame and incorrect orientation of the end-effector in the robot base frame computed through forward kinematics with the noisy joint angle measurements, $\tilde{q}^i_t$, respectively.
The orientation of the end-effector is iteratively set to
\begin{equation}
    \mathbf{R}^e_t = \left( \mathbf{R}^c_{b-}  \mathbf{R}^{b-}_{n_b} (\hat{\mathbf{w}}_t) \right)^{-1} \mathbf{R}^c_g
\end{equation}
where $\mathbf{R}^c_{b-} \in SO(3)$ and $\mathbf{R}^{b-}_{n_b} (\hat{\mathbf{w}}_t) \in SO(3)$ are the rotation matrix of $\mathbf{T}^c_{b-}$ and $\mathbf{T}^{b-}_{n_b} (\hat{\mathbf{w}}_t, \hat{\mathbf{b}}_t)$ respectively.
Similar to the position of the end-effector, the orientation $\mathbf{R}^e_t$ is set using inverse kinematics and joint level regulators which use the noisy joint angle readings, $\tilde{q}^i_t$, as feedback.
The controller effectively regulates the surgical robotic tools along the generated motion plan to complete the task of suture needle regrasping as shown in Fig. \ref{fig:blood_and_needle}.

\subsection{Implementation Details for Particle Filter}
\label{appendix:particle_filter_details}

The parameters used to describe the motion models and observation models for da Vinci in simulation and dVRK tracking experiments are shown in Table \ref{table:particle_filter_parameters_dvrk}.
These values were chosen based on our previously equivalent work \cite{super}.
The only modification for the non-stationary robotic endoscopse case are the covariance for the Lumped Errors, $\mathbf{\Sigma}_{\mathbf{w}, t}$ and $\mathbf{\Sigma}_{\mathbf{b}, t}$, are scaled by 2.
Note that the relationship of $\mathbf{\Sigma}_{\mathbf{w}, \mathbf{b}, 0} = 10(\mathbf{\Sigma}_{\mathbf{w}, \mathbf{b}, t})$ is still kept after the scaling.
The markers are located in similar locations as the detected features from previous work by Ye et. al \cite{template_matching_EKF}.
All marker locations relative to the joint coordinate frames, $\mathbf{p}^{j_i}$ from (\ref{eq:marker_camera_projection}), were measured using calipers on the dVRK.
For the Baxter experiment, the parameters used in the particle filter are listed in Table \ref{table:particle_filter_parameters_baxter}.
These values were chosen based on a previously developed report on Baxter's performance \cite{cremer2016performance}.

\begin{table}[t]
\caption{Parameters used for particle filter to track da Vinci tools in simulation and dVRK. All angles and distances are in radians and millimeters respectively.}
\label{table:particle_filter_parameters_dvrk}
\centering
\small
\setlength\tabcolsep{1.0em}
\begin{tabular}{ c | l }
 \textbf{Parameter} & \textbf{Value} \\ &  \\[-1em] \hline &  \\[-0.75em]
 $\mathbf{a}_{\hat{\mathbf{e}}}$ &  $\begin{matrix} \big[ 0.004 & 0.004 & 2  & 0.004  \\   0.004 & 0.004 & 0.01 \big]  \end{matrix}$  \\ &  \\[-0.75em] \hline &  \\[-0.75em]
 $\mathbf{\Sigma}_{\hat{\mathbf{e}},t}$ &  diag $\left( \begin{matrix}\big[ 0.0025 & 0.0025 & 1 \\ 0.0025 & 0.0025 & 0.0025 \\ 0.005 \big] \end{matrix} \right)$  \\ &  \\[-0.75em] \hline &  \\[-0.75em]
  $\mathbf{a}_{\hat{\mathbf{e}}^c}$ &  $\begin{matrix} \big[ 0.004 & 0.004 & 2  & 0.004 \big]  \end{matrix}$  \\ &  \\[-0.75em] \hline &  \\[-0.75em]
  $\mathbf{\Sigma}_{\hat{\mathbf{e}}^c,t}$ &  diag $\left( \begin{matrix}\big[ 0.01 & 0.01 & 2.5 \\ 0.01 \big] \end{matrix} \right)$  \\ &  \\[-0.75em] \hline &  \\[-0.75em]
 $\mathbf{\Sigma}_{\mathbf{w},\mathbf{b}, t}$ &  $\text{diag}(\begin{bmatrix} \mathbf{\Sigma}_{\mathbf{w},t} & \mathbf{\Sigma}_{\mathbf{b},t} \end{bmatrix})$  \\ &  \\[-0.75em] \hline &  \\[-0.75em]
 $\mathbf{\Sigma}_{\mathbf{w},t}$ &  $\text{diag}(\begin{bmatrix} 0.005 & 0.005 & 005 \end{bmatrix})$  \\ &  \\[-0.75em] \hline & \\[-0.75em]
$\mathbf{\Sigma}_{\mathbf{b},t}$ &  $\text{diag}(\begin{bmatrix} 0.25 & 0.25 & 0.25 \end{bmatrix})$  \\ &  \\[-0.75em] \hline &  \\[-0.75em]
$\mathbf{\Sigma}_{\mathbf{w},\mathbf{b}, 0}$ & $10(\mathbf{\Sigma}_{\mathbf{w},\mathbf{b}, t})$ \\ &  \\[-0.75em] \hline &  \\[-0.75em]
$\begin{bmatrix} \gamma_m & \gamma_\phi & \gamma_\rho \end{bmatrix}$ & $\begin{bmatrix} 0.15 & 40.0 & 0.1 \end{bmatrix}$ \\ &  \\[-0.75em] \hline &  \\[-0.75em]
$\begin{bmatrix} C^m_{max} & C^l_{max} \end{bmatrix}$ & $\begin{bmatrix} 25 \gamma_m  & 0.1\gamma_\phi + 25 \gamma_\rho \end{bmatrix}$ \\ &  \\[-0.75em] \hline &  \\[-0.75em]
$\begin{bmatrix} N & N_{eff} \end{bmatrix}$ & $\begin{bmatrix} 1000 & 0.5 \end{bmatrix}$
\end{tabular}
\end{table}

\begin{table}[t]
\caption{Parameters used for particle filter to track Baxter robot. All angles and distances are in radians and millimeters respectively.}
\label{table:particle_filter_parameters_baxter}
\centering
\small
\setlength\tabcolsep{1.0em}
\begin{tabular}{ c | l }
 \textbf{Parameter} & \textbf{Value} \\ &  \\[-1em] \hline &  \\[-0.75em]
 $\mathbf{a}_{\hat{\mathbf{e}}}$ &  $\begin{matrix} \big[ 0.01 & 0.01 & 0.01  & 0.01  \\   0.01 & 0.01 & 0.01 \big]  \end{matrix}$  \\ &  \\[-0.75em] \hline &  \\[-0.75em]
 $\mathbf{\Sigma}_{\hat{\mathbf{e}},t}$ &  diag $\left( \begin{matrix}\big[ 0.001 & 0.001 & 0.001 &  \\ 0.001 & 0.001 & 0.001 \\ 0.001 \big] \end{matrix} \right)$  \\ &  \\[-0.75em] \hline &  \\[-0.75em]
 $\mathbf{\Sigma}_{\mathbf{w},\mathbf{b}, t}$ &  $\text{diag}(\begin{bmatrix} \mathbf{\Sigma}_{\mathbf{w},t} & \mathbf{\Sigma}_{\mathbf{b},t} \end{bmatrix})$  \\ &  \\[-0.75em] \hline &  \\[-0.75em]
 $\mathbf{\Sigma}_{\mathbf{w},t}$ &  $\text{diag}(\begin{bmatrix} 0.001 & 0.001 & 0.001 \end{bmatrix})$  \\ &  \\[-0.75em] \hline & \\[-0.75em]
$\mathbf{\Sigma}_{\mathbf{b},t}$ &  $\text{diag}(\begin{bmatrix} 0.25 & 0.25 & 0.25 \end{bmatrix})$  \\ &  \\[-0.75em] \hline &  \\[-0.75em]
$\mathbf{\Sigma}_{\mathbf{w},\mathbf{b}, 0}$ & $10(\mathbf{\Sigma}_{\mathbf{w},\mathbf{b}, t})$ \\ &  \\[-0.75em] \hline &  \\[-0.75em]
$\begin{bmatrix} \gamma_m & N & N_{eff} \end{bmatrix}$ & $\begin{bmatrix} 5 & 200 & 0.5 \end{bmatrix}$
\end{tabular}
\end{table}

\subsection{Camera Projection Equation for Cylinder} 
\label{appendix:shaft_projection}
The camera projection of a cylinder is an adaptation of previous work by Chaumette \cite{project_cylinder}. A cylinder is described by three parameters: a radius $r \in \mathbb{R}$, a directional vector $\mathbf{d}^j 
\in \mathbb{R}^3$ of its center axis, and a position along its center axis $\mathbf{p}_0^j \in \mathbb{R}^3$.  Let $\mathbf{d}^j$ and $\mathbf{p}_0^j$ be defined in joint frame $j$, which is the insertion shaft, and $||\mathbf{d}^j|| = 1$.
Using (\ref{eq:final_estimation_equation}) or (\ref{eq:lumped_error_endoscopic_case}) for the stationary endoscope or robotic endoscope cases respectively, $\mathbf{d}^j$ and $\mathbf{p}_0^j$ are transformed to the camera frame and denoted as $\mathbf{d}^c ( \mathbf{w}_t,\mathbf{b}_t, \mathbf{e}_t)  = \begin{bmatrix} a^c & b^c& c^c \end{bmatrix}^\top$ and $\mathbf{p}_0^c ( \mathbf{w}_t,\mathbf{b}_t, \mathbf{e}_t) = \begin{bmatrix} x^c_0 & y^c_0 & z^c_0 \end{bmatrix}^\top$ respectively. Note that $( \mathbf{w}_t,\mathbf{b}_t)$ should be replaced with $( \mathbf{w}^l_t,\mathbf{b}^l_t)$ in the robotic endoscope case. 
The center axis of the cylinder in the camera frame can be described as:
\begin{equation}
    \label{eq:insertion_shaft_cylinder_center_line}
    \mathbf{p}^c_a = \mathbf{p}_0^c( \mathbf{w}_t,\mathbf{b}_t, \mathbf{e}_t) + \lambda \mathbf{d}^c( \mathbf{w}_t,\mathbf{b}_t, \mathbf{e}_t)
\end{equation}
where $\lambda \in \mathbb{R}$. The cross section of a cylinder that is normal to the center axis can be described as the intersection between the surface of a sphere with radius $r$ centered along $\mathbf{p}^c_a$ and a plane with normal $\mathbf{d}^c$ that contains the point $\mathbf{p}^c_a$. This intersection is described as:
\begin{equation}
    \label{eq:insertion_shaft_cylinder_cross_section_intersection}
    \begin{cases}
        (\mathbf{p}^c_c -  \mathbf{p}^c_a)^\top (\mathbf{p}^c_c -  \mathbf{p}^c_a) - r^2 = 0 \\
        \mathbf{d}^c(\mathbf{w}_t, \mathbf{b}_t, \mathbf{e}_t)^\top (\mathbf{p}^c_c -  \mathbf{p}^c_a ) = 0
    \end{cases}
\end{equation}
where $\mathbf{p}^c_c \in \mathbb{R}^3$ is a point on the perimeter of the circle from the cross section of the cylinder in the camera frame. 

By combining (\ref{eq:insertion_shaft_cylinder_center_line}) and (\ref{eq:insertion_shaft_cylinder_cross_section_intersection}), an expression for the surface of a cylinder can be derived. The resulting expression of the cylinder is:
\begin{multline}
    \big(\mathbf{p}^c_s -  \mathbf{p}_0^c( \mathbf{w}_t,\mathbf{b}_t, \mathbf{e}_t) \big)^\top \big( \mathbf{p}^c_s -  \mathbf{p}_0^c( \mathbf{w}_t,\mathbf{b}_t, \mathbf{e}_t) \big) \\ - \Big( \mathbf{d}^c(\mathbf{w}_t, \mathbf{b}_t, \mathbf{e}_t)^\top \big( \mathbf{p}^c_s -  \mathbf{p}_0^c( \mathbf{w}_t,\mathbf{b}_t, \mathbf{e}_t) \big) \Big)^2 - r^2 = 0 
\end{multline}
where $\mathbf{p}^c_s = \begin{bmatrix} x^c_s & y^c_s & z^c_s \end{bmatrix}^\top$ is a point on the surface of the cylinder in the camera frame. 

Without loss of generality, let $(X,Y)$ be the projected pixel coordinates of the cylinder to a unit camera using the pin-hole model. This can be converted to the $(u,v)$ pixel location on a different camera by setting:
\begin{align}
    \label{eq:convert_unit_camera_to_other_camera}
    X = \frac{u - c_u}{f_x} &&
    Y = \frac{v - c_v}{f_y}
\end{align}
where $f_x,f_y$ and $c_u,c_v$ are the focal lengths and principal point in pixel units respectively from the camera intrinsic matrix $\mathbf{K}$. 

Applying the camera pin-hole model to the surface of the cylinder in the camera frame results in a quadratic:
\begin{equation}
    A + \frac{1}{z} B  + \frac{1}{z^2} C  = 0
\end{equation}
where
\begin{equation}
\begin{split}
    A &= X^2 + Y^2 + 1 - (a^c X + b^c Y + c^c)^2 \\
    B &= -2 \Big( (x^c_0 - a^c \nu) X + (y^c_0 - b^c \nu) Y + z^c_0 - c^c \nu \Big) \\ 
    C &= \big( x^c_0 \big)^2 + \big( y^c_0 \big)^2 + \big( z^c_0 \big)^2 - \nu^2 - r^2
\end{split}
\end{equation}
and
\begin{equation}
    \nu = a^c x^c_0 + b^c y^c_0 + c^c z^c_0
\end{equation}
The quadratic expression with respect to the depth occurs because there can be at most two solutions to the depth for each $(X,Y)$ when the cylinder is projected onto an image plane. One solution is the visible side of the cylinder, and the other is the obstructed side of the cylinder. The case of a single solution to depth would occur only at the two edges of the projected cylinder. This can be enforced by setting the determinant of the quadratic to zero ($B^2 - 4AC = 0$) resulting in:
\begin{equation}
    \Bigg( \frac{Br}{-2\sqrt{C}} - \alpha X  - \beta Y  - \kappa \Bigg) \Bigg( \frac{Br}{-2\sqrt{C}} + \alpha X  + \beta Y  + \kappa \Bigg) = 0
\end{equation}
where
\begin{equation}
    \alpha = c^c y^c_0 - b^c z^c_0 \qquad
    \beta = a^c z^c_0 - c^c x^c_0 \qquad
    \kappa = b^c x^c_0 - a^c y^c_0
\end{equation}
after simplification. 

Therefore, it is evident that the two edges from the projection of the cylinder result in two lines:
\begin{multline}
    \Bigg( \frac{ r (x^c_0 - a^c \nu)}{\sqrt{C}} - \alpha \Bigg) X + \Bigg( \frac{ r (y^c_0 - b^c \nu)}{\sqrt{C}} - \beta \Bigg) Y \\
    + \Bigg( \frac{ r (z^c_0 - c^c \nu)}{\sqrt{C}} - \kappa \Bigg) = 0
\end{multline}
and
\begin{multline}
    \Bigg( \frac{ r (x^c_0 - a^c \nu)}{\sqrt{C}} + \alpha \Bigg) X + \Bigg( \frac{ r (y^c_0 - b^c \nu)}{\sqrt{C}} + \beta \Bigg) Y \\
    + \Bigg( \frac{ r (z^c_0 - c^c \nu)}{\sqrt{C}} + \kappa \Bigg) = 0
\end{multline}
Through simple arithmetic and (\ref{eq:convert_unit_camera_to_other_camera}), both of these edges can be converted to the normal form described in (\ref{eq:hough_shaft_parameters}) for any camera. 
In the normal form, let the resulting two edges be parameterized by  $\big(\hat{\rho}_1(\mathbf{w}_t,\mathbf{b}_t, \mathbf{e}_t), \hat{\phi}_1(\mathbf{w}_t$,$\mathbf{b}_t, \mathbf{e}_t)\big)$ and $\big(\hat{\rho}_2(\mathbf{w}_t, \mathbf{b}_t, \mathbf{e}_t), \hat{\phi}_2(\mathbf{w}_t,\mathbf{b}_t, \mathbf{e}_t)\big)$.

\subsection{da Vinci Simulation Details}
\label{appendix:simulation_detials}
The stereoscopic endoscope's virtual cameras are set to render 540 by 432 images with a field of view of 60 degrees.
The baseline distance for the stereo cameras is set to 5mm to give similar depth challenges in real stereoscopic endoscopes. 
The random walk for the orientation, $\mathbf{w}^e_t$ a quaternion vector, of the end-effector is:
\begin{equation}
    \mathbf{w}^e_{t+1} = \mathbf{w}^e_t \mathbf{w}^n_t
\end{equation}
where $\mathbf{w}^n_t$ is the quaternion representation of axis-angle vector whose angle is sampled from $\mathcal{U}(0, 0.07)$ radians and axis is uniformly sampled in spherical coordinates:
\begin{equation}
    \begin{bmatrix}
    \sin{(\phi^n_t)} \cos{(\theta^n_t)} \\ \sin{(\phi^n_t)} \sin{(\theta^n_t)} \\ \cos{(\phi^n_t)}
    \end{bmatrix}
\end{equation}
where $\theta^n_t = \arccos{(u_t)}$, $u_t \sim \mathcal{U}(-1,1)$, and $\phi^n_t \sim \mathcal{U}(0, 2\pi)$.
The trajectory per trial is ran for 140 time steps.
Additional parameters are given in Table \ref{table:simulation_parameters}.

\begin{table}[b]
\caption{Parameters used for simulated experiments. All angles and distances are in radians and millimeters respectively.}
\label{table:simulation_parameters}
\centering
\small
\setlength\tabcolsep{1.0em}
\begin{tabular}{ c | l }
 \textbf{Parameter} & \textbf{Value} \\ &  \\[-1em] \hline &  \\[-0.75em]
$\mathbf{\Sigma}^{b-}_{\mathbf{w}, \mathbf{b}}$ &  $\text{diag} \left( \begin{matrix} \big[ 0.005 & 0.005 & 0.005 \\ 5 & 5 & 5 \big] \end{matrix} \right)$  \\ &  \\[-0.75em] \hline &  \\[-0.75em]
$\mathbf{a}^b_e$ &  $\ \begin{matrix} \big[ 0.004 & 0.004 & 2 & 0.004 \\ 0.004 & 0.004 & 0.01 \big] \end{matrix}$  \\ &  \\[-0.75em] \hline &  \\[-0.75em]
$\mathbf{e}_c$ &  $\ \begin{matrix} \big[ 0.02 & 0.02 & 0.0025 & 0.02 \\ 0.02 & 0.02 & 0.05 \big] \end{matrix}$  \\ &  \\[-0.75em] \hline &  \\[-0.75em]
$\mathbf{a}^c_e$ &  $ \begin{matrix} \big[ 0.004 & 0.004 & 2 & 0.004 \big] \end{matrix}$  \\ &  \\[-0.75em] \hline &  \\[-0.75em]
$\sigma_{c, l}$ &  $\begin{matrix} \big[ 0.0075 & 0.0075 & 0.75 & 0.0075 \big] \end{matrix}$ 
\end{tabular}
\end{table}

\balance
\bibliographystyle{ieeetr}
\bibliography{references}

\end{document}